\documentclass{article}

    \PassOptionsToPackage{numbers, compress}{natbib}
\usepackage[preprint]{neurips_2024}

\usepackage[utf8]{inputenc} % allow utf-8 input
\usepackage[T1]{fontenc}    % use 8-bit T1 fonts
\usepackage[hyphens]{url}            % simple URL typesetting
\usepackage{hyperref}       % hyperlinks
\usepackage{booktabs}       % professional-quality tables
\usepackage{amsfonts}       % blackboard math symbols
\usepackage{nicefrac}       % compact symbols for 1/2, etc.
\usepackage{microtype}      % microtypography
\usepackage{xcolor}         % colors
\usepackage{rotating}

\usepackage{color-edits}

\addauthor{ez}{blue}
\addauthor{ep}{orange}

\usepackage{amssymb,amsmath,amsthm,enumitem,graphicx,tikz,bbm, 
appendix, accents}
\usepackage[noabbrev,capitalize]{cleveref}
\crefname{equation}{}{}
\usepackage{algorithm, algpseudocode}
\usepackage{thm-restate}

\newcommand{\ubar}[1]{\underaccent{\bar}{#1}}

\newtheorem{theorem}{Theorem}[section] 

\newtheorem{lemma}[theorem]{Lemma}
\newtheorem{claim}[theorem]{Claim}

\theoremstyle{definition}

\newtheorem{assumption}[theorem]{Assumption}

\newcommand{\floor}[1]{\left\lfloor #1 \right\rfloor}
\newcommand{\ceil}[1]{\left\lceil #1 \right\rceil}
\newcommand{\R}{\mathbb {R}}  % The real numbers.
\newcommand{\abs}[1]{\left\lvert#1\right\rvert}
\newcommand{\norm}[1]{\left\lVert#1\right\rVert}

\newcommand{\sang}[1]{\left\langle #1 \right\rangle}

\newcommand{\paren}[1]{\left( #1 \right)}
\newcommand{\sqb}[1]{\left[ #1 \right]}
\newcommand{\set}[1]{\left\{ #1 \right\}}

\newcommand {\F } {{\mathrm{F}}}

\DeclareMathOperator*{\argmin}{argmin}

\newcommand{\Diam}{\text{Diam}}

\definecolor{codered}{rgb}{.8,.1,.1}
\definecolor{codegreen}{rgb}{0,0.6,0}
\definecolor{codeblue}{rgb}{0,0.1,1}
\definecolor{codegray}{rgb}{0.5,0.5,0.5}
\definecolor{codepurple}{rgb}{0.58,0,0.82}

\title{
Heterogeneous Treatment Effects in Panel Data}

\author{%
  Retsef Levi\\
  Sloan School of Management\\
  Massachusetts Institute of Technology\\
  \texttt{retsef@mit.edu} 
  \And
  Elisabeth Paulson \\
  Harvard Business School \\
  \texttt{epaulson@hbs.edu} \\
  \AND
  Georgia Perakis \\
  Sloan School of Management\\
  Massachusetts Institute of Technology \\
  \texttt{georgiap@mit.edu} \\
  \And
  Emily Zhang \\
  Operations Research Center \\
  Massachusetts Institute of Technology \\
  \texttt{eyzhang@mit.edu} \\
}

\begin{document}

\maketitle

\begin{abstract}
We address a core problem in causal inference: estimating heterogeneous treatment effects using panel data with general treatment patterns. Many existing methods either do not utilize the potential underlying structure in panel data or have limitations in the allowable treatment patterns. In this work, we propose and evaluate a new method that first partitions observations into disjoint clusters with similar treatment effects using a regression tree, and then leverages the (assumed) low-rank structure of the panel data to estimate the average treatment effect for each cluster. Our theoretical results establish the convergence of the resulting estimates to the true treatment effects. Computation experiments with semi-synthetic data show that our method achieves superior accuracy compared to alternative approaches, using a regression tree with no more than 40 leaves. Hence, our method provides more accurate and interpretable estimates than alternative methods.
\end{abstract}

\section{Introduction}

Suppose we observe an outcome of interest across $n$ distinct units over $T$ time periods; such data is commonly referred to as panel data. Each unit may have been subject to an intervention during certain time periods that influences the outcome. As a concrete example, a unit could be a geographic region affected by a new economic policy, or an individual consumer or store influenced by a marketing promotion. Our goal is to understand the impact of this intervention on the outcome. Since the effect of the intervention might vary across individual units and time periods as a function of the unit-level and time-varying covariates, we aim to estimate the heterogeneous treatment effects. This is a key problem in econometrics and causal inference, enabling policymakers or business owners to make more informed decisions about which units to target for future interventions.

% We delay a formal development to later. 

% \subsection{Contributions}

We develop a new method, the Panel Clustering Estimator (PaCE), for estimating heterogeneous treatment effects in panel data under general treatment patterns. The causal effects are modeled as a non-parametric function of the covariates of the units, which may vary over time. To estimate heterogeneous effects, PaCE splits the observations into disjoint clusters using a regression tree and estimates the average treatment effect of each cluster.
%using an extension of the methodology from \cite{farias2021learning}. 
%The splitting process is an iterative, greedy process, that progressively splits the regression tree, partitioning the data into more granular clusters that better represent the true heterogeneous treatment effects. \epcomment{Do we really need this last sentence?}
% Direct evaluation of causal inference methods on real-world test sets, as commonly done in prediction tasks, is not possible because true treatment effects are not observed. Instead, asymptotic theory and evaluation using semi-synthetic data, where the treatment effect is known, are used to validate methods. We show that our method converges to the true heterogeneous treatment effect estimates asymptotically, and validate our method empirically using semi-synthetic data.
In addition to the development of this novel method, we make the following theoretical and empirical contributions. 
% Our contributions are formally summarized as follows.

\paragraph{Regression tree with bounded bias (\cref{thm: tree})} We show that, subject to mild assumptions on the regression tree and a density assumption on the covariates, the bias introduced by approximating the treatment effect function with a piece-wise constant function, produced by the regression tree, decreases polynomially as the number of leaves in the tree increases.

\paragraph{Optimal rate for average treatment effect estimation (\cref{thm: convergence})}
We obtain guarantees for the recovery of average treatment effects for each cluster of observations that matches the optimal rate, up to logarithmic factors. This extends the convergence guarantee of \cite{farias2021learning} to a setting with \textit{multiple} treatments. Similar to \cite{farias2021learning}, we allow for general intervention patterns. %\epdelete{Essentially, all our assumptions are the extension of the (provably minimal) assumptions made in \cite{farias2021learning} to multiple treatments.}

\paragraph{Enhanced empirical accuracy with simple estimator}
Using semi-synthetic data, we demonstrate that PaCE achieves empirical performance competitive with, and often superior to, alternative methods for heterogeneous treatment effect estimation, including double machine learning, causal forest, and matrix completion-based methods. Furthermore, PaCE offers a much simpler and more interpretable estimate, as we achieve this enhanced performance using a regression tree with at most 40 leaves. 

\subsection{Related literature}

Causal inference using panel data is often approached through the synthetic control framework, where treated units are compared to a synthetic control group constructed from untreated units \cite{abadie2010synthetic, abadie2003economic}. However, for these methods, the treatment is restricted to a block. 
Matrix completion approaches offer a more flexible alternative, allowing for general treatment patterns \cite{matrix_completion, bai2021matrix, xiong2023large}. The treated entries are viewed as missing entries, and low-rank matrix completion is utilized to estimate counterfactual outcomes for the treated entries. 
In most of the matrix completion literature, the set of missing or treated entries are assumed to be generated randomly \cite{chen2019inference}.
However, this assumption is not realistic in our panel data setting. Other works explore matrix completion under deterministic sampling \cite{chatterjee2020deterministic, farias2021learning, foucart2020weighted, klopp2017robust, liu2017new}, and our work extends the methodology from \cite{farias2021learning}.
% Furthermore, as both synthetic control and matrix completion methods focus on estimating counterfactual outcomes for the treated observations, they do not readily provide heterogeneous treatment effect estimates for untreated observations.

Several methods for estimating heterogeneous treatment effects use supervised learning techniques to fit models for the outcome or treatment assignment mechanisms \cite{chernozhukov2017double, chernozhukov2017orthogonal, DR, XLearner, nie2021quasi}.
Although these methods are intended for cases where observations are drawn i.i.d.~from one or multiple predefined distributions, they could still be applied to panel data. In contrast, our approach is specifically designed for panel data with underlying structure. By avoiding complex machine learning methods, our approach not only provides accurate results for panel data but also offers interpretable estimates. 

The first step of our method involves grouping data into clusters that vary in their treatment effects using a regression tree. As in \cite{breiman2001random}, we greedily search for good splits, iteratively improving the tree's representation of the true heterogeneous treatment effects. However, unlike supervised learning where the splitting criterion is directly computable, treatment effects are not readily observed. Therefore, we estimate the splitting criterion, an approach used in \cite{athey2016recursive, athey2019generalized, su2009subgroup, wager2018estimation}. While methods like causal forest only use information within a given leaf to estimate the treatment effect for observations in that leaf \cite{wager2018estimation}, our method leverages global information from the panel structure to determine the average treatment effect of each cluster, enabling us to achieve high accuracy with just one tree.

% , which focused on the case where observations are drawn i.i.d.~from a given distribution

The second step involves estimating the average treatment effect in each cluster. For this step, we utilize the de-biased convex estimator introduced by \cite{farias2021learning}. This estimator is applicable to panel data with general treatment patterns but is specifically designed for estimating average, rather than heterogeneous, treatment effects. 
It uses a convex estimator, similar to those used in \cite{gobillon2016regional, xiong2023optimal}, and introduces a novel de-biasing technique that improves upon the initial estimates. Additionally, 
\cite{farias2021learning} builds on previous works \cite{chen2020noisy, chen2021bridging} to provide theoretical guarantees for recovering the average treatment effect with provably minimal assumptions on the intervention patterns. The results of \cite{farias2021learning} are specific to the case where there is only one treatment. We build on their proof techniques to show convergence when the number of treatments is $O(\log n)$, and also nest this technique into the first step of our algorithm to estimate heterogeneous effects.

\subsection{General notation} For any matrix $A$, $\norm{A}$ denotes the spectral norm, $\norm{A}_\F$ denotes the Frobenius norm, $\norm{A}_\star$ denotes the nuclear norm, and $\|A\|_{\text{sum}}$ denotes the sum of the absolute values of the entries in $A$. For any vector $a$, $\norm{a}$ denotes its Euclidean norm, and $\norm{a}_{\infty}$ denotes the maximum absolute value of its components. We use $\circ$ to denote element-wise matrix multiplication. %\epcomment{This might be clearer if we say $a(t)\lesssim b(t)$ whenever $a(t) \leq c\cdot b(t)$ for some universal constant $c$ and all $t$.} 
% $\sqb{I_{\paren{A_{ij}\leq c}}}_{i,j}$. 
For a matrix subspace $\mathbf T$, let $P_{{\mathbf T}^{\perp}}(\cdot)$ denote the projection operator onto ${\mathbf T}^{\perp}$, the orthogonal space of $\mathbf T$. 
We define the sum of two subspaces $\mathbf T_1$ and $\mathbf T_2$ as $\mathbf T_1+ \mathbf T_2 = \set{T_1+T_2\mid T_1\in\mathbf T_1,T_2\in\mathbf T_2}$.
%$I$ is an indicator variable. %_C=1$ if condition $C$ is true and 0 otherwise. 
%For a matrix $A$ and a scalar $c$, $I_{\set{A \leq c}}$ denotes the matrix formed by applying the indicator condition element-wise. 
We write $a\lesssim b$ whenever $a \leq c\cdot b$, for some fixed constant $c$; $a\gtrsim b$ is defined similarly. 

\section{Model and algorithm}\label{section: model}
%\epedit{In this section we introduce our formal model of the observed data and treatment effects, and present our algorithm, the Panel Clustering Estimator (PaCE).}

PaCE is designed for panel data, where an outcome of interest is observed across $n$ distinct units for $T$ time periods. Let $O\in \mathbb  R^{n\times T}$ be the matrix of these observed outcomes. Our objective is to discern how these outcomes were influenced by various treatments under consideration. 

We assume that, in a hypothetical scenario where the treatments were not applied, the expected outcomes would be represented by an unknown low-rank matrix $M^{\star}\in \mathbb  R^{n\times T}$. 
To incorporate the influence of the external factors, we introduce the covariate tensor $\mathbf X:=[X_1,\dots,X_p] \in \mathbb  R^{p\times n\times T}$, where an element $X_{izt}$ signifies the $i$-th covariate of unit $z$ at time $t$. For convenience, we will write $X^{zt} := [X_{1zt},\dots,X_{pzt}]$ as the vector of relevant covariates for unit $z$ at time $t$.
We consider $q$ distinct treatments, each represented by a binary matrix $W_1,\dots,W_q$ that encodes which observations are subjected to that treatment. 
For each treatment $i\in[q]$ applied to unit $z$ at time period $t$, the treatment effect is modeled as a non-parametric Lipschitz function of the covariates $\mathcal T_i^{\star}(X^{zt})$. 
We consider the treatment effects of distinct treatments to be additive.\footnote{This is a flexible framework because, as we allow for up to $\log(n)$ treatments, the interactions of individual treatments can be modeled as additional treatments.}
Combining all of these elements, the observed outcome matrix $O$ can be expressed as the sum of $M^{\star}$, the combined effects of the treatments, and a noise matrix $E$. Mathematically,
\begin{align*}
O = M^{\star} + \sum_{i=1}^q \mathcal T_i^{\star}(\mathbf X)\circ W_i + E,
\end{align*}
where 
$\mathcal T_i^{\star}(\mathbf X)$ denotes the matrix 
where the element in $z$-th row and $t$-th column is 
$\mathcal T_i^{\star}(X^{zt})$. We note that our model implicitly assumes unconfoundedness, which is common in the literature \cite{imbens2015causal}.

Our goal is to estimate the treatment effect function $\mathcal T^{\star}_i$ for each $i\in[q]$, having observed $O$, $\mathbf X$, and $W_1,\ldots,W_q$. Our methodology PaCE consists of two main steps. First, we build a specialized tree for each treatment $i\in[q]$, designed to partition the observations into disjoint clusters that differ in their treatment effects. Next, we estimate the average treatment effect of each cluster, thus obtaining estimates for the treatment effect functions that are piece-wise constant functions of the covariates. The following subsections will delve into the specifics of these two steps.

\subsection{Clustering the observations}

Initially, the entire dataset with $n\times T$ entries is in one single cluster. This cluster serves as the starting point for the iterative partitioning process, where we split the data into $\ell$ more granular clusters for each treatment that better represent the true heterogeneous treatment effects. We present our method for clustering observations in \cref{alg:clustering}. Below, we discuss its two main steps.

\begin{algorithm}[!ht]%[H]
\caption{Clustering the observations}
\label{alg:clustering}
\begin{algorithmic}[1]
\State \textbf{Input:} outcome matrix $O$, covariate matrices $X_1,\dots, X_p$,  treatment matrices $W_1,\dots,W_q$, maximum number of leaves $\ell$, regularization parameter $\lambda$.
\State 
For each treatment $i\in [q]$, initialize $j$-th cluster matrix $C^i_j$ such that all observations are in one single cluster:
$C_1^i\leftarrow \mathbf{1}_{n\times T}$ and $C_j^i\leftarrow \mathbf{0}_{n\times T}$ for $j >1$. 
\For{$l = 1$ to $\ell-1$ }% \textbf{and} a valid split was found}
    \State \textbf{Step 1:} Estimate the average treatment effect for each cluster:
    \begin{align}\label{prob:convex_clustering}
    \paren{\hat{M}, \hat{\tau}, \hat{m}} \leftarrow \argmin_{\substack{M \in \mathbb{R}^{n \times T}, \\
    \tau \in \mathbb{R}^{q \times l}, 
    m \in \mathbb{R}^n}}  
    \frac{1}{2}\norm{O-M - m\mathbf{1}^\top -\sum_{i=1}^q \sum_{j=1}^l \tau_{i,j}W_i\circ C^i_j}^2_\F + \lambda\norm{M}_\star.
    \end{align}
    \State \textbf{Step 2:} The following is performed for each treatment $q'\in[q]$ to select the valid single-variable split that minimizes the estimated mean squared error: 
    \begin{itemize} 
        \item 
        % For $l' \in [l], p'\in[p], x\in[\min(X_{p'}),\max(X_{p'})]$, 
        % We define the new $j$-th cluster matrix for treatment $i$, formed by splitting the $l'$-th cluster for treatment $q'$ on $X_{p'}=x$, as follows:
        For $i\in \sqb{q}$ and $j\in \sqb{l}$, let the function $\mathcal C_{j}^{i}(l', p', x)$ encode the new cluster $C^i_j$ after splitting the $l'$-th cluster for treatment $q'$ based on the comparison of the $p'$-th covariate to $x$: %$X_{p'}=x$; that is, 
        \begin{align*}
        \mathcal C_{j}^{i}(l', p', x) :=
        \begin{cases}
            C_{l'}^{q'} \circ  I_{\{X_{p'}\leq x\}}, & \text{for } j = l'\text{ and }i=q' \\
            C_{l'}^{q'} \circ  I_{\{X_{p'}> x\}}, & \text{for } j = l + 1\text{ and }i=q'\\
            C_j^i, & \text{otherwise.}
        \end{cases}
        \end{align*}

        \item 
        Identify the split that minimizes the estimated mean squared error: $(l^{\star},p^{\star},x^{\star}) \leftarrow \argmin_{l', p', x}\widehat{MSE}(l', p', x)$, where
        \begin{align*} 
        \widehat{MSE}(l', p', x):= \min_{\tau\in\mathbb {R}^{q\times (l +1)}}  \norm{O - \hat M - \hat m \mathbf{1}^\top - \sum_{i=1}^q \sum_{j=1}^{l+1} \tau_{i,j}W_i\circ \mathcal C^{i}_j(l', p', x)}_\F^2.
        \end{align*}

        \item For $j\in \set{l^\star, l + 1}$, update the cluster $C_j^{q'} \leftarrow \mathcal C_{j}^{q'}\paren{l^{\star},p^{\star},x^{\star}}$.

        % \item For any lost function $L$ that measures the discrepancy between $O-\hat M-\hat m\mathbf 1^\top$ and the combined estimated treatment effect $\sum_{i=1}^q \sum_{j=1}^{l+1} \tau_{i,j}W_i\circ N^{i}_j(l', p', x)$, where $\tau_{i,j}$ is selected to minimize $L$, select the split $(l^{\star},p^{\star},x^{\star})$ that minimizes $L$.

    \end{itemize}
\EndFor
% \State \textbf{Final Step:} De-bias the estimate by applying the method given in \cref{subsection: debiasing}.
% Return 
    % $\hat\tau - D^{-1}\Delta_1$,
% where $D$ and $\Delta_1$ are defined in \cref{lem:tau-decomposition}. 
\end{algorithmic}
\end{algorithm}

\paragraph{Step 1} 
We solve (the convex) Problem \ref{prob:convex_clustering} to obtain a rough estimate of the counterfactual matrix, represented by $\hat M + \hat m\mathbf 1^\top$, for a given clustering of observations. The regularization term $\norm{M}_\star$ penalizes the rank of $M$. 
The term $m\mathbf 1^\top$ is introduced to center each row of matrix $\hat M$ without penalization from the nuclear norm term. The inclusion of this term deviates from the convex problem in \cite{farias2021learning}, and we find that its inclusion improves empirical performance.
% We suggest the following strategy to select the regularization parameter $\lambda$. Supposing the suggested rank $r$ of $M^{\star}$ is known, then we recommend binary searching for a $\lambda$ such that the resulting $\hat M$ will be a rank $r$ matrix as well.
The choice of $\lambda$ is discussed in our theoretical results.

\paragraph{Step 2}
% For each treatment $i \in [q]$ and every $j\in[\ell]$, $C_j^i$ is a binary $n\times T$ matrix that indicates membership of each entry in cluster $j$. 
In the second step, we greedily choose a \textit{valid split} to achieve a better approximation of the true heterogeneous treatment effects. A valid split is one that complies with predefined constraints, which will be discussed in our theoretical guarantees. While the provided algorithm selects splits that minimize the mean squared error (MSE), alternative objectives such as mean absolute error (MAE) may be used instead.

% \textbf{Final Step:} Our last step is to adjust our estimate $\hat\tau$ to correct for any bias introduced by the regularization term. This correction is described in the following subsection. The estimate for the treatment effect of a given treatment on a given observation is the de-biased $\tau$ value for that cluster. 

\subsection{Estimating the average treatment effect within each cluster}\label{subsection: debiasing}

After partitioning all observations into $\ell$ clusters for each of the $q$ treatments, the next step in our algorithm is to estimate the average treatment effect within each cluster. This is done using an extension of the de-biased convex estimator introduced in \cite{farias2021learning}. We first obtain an estimate of the average treatment effects by solving a convex optimization problem. This initial estimate is then refined through a de-biasing step.

For simplicity, we define separate treatment matrices for each cluster. Letting for $k := q\times\ell$, $\Tilde Z_1,\dots, \Tilde Z_k$ are defined as follows: $\Tilde Z_{(i-1)\times \ell + j} = W_i\circ C_j^i$ for $i\in[q], j\in[\ell]$. These are binary matrices that identify observations that are in a given cluster and receive a given treatment. We denote the true average treatment effect of the treated in each cluster $i\in[k]$ as $\tilde\tau^{\star}_i$, and let $\delta_i$ be the associated residual matrices, which represent the extent to which individual treatment effects differ from the average treatment effect. More precisely, for $i\in[k]$, 
\begin{align*}
\tilde\tau^{\star}_i:={\big\langle\mathcal T^{\star}_{\ceil{\frac{i}{\ell}}}(\mathbf X),\Tilde Z_i\big\rangle}\big/{\big\|\Tilde Z_{i}\big\|_{\text{sum}}} \qquad\text{and}\qquad \delta_i :=  \mathcal T^{\star}_{\ceil{\frac{i}{\ell}}}(\mathbf X) \circ\Tilde Z_i - \tilde\tau^{\star}_i\Tilde Z_i.
\end{align*}
Finally, we define $\delta$ as the combined residual error, where $\delta := \sum_{i=1}^k \delta_i$, and $\hat{E}:= E + \delta$ is the total error. 

We normalize all of these treatment matrices to have Frobenius norm 1. This standardization allows the subsequent de-biasing step to be less sensitive to the relative magnitudes of the treatment matrices. Define $Z_{i} := \Tilde Z_i/\|\Tilde Z_i\|_\F$. By scaling down the treatment matrices in this way, our estimates are proportionally scaled up. Similarly, we define a rescaled version of $\tilde \tau^{\star}$, given by $\tau^{\star}_i:= \tilde\tau^{\star}_i \|\Tilde Z_i\|_\F.$

First, we solve the following convex optimization problem with the normalized treatment matrices to obtain an estimate of $M^{\star}$ and $\tau^{\star}$:
\begin{align} \label{eq:convex-program}
\paren{\hat{M}, \hat{\tau}, \hat{m}} \in \argmin_{M\in \mathbb {R}^{n\times T}, \tau\in\mathbb {R}^k, m\in\mathbb {R}^n} 
\frac{1}{2}\norm{O-M - m\mathbf{1}^\top -\sum_{i=1}^k \tau_iZ_i }^2_\F + \lambda\norm{M}_\star.
\end{align}

Because we introduced the $m\mathbf 1^\top$ term in the convex formulation, which centers the rows of $\hat M$ without penalization from the nuclear norm term, $\hat M$ serves as our estimate of $M^{\star}$, projected to have zero-mean rows, while $\hat m$ serves as our estimate of the row means of $M^{\star}$. For notational convenience, let $P_{\mathbf 1}$ denote the projection onto $\set{\alpha \mathbf 1^\top \mid \alpha\in\mathbb  R^n}$. That is, for any matrix $A\in \mathbb{R}^{n\times T}$,
\begin{align*}
P_{\mathbf 1}(A) = A\frac{\mathbf1 \mathbf1^\top}{T}\qquad\text{and}\qquad
P_{{\mathbf 1}^\perp}(A) = A\paren{I-\frac{\mathbf1 \mathbf1^\top}{T}}. 
\end{align*}
Hence $P_{{\mathbf 1}^\perp}(M^{\star})$ is the projection of $M^{\star}$ with zero-mean rows. Let $m^{\star} := \frac{M^{\star}\mathbf 1}{T}$ be the vector representing the row means of $M^{\star}$.  Let $r$ denote the rank of $P_{{\mathbf 1}^\perp}(M^{\star}) $,\footnote{Note that we can show $r \leq \text{rank}(M^*)+1$, so $P_{{\mathbf 1}^\perp}(M^{\star})$ is a low-rank matrix as well.} and let $P_{{\mathbf 1}^\perp}(M^{\star}) = U^{\star}\Sigma^{\star}V^{\star\top}$ be its SVD, with $U^{\star} \in \mathbb  R^{n\times r}, \Sigma^{\star} \in \mathbb  R^{r\times r},V^{\star} \in \mathbb  R^{T\times r}$. We denote by $\mathbf T^{\star}$ the span of the space $\set{\alpha \mathbf 1^\top \mid \alpha \in \mathbb  R^n}$ and the tangent space of $P_{{\mathbf 1}^\perp}(M^{\star})$ in the manifold of matrices with rank $r$. That is,
\begin{align*}
\mathbf T^{\star} &= \set{U^{\star}A^\top + BV^{\star\top} + a\mathbf{1}^\top \mid A\in \mathbb{R}^{T \times r}, B \in \mathbb{R}^{n \times r}, a \in \mathbb{R}^{n}}.
\end{align*}
The closed form expression of the projection $P_{{\mathbf T}^{\star\perp}}(\cdot)$ is given in \cref{lemma: projection}.

The following lemma gives a decomposition of $\hat \tau - \tau^{\star}$ that suggests a method for de-biasing of $\hat \tau$. The difference between our error decomposition and the decomposition presented in \cite{farias2021learning} stems from the fact that we introduced the $m\mathbf 1^\top$ term in \eqref{eq:convex-program}. %Consequently, our convex estimator allows each row of the estimated counterfactual matrix $\hat M$ to be centered without penalization.

\begin{lemma}[Error decomposition]\label{lem:tau-decomposition}
%Suppose $(\tau, M)$ satisfies \cref{eq:convex-condition-tau} and \cref{eq:convex-condition-M}, then
Suppose $(\hat{M}, \hat{m}, \hat{\tau})$ is a minimizer of \eqref{eq:convex-program}. Let $\hat{M} = \hat{U}\hat{\Sigma} \hat{V}^{\top}$ be the SVD of $\hat{M}$, and let $\hat{\mathbf T}$ denote the span of the tangent space of $\hat{M}$ and $\set{\alpha \mathbf 1^\top \mid \alpha \in \mathbb  R^n}$. 
Then,
\begin{equation} \label{eq:tau-decomposition}
D \paren{\hat{\tau}-\tau^{\star}}
= \Delta^1 + \Delta^2 + \Delta^3,
\end{equation}
where 
$D \in \R^{k \times k}$ is the matrix with entries $D_{ij} = \langle
P_{\hat {\mathbf T}^{\perp}}(Z_i),P_{\hat {\mathbf T}^{\perp}}(Z_j)
\rangle
$
and $\Delta^1, \Delta^2, \Delta^3 \in \mathbb {R}^{k}$ are vectors with components
\begin{align*}
\Delta^1_i = 
\lambda 
\sang{
Z_i, \hat U \hat V^\top 
} ,\quad
\Delta^2_i = 
\sang{Z_i, P_{\hat {\mathbf T}^{\perp}}(\hat{E})},\quad
\Delta^3_i = 
\sang{Z_i, P_{\hat {\mathbf T}^{\perp}}(M^{\star})}.
\end{align*}
\end{lemma}

As noted in \cite{farias2021learning}, since $D^{-1}\Delta^1$ depends solely on observed quantities, it can be subtracted from our estimate. Thus, we define our de-biased estimator $\tau^d$ as $\tau^d_i= \paren{\hat{\tau}- D^{-1}\Delta^1}_i/\|\tilde Z_i\|_\F$ for each $i\in[k]$. The rescaling by $\|\tilde Z_i\|_\F$ adjusts our estimate to approximate $\tilde\tau^{\star}$, rather than $\tau^{\star}.$

The de-biased $\tau^d$ serves as our final estimate. Specifically, if an observation's covariates $X^{zt}$ belong to the $j$-th cluster, then for each treatment $i\in[q]$, our estimate of $\mathcal T^{\star}_i(X^{zt})$ is given by $\hat{\mathcal T}_i(X^{zt})=\tau^d_{(i-1)\times \ell + j}$, which represents the average treatment effect of treatment $i$ associated with the cluster that $X^{zt}$ belongs to.

\section{Theoretical guarantees}\label{section: theoretical guarantees}

Our analysis determines the conditions for the convergence of our algorithm's estimates to the true treatment effect functions $\mathcal T_i^{\star}(\mathbf X)$ for each $i \in [q]$. Our proof incorporates ideas from \cite{farias2021learning, wager2018estimation}.

The final treatment effect estimate given by PaCE has two sources of error. First, there is the \textit{approximation error}, which stems from the approximation of the non-parametric functions $\mathcal T^{\star}_i$ by piece-wise constant functions. The second source of error, the \textit{estimation error}, arises from the estimation of the average treatment effect of each cluster. Note that the deviation between our estimate and the true treatment effect is, at most, the sum of these two errors. We will analyze and bound these two different sources of error separately to demonstrate the convergence of our method.

\subsection{Approximation error}
We start by bounding the bias introduced by approximating $\mathcal T^{\star}_i$ by a piece-wise constant function. To guarantee consistency, we establish constraints on valid splits. We adopt the \textit{$\alpha$-regularity} condition from Definition 4b of \cite{wager2018estimation}. This condition requires that each split retains at least a fraction $\alpha \in (0,\frac{1}{2})$ of the available training examples and at least one treated observation on each side of the split. We further require that the depth of each leaf be on the order $\log \ell$ and that the trees are \textit{fair-split trees}. That is, during the tree construction procedure in \cref{alg:clustering}, for any given node, if a covariate $j$ has not been used in the splits of its last $\pi p$ parent nodes, for some $\pi >1$, the next split must utilize covariate $j$. If multiple covariates meet this criterion, the choice between them can be made arbitrarily, based on any criterion.

Additionally, we assume that the proportion of treated observations with covariates within a given hyper-rectangle should be approximately proportional to the volume of the hyper-rectangle. This is a `coverage condition' that allows us to accurately estimate the heterogeneous treatment effect on the whole domain using the available observations.
% This is similar to the positivity assumption often made in the causal analysis of observational data \cite{cole2009consistency}. %\epcomment{can we give some intuition for why we need an assumption like this?} \enote{It's a coverage condition that allows us to accurately estimate the heterogeneous treatment effect on the whole domain using the available observations.} 
For the purposes of this analysis, we assume that all covariates are bounded and normalized to $[0,1]^p$.
\begin{assumption} \label{assumption: continuous covariates}
Suppose that all covariates belong to $[0,1]^p$. Let $x^{(1)}\leq x^{(2)}\in[0,1]^p$ be the lower and upper corners of any hyper-rectangle. We assume that the proportion of observations that have covariates inside this hyper-rectangle is proportional to the volume of the hyper-rectangle $V := \prod_{p'\in[p]} \paren{x_{p'}^{(2)}-x_{p'}^{(1)}}$, with a margin of error $M:= \sqrt{\frac{\ln(nT)(p+1)}{\min(n,T)}}$. More precisely,
\begin{equation*}
\ubar c  V -c_mM\leq 
\frac{\#\set{(z,t):  x^{(1)}\leq X^{zt}\leq x^{(2)}
}  }{nT} \leq \bar c  V+ c_mM,
\end{equation*}
for some fixed constants $\bar c \geq \ubar c >0$ and some $c_m>0$.  
\end{assumption}

\cref{assumption: continuous covariates} is a broader condition than the requirement for the covariates to be i.i.d.~across observations, as assumed in \cite{wager2018estimation}. Our assumption only requires that the covariates maintain an even density. Indeed, \cref{lemma: iid satisfies assumption} in \cref{appendix: approx error} illustrates that while i.i.d.~covariates satisfy the assumption with high probability, the assumption additionally accommodates (with high probability) scenarios where covariates are either constant over time and only vary across different units, or constant across different units and only vary over time.

In the following result, we demonstrate that, as the number of observations grows, a regression tree, subject to some regularity constraints, contains leaves that become increasingly homogeneous in terms of treatment effect.

\begin{theorem}\label{thm: tree}
Let $\mathbb  T$ be an $\alpha$-regular, fair-split tree, split into $\ell$ leaves, each of which has a depth of at least $c\log \ell$ for some constant $c$. Suppose that $\ell \leq \paren{\frac{\alpha}{2c_m M}}^{\frac{1}{c\log 1/\alpha}}$ and that \cref{assumption: continuous covariates} is satisfied. Then, for each treatment $i\in[q]$ and every leaf of $\mathbb  T$, the maximum difference in treatment effects between any two observations within the cluster $C$ of observations in that leaf is upper bounded as follows:
\begin{equation*}
     \max_{X_1,X_2\in C} \abs{\mathcal T_i^{\star}(X_1) - \mathcal T_i^{\star}(X_2)} \leq \frac{2L_i\sqrt{ p }}{\ell^s},
\end{equation*}
% \begin{equation*}
%      \max_{X\in C} \abs{\mathcal T_i^{\star}(X) - \frac{1}{\abs{C}}\sum_{X'\in C}\mathcal T_i^{\star}(X')} \leq \frac{2L_i\sqrt{ p }}{\ell^s},%,\quad i \in[q],j\in[l].
% \end{equation*}
where $s:=\frac{c}{(\pi+1)p}\frac{\alpha\ubar c}{4\bar c}$ and $L_i$ is the Lipschitz constant for the treatment effect function $\mathcal T_i^{\star}$.
\end{theorem}

Using the above theorem, we can make the following observation. For any $i\in[q]$, denote by $\tilde\tau^{\star,i}$ the regression tree function that is piece-wise constant on each cluster $C_j^i$ for $j\in[\ell]$, with a value of $\tilde\tau^{\star}_{(i-1)\ell+j}$ assigned to each cluster $j$. Then, for all covariate vectors $X$, we have $\abs{\mathcal T_i^{\star}(X) - \tilde\tau^{\star,i}(X) } \leq \frac{2L_i\sqrt{ p }}{\ell^s}$ because the value assigned to each cluster by $\tilde\tau^{\star,i}$ is an average of treatment effects within that cluster.
% For every $i\in[q]$ and $j\in [\ell]$, because the quantity $\tilde\tau^{\star}_{(i-1)\ell+j} $ is the average treatment effect within cluster $C_j^i$, we have
% \begin{align*}
%     \max_{X\in C_j^i} \abs{\mathcal T_i^{\star}(X) - \tilde\tau^{\star}_{(i-1)\ell+j} }  \leq \frac{2L_i\sqrt{ p }}{\ell^s}.
% \end{align*}
% If the treatment effect estimate for a given cluster $C$ is any convex combination of values from the set  $\{\mathcal T_i^\star(X),X\in C\}$, then the approximation error for any observation with covariates in $C$ is bounded by $\frac{2L_i\sqrt{ p }}{\ell^s}$. 
This demonstrates that it is possible to approximate the true treatment effect functions with a regression tree function with arbitrary precision. 
As the number of leaves $\ell$ in the tree $\mathbb  T$ increases, the approximation error of the tree decreases polynomially. The theorem requires an upper bound on $\ell$, which grows with $\min(n,T)$ as the dataset expands, due to the definition of the margin or error $M$. For PaCE, we constrain $\ell$ to be on the order of $\log n$ (we require $k = O\paren{ \log n}$ in \cref{thm: convergence}), thereby complying with the upper bound requirement.

\subsection{Estimation error}

We now bound the estimation error, showing that it converges as $\max(n, T)$ increases. 
For simplicity, we will assume that $n\gtrsim T$ so that we can write all necessary assumptions and statements in terms of $n$ directly. We stress, however, that even without $n\gtrsim T$, the statement of \cref{thm: convergence} also holds by replacing all occurrences of $n$ in \cref{assum:identification}, \cref{assum: error assumption}, and \cref{thm: convergence} with $\max(n,T)$. Proving this only requires changing $n$ to $\max(n,T)$ in the proofs accordingly.

Our main technical contribution is extending the convergence guarantee from \cite{farias2021learning} to a setting with multiple treatment matrices and a convex formulation that centers the rows of $\hat M$ without penalization. We require a set of conditions that extend the assumptions made in \cite{farias2021learning} to multiple treatment matrices.

Throughout the paper, we define the subspace $\mathbf Z := \text{span}\{Z_i,i\in[k]\} + \text{span}\{\alpha \mathbf 1^\top \mid \alpha \in \mathbb{R}^n\}$. 
We next define $D^{\star}\in \R^{k \times k}$ and $\Delta^{\star1} \in \mathbb {R}^{k}$ similarly to $D$ and $\Delta^1$ in \cref{lem:tau-decomposition}, except we replace all quantities associated with $\hat M$ with the corresponding quantity for $M^{\star}$. That is, $D^{\star}_{ij} = \sang{P_{{\mathbf T}^{\star\perp}}(Z_i),P_{{\mathbf T}^{\star\perp}}(Z_j)}$, and  $\Delta^{\star1}_i =  \sang{Z_i, U^{\star}  V^{\star\top} }$ for $i,j\in[k]$. Now, we are ready to introduce the assumptions that we need to bound the estimation error.

\begin{assumption}[Identification] \label{assum:identification}There exist positive constants $c_{r_1}, c_{r_2}, c_{s}$ such that
\begin{enumerate}[label={(\alph*)}, ref={\theassumption(\alph*)}]
\item \label[assumption]{assum:conditions-Z}
$\norm{ ZV^{\star}}_\F^2 + \norm{{ Z^{\top}U^{\star}}}_\F^2 \leq \paren{1-{c_{r_1}}/{\log(n)}}\norm{Z}_\F^2,$ $\forall Z\in \mathbf Z$,
\item \label[assumption]{assumption: spectral}
$
\norm{D^{\star-1}}\norm{\Delta^{\star1}}\sum_{i=1}^k\norm{P_{{\mathbf T}^{\star\perp}}(Z_i)} \leq 1-{c_{r_2}}/{\log n},
$
\item \label[assumption]{assum:conditions-D}
$\sigma_{\min}(D^{\star}) \geq {c_s}/{\log n}.$
\end{enumerate}
\end{assumption}

\cref{assum:conditions-Z} and \cref{assumption: spectral} are the direct generalization of Assumptions 3(a) and 3(b) from \cite{farias2021learning} to multiple treatment matrices. 
% Intuitively, they require that every treatment matrix $Z$ be distinguishable from $M^{\star}$ in the sense that the projection of $Z$ onto the tangent space of $M^{\star}$ cannot be too large. 
Proposition 2 of \cite{farias2021learning} proves that their Assumptions 3(a) and 3(b) are nearly necessary for identifying $\tau^{\star}$. 
%Intuitively, these assumptions ensure that the matrices $Z$ do not lie in the tangent space of $M^{\star}$ in the space of rank-$r$ matrices.
This is because it is impossible to recover $\tau^{\star}$ if $M^{\star} + \gamma Z$ were also a matrix of rank $r$ for some $\gamma \neq 0$.
More generally, if $Z$ lies close to the tangent space of $M^{\star}$ in the space of rank-$r$ matrices, recovering $M^{\star}$ is significantly harder (the signal becomes second order and may be lost in the noise).
In our setting with multiple treatment matrices, we further require that there is no collinearity in the treatment matrices; otherwise, we would not be able to distinguish the treatment effect associated with each treatment matrix. Mathematically, this is captured by \cref{assum:conditions-D}.

\begin{assumption}[Noise assumptions] \label{assum: error assumption}
We make the following assumptions about $E$ and $\delta$. 
\begin{enumerate}[label={(\alph*)}, ref={\theassumption(\alph*)}]
% \item \label[assumption]{assum:E i.i.d.}
% The entries of $E$ are independent, zero-mean, sub-Gaussian random variables, and the sub-Gaussian norm of each entry is bounded by $\sigma$,
%\item \label[assumption]{assum:E uncounfounded Z}
%$\abs{\sang{P_{\mathbf T^{\star \perp}}(Z_i),E}} = O(\sigma\sqrt n)$ for all $i\in[k]$,
\item \label[assumption]{assum:delta singular value}
$\norm{E},\norm{\delta} = O(\sigma\sqrt{n})$,
\item \label[assumption]{assum:E Z}
$\abs{\sang{Z_i, E}} =O( \sigma\sqrt n)$ for all $i\in[k]$,
\item\label[assumption]{assum:delta Z}
$\abs{\sang{Z_i, \delta}} =O( \sigma\sqrt n)$ for all $i\in[k]$.
\end{enumerate}
\end{assumption}

The assumptions involving $E$ are mild: they would be satisfied with probability at least $1-e^{-n}$ if $E$ has independent, zero-mean, sub-Gaussian entries with sub-Gaussian norm bounded by $\sigma$ (see \cref{lemma: i.i.d. E satisfies assumptions}), an assumption that is canonical in matrix completion literature \cite{farias2021learning}.
Both \cref{assum:delta singular value} and \cref{assum:E Z} were also used in \cite{farias2021learning}.
%\cref{assum:E i.i.d.} is canonical in matrix completion literature \cite{farias2021learning}. 
Because $\delta$ is a zero-mean matrix that is zero outside the support of $\sum_{i\in[k]}Z_i$, the assumption $\norm{\delta} = O(\sigma\sqrt{n})$ is mild. In fact, it is satisfied with high probability if the entries of $\delta$ are sub-Gaussian with certain correlation patterns \cite{moon2016linear}. 
Furthermore, we note that if there is only one treatment matrix $(k=1)$ or if all treatment matrices are disjoint, then \cref{assum:delta Z} is guaranteed to be satisfied. This is because, for each $i\in[k]$, $\delta_i$ is a zero-mean matrix that is zero outside of the support of $Z_i$. Consequently, in these scenarios, we simply have $\abs{\sang{Z_i, \delta}} =\abs{\sang{Z_i, \delta_i}} = 0$.

We denote by $\sigma_{\max}$ and $\sigma_{\min}$ the largest and smallest singular values of $P_{{\mathbf 1}^\perp}(M^{\star})$ respectively, and we let $\kappa := \sigma_{\max}/\sigma_{\min}$ be the condition number of $P_{{\mathbf 1}^\perp}(M^{\star})$.

\begin{theorem}\label{thm: convergence} 
Suppose that \cref{assum:identification} and \cref{assum: error assumption} are satisfied, $k = O\paren{ \log n}$, and $\frac{\sigma\sqrt n}{\sigma_{\min}} \lesssim \frac{1}{\kappa^2r^2\log^{12.5}(n)}$. Take $\lambda = \Theta\paren{\sigma \sqrt{nr}\log^{4.5}(n)}$. Then for sufficiently large $n$, 
%with probability at least $1-\frac{1}{n^{8}}$, 
we have 
    \begin{align*}
        \abs{\tau^d_i - \tilde\tau^{\star}_i} \lesssim \frac{\sigma^2 r^2 \kappa n\log^{13.5}(n)}{\sigma_{\min}\|\tilde Z_i\|_\F}
        +\frac{\log^{1.5}(n)}{\|\tilde Z_i\|_\F} \cdot\max_{j\in[k]}\frac{\abs{\sang{P_{{\mathbf T}^{\star\perp}}(\tilde Z_j), P_{{\mathbf T}^{\star\perp}}(E+\delta)}} }{\|\tilde Z_j\|_\F}.
    \end{align*}
\end{theorem}
Note that the bound becomes weaker as $\|\tilde Z_j\|_\F$ shrinks. In other words, more treated entries within a cluster results in a more accurate estimate for that cluster, as expected.
Our rate in \cref{thm: convergence} matches, up to $\log n$ and $r$ factors, the optimal error rate obtained in Theorem 1 of \cite{farias2021learning}, which bounds the error in the setting with one treatment matrix. Our bound accumulated extra factors of $\log n$ because we allowed $k = O(\log n)$ treatment matrices, and we have an extra factor of $r$ because we assumed $\norm{\delta}\lesssim \sigma\sqrt n$, instead of $\norm{\delta}\lesssim \sigma\sqrt{n/r}.$
Proposition 1 from \cite{farias2021learning} proves that their convergence rate is optimal (up to $\log n$ factors) in the ``typical scenario" where $\sigma, \kappa, r = O(1)$ and $\sigma_{\min} = \Theta(n)$. We can conclude the same for our rate in \cref{thm: convergence}.

The convergence rate of PaCE is primarily determined by the slower asymptotic convergence rate of the approximation error. However, our result in \cref{thm: convergence} is of independent interest, as it extends previous results on low-rank matrix recovery with a deterministic pattern of treated entries to allow for multiple treatment matrices.

\section{Empirical evaluation}\label{section: experiments}

To assess the accuracy of PaCE, we demonstrate its performance on semi-synthetic data. Using publicly available panel data as the baseline $M^{\star} + E$, we introduced an artificial treatment and added a synthetic treatment effect to treated entries to generate the outcome matrix $O$. Since the true heterogeneous treatment effects $\mathcal T^{\star}(\mathbf X)$ are known, we can verify the accuracy of various methods. Our results show that PaCE often surpasses alternative methods in accuracy, while using a tree with no more than 40 leaves to cluster observations. Thus, PaCE not only offers superior accuracy, but also provides a simple, interpretable solution.

\paragraph{Data}
To demonstrate the effectiveness of our methodology, we use two publicly available U.S.~economic datasets. The first source comprises monthly Supplemental Nutrition Assistance Program (SNAP) user counts by zip code in Massachusetts, spanning January 2017 to March 2023, obtained from the Department of Transitional Assistance (DTA) public records \cite{mass_dta_reports}. SNAP is a federal aid program that provides food-purchasing assistance to low-income families \cite{usda_snap}. The second data source comprises nine annual demographic and economic data fields for Massachusetts zip codes, U.S.~counties, and U.S.~states from 2005 to 2022, provided by the U.S.~Census Bureau \cite{us_census_bureau_acs1_api}. 

We use this data to conduct experiments on panels of varying sizes. In the first set of experiments, we use the SNAP user count as the baseline, $M^{\star} + E$, with the zip code-level census data for Massachusetts as covariates. The panel size is 517 zip codes over 70 months. The second set of experiments uses the population below the poverty line (from the U.S.~census data) in each state and U.S. territory as the baseline and the remaining census data fields as covariates, with a panel size of 52 states and territories over 18 years. The third set of experiments is similar to the second, but the unit is changed to counties, resulting in a panel size of 770 counties over 18 years. The time period is provided as an additional covariate for the methods that we benchmark against.

\paragraph{Treatment generation} 
To generate the treatment pattern, we varied two parameters: 1) the proportion $\alpha$ of units receiving the treatment $\alpha\in\set{0.05, 0.25, 0.5, 0.75, 1.0}$ and 2) the functional form of the treatment---either adaptive or non-adaptive. In the non-adaptive case, a random $\alpha$ proportion of units receive the treatment for a random consecutive period of time. In the adaptive case, for each time period, we apply the treatment to the $\alpha/2$ proportion of units that show the largest absolute percentage change in outcome in the previous time period. The purpose of the adaptive treatment is to mimic public policies that target either low-performing or high-performing units.

The treatment effect is generated by randomly sampling two covariates and either adding or multiplying them, then normalizing the magnitude of the effects so that the average treatment effect is 20\% of the average outcome. For the SNAP experiments, the treatment effect is added to the outcome to simulate an intervention increasing SNAP usage. For the state and county experiments, the treatment effect is subtracted from the outcome to represent an intervention that reduces poverty.

Each set of parameters is tested on 200 instances, with the treatment pattern and treatment effect being randomly regenerated for each instance.

\paragraph{Benchmark methods}
We benchmark PaCE against the following methods: double machine learning (DML) \cite{chernozhukov2017double, chernozhukov2017orthogonal}, doubly robust (DR) learner \cite{DR}, XLearner \cite{XLearner}, and multi-arm causal forest \cite{athey2019generalized}.\footnote{Additional methods are shown in \cref{appendix: computational}. In the main text, we include the five best methods.} 
We implemented PaCE to split the observations into at most 40 clusters. 
The hyper-parameter $\lambda$ used in the convex programs \eqref{prob:convex_clustering} and \eqref{eq:convex-program} was tuned so that $\hat M$ had a rank of $r=6$, which was chosen arbitrarily and fixed across all experiments.  During the splitting procedure in \cref{alg:clustering}, we utilize binary search on each covariate to find the best split.
See \cref{appendix: computational} for implementation details.

\begin{table}[ht]
\caption{Proportion of instances in which each method results in the lowest nMAE.} 
%Note: the total number of instances per data set is about 4,000 ($=200\times2\times5\times2$).
\label{tab:winners}
\centering
\resizebox{\textwidth}{!}{
\begin{tabular}[t]{lccccccccccccc}
\toprule
 & && \multicolumn{2}{c}{Adaptive} && \multicolumn{5}{c}{Proportion treated} && \multicolumn{2}{c}{Effect}\\
\cline{4-5}\cline{7-11}\cline{13-14}
\addlinespace
 & All  && No & Yes && 0.05 & 0.25 & 0.5 & 0.75 & 1.0 && Add. & Mult.\\
\midrule
SNAP \\
\cline{1-1}
\addlinespace
PaCE & 0.59 &  & 0.58 & 0.60 &  & 0.26 & 0.54 & 0.76 & 0.75 & 0.65 &  & 0.64 & 0.54\\
Causal Forest & 0.16 &  & 0.17 & 0.16 &  & 0.14 & 0.17 & 0.16 & 0.14 & 0.21 &  & 0.02 & 0.30\\
% CausalForestDML & 0.00 &  & 0.00 & 0.00 &  & 0.00 & 0.00 & 0.00 & 0.00 & 0.00 &  & 0.00 & 0.00\\
DML & 0.11 &  & 0.21 & 0.00 &  & 0.12 & 0.13 & 0.08 & 0.09 & 0.11 &  & 0.17 & 0.04\\
% DRLearner & 0.00 &  & 0.00 & 0.00 &  & 0.00 & 0.00 & 0.00 & 0.00 & 0.00 &  & 0.00 & 0.00\\
% ForestDRLearner & 0.00 &  & 0.00 & 0.00 &  & 0.00 & 0.00 & 0.00 & 0.00 & 0.00 &  & 0.00 & 0.00\\
LinearDML & 0.14 &  & 0.04 & 0.24 &  & 0.48 & 0.15 & 0.01 & 0.02 & 0.04 &  & 0.17 & 0.11\\
XLearner & 0.00 &  & 0.00 & 0.00 &  & 0.00 & 0.00 & 0.00 & 0.00 & 0.00 &  & 0.00 & 0.00\\
\addlinespace
State \\
\cline{1-1}
\addlinespace
PaCE & 0.41 &  & 0.39 & 0.42 &  & 0.23 & 0.32 & 0.39 & 0.54 & 0.56 &  & 0.46 & 0.36\\
Causal Forest & 0.05 &  & 0.09 & 0.01 &  & 0.06 & 0.04 & 0.04 & 0.05 & 0.04 &  & 0.04 & 0.05\\
% CausalForestDML & 0.07 &  & 0.07 & 0.07 &  & 0.03 & 0.07 & 0.10 & 0.07 & 0.08 &  & 0.06 & 0.09\\
DML & 0.27 &  & 0.29 & 0.26 &  & 0.41 & 0.36 & 0.27 & 0.17 & 0.16 &  & 0.31 & 0.24\\
% DRLearner & 0.00 &  & 0.00 & 0.00 &  & 0.00 & 0.00 & 0.00 & 0.00 & 0.00 &  & 0.00 & 0.00\\
% ForestDRLearner & 0.01 &  & 0.01 & 0.01 &  & 0.00 & 0.00 & 0.00 & 0.01 & 0.02 &  & 0.00 & 0.01\\
LinearDML & 0.02 &  & 0.01 & 0.04 &  & 0.02 & 0.06 & 0.01 & 0.01 & 0.02 &  & 0.02 & 0.02\\
XLearner & 0.17 &  & 0.13 & 0.20 &  & 0.25 & 0.14 & 0.19 & 0.15 & 0.11 &  & 0.10 & 0.23\\
\addlinespace
County &&&&&&&\\
\cline{1-1}
\addlinespace
PaCE & 0.06 &  & 0.10 & 0.02 &  & 0.16 & 0.08 & 0.04 & 0.02 & 0.01 &  & 0.07 & 0.06\\
Causal Forest & 0.29 &  & 0.38 & 0.20 &  & 0.25 & 0.28 & 0.31 & 0.31 & 0.30 &  & 0.23 & 0.34\\
% CausalForestDML & 0.04 &  & 0.05 & 0.03 &  & 0.00 & 0.01 & 0.03 & 0.05 & 0.11 &  & 0.03 & 0.05\\
DML & 0.35 &  & 0.35 & 0.35 &  & 0.34 & 0.37 & 0.34 & 0.39 & 0.31 &  & 0.47 & 0.24\\
% DRLearner & 0.00 &  & 0.00 & 0.00 &  & 0.00 & 0.00 & 0.00 & 0.00 & 0.00 &  & 0.00 & 0.00\\
% ForestDRLearner & 0.01 &  & 0.00 & 0.01 &  & 0.00 & 0.00 & 0.00 & 0.00 & 0.02 &  & 0.00 & 0.01\\
LinearDML & 0.20 &  & 0.04 & 0.35 &  & 0.22 & 0.22 & 0.21 & 0.16 & 0.17 &  & 0.19 & 0.20\\
XLearner & 0.05 &  & 0.07 & 0.04 &  & 0.03 & 0.03 & 0.06 & 0.07 & 0.07 &  & 0.01 & 0.10\\
\bottomrule
\end{tabular}
}
\end{table}

\paragraph{Evaluation} We represent the estimated treated effects by $\hat{\mathcal T}(\mathbf X)$, a matrix where the element in $z$-th row and $t$-th column is the estimate $\hat{\mathcal T}(X^{zt})$. Accuracy of the estimates is measured using the normalized mean absolute error (nMAE), calculated relative to the true treatment effect $\mathcal T^{\star}(\mathbf X)$:
\begin{align*}
    \text{nMAE}\paren{ \hat{\mathcal T}(\mathbf X), \mathcal T^{\star}(\mathbf X)  } = {    \|\mathcal T^{\star}(\mathbf X) - \hat{\mathcal T}(\mathbf X)}\|_{\text{sum}}/\norm{\mathcal T^{\star}(\mathbf X)}_{\text{sum}}.
\end{align*}

In \cref{tab:winners}, we present the proportion of instances where each method results in the lowest nMAE. We observe that PaCE performs best in the SNAP and state experiments, achieving the highest proportion of instances with the lowest nMAE. However, it often performs worse than DML and Causal Forest in the county experiments.
We hypothesize that this is because the county census data, having many units relative to the number of time periods, `most resembles i.i.d.~data' and has less underlying panel structure. However, the average nMAE achieved by PaCE on this data set is competitive with that of DML and Causal Forest, as shown in Table \ref{tab:avg_nmae}. Given that Causal Forest and DML use hundreds of trees, PaCE, which uses only one regression tree, is still a great option if interpretability is desired.

In \cref{fig:combined_plot}, we show the average nMAE as we vary the proportion of units treated. We only show the results for the SNAP experiments with a non-adaptive treatment pattern, and we refer to \cref{tab:avg_nmae} and \cref{tab:avg_nmae_treated} in \cref{appendix: computational} for the results of additional experiments and for the standard deviations. We observe that while the performance of all methods improves as the proportion of treated entries increases, PaCE consistently achieves the best accuracy on this set of experiments.

\begin{figure}[ht!]
  \centering
    \includegraphics[width=.97\textwidth]{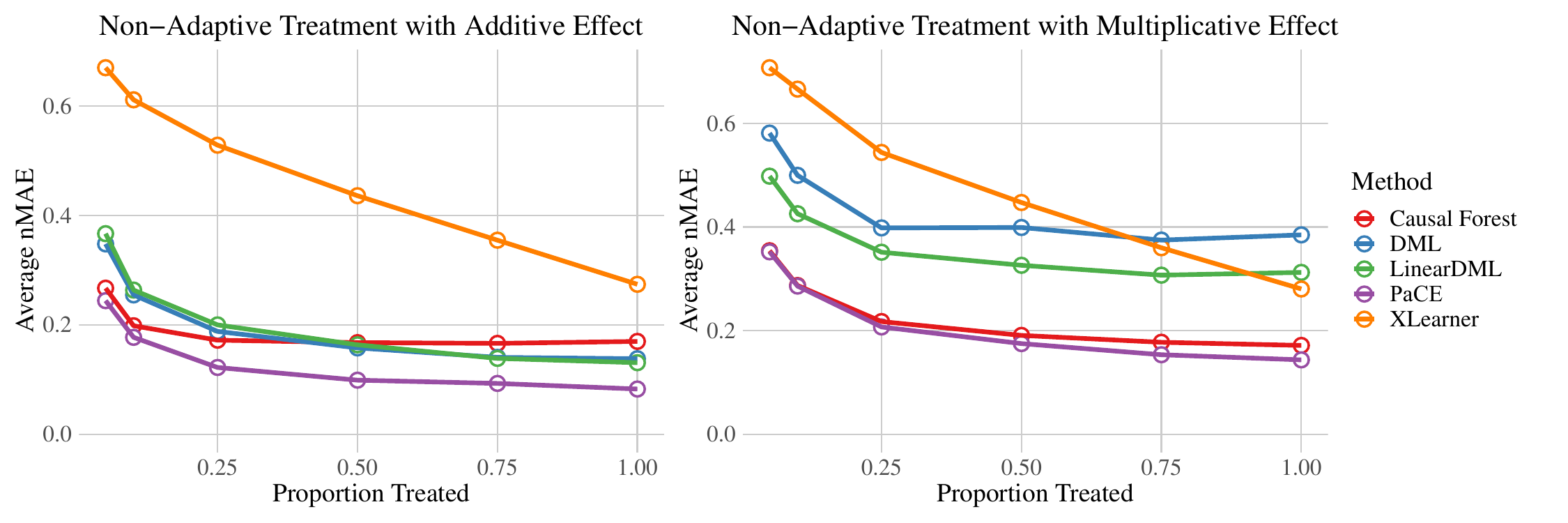}
  \caption{nMAE for SNAP experiments as the proportion of treated units increases.}
  \label{fig:combined_plot}
\end{figure}

\section{Discussion and Future Work}\label{section:discussion}

This work introduces PaCE, a method designed for estimating heterogeneous treatment effects within panel data. Our technical contributions extend previous results on low-rank matrix recovery with a deterministic pattern of treated entries to allow for multiple treatments. We validate the ability of PaCE to obtain strong accuracy with a simple estimator on semi-synthetic datasets.

One interesting direction for future research is to investigate whether the estimates produced by PaCE are asymptotically normal around the true estimates, which would allow for the derivation of confidence intervals. Further computational testing could also be beneficial. Specifically, assessing the estimator's performance across a broader array of treatment patterns and datasets would provide deeper insights into its adaptability and limitations. Expanding these tests to include multiple treatments will help in understanding the estimator's effectiveness in more complex scenarios.

\appendix

\begin{ack}
The authors are grateful to Tianyi Peng for valuable and insightful discussions. We acknowledge the MIT SuperCloud \cite{reuther2018interactive} and Lincoln Laboratory Supercomputing Center for providing HPC resources that have contributed to the research results reported within this paper. This work was supported by the National Science Foundation Graduate Research Fellowship under Grant No.~2141064. Any opinion, findings, and conclusions or recommendations expressed in this material are those of the authors(s) and do not necessarily reflect the views of the National Science Foundation.

% Moreover, you are required to declare
% More information about this disclosure can be found at: \url{https://neurips.cc/Conferences/2024/PaperInformation/FundingDisclosure}.

\end{ack}

\medskip

{\small
% \printbibliography
\bibliographystyle{apalike} % Specify the bibliography style

\bibliography{zotero}        % Specify the bibliography file
}

%%%%%%%%%%%%%%%%%%%%%%%%%%%%%%%%%%%%%%%%%%%%%%%%%%%%%%%%%%%%

\appendix

\section{Proofs for the approximation error}\label{appendix: approx error}

\begin{proof}[Proof of \cref{thm: tree}]

We define the diameter of a given set $C$ of points as the maximum distance between two points in the set $C$:
\begin{equation*}
    \Diam(C) = \sup_{(z_1,t_1),(z_2,t_2)\in C} \|X^{z_1 t_1} - X^{z_2 t_2}\|.
\end{equation*}

Because the treatment effect function $\mathcal T^\star_i$ is Lipschitz continuous for every $i\in[q]$, with Lipschitz constant $L_i$, it suffices to show that, for any leaf $x$ in $\mathbb  T$, the set $C_x$ of points within that leaf has a diameter bounded above as follows: 
\begin{equation*}
    \Diam(C_x) \leq \frac{2\sqrt{ p }}{\ell^s}.
\end{equation*}

In this proof we assume, without loss of generality, that all leaves are at depth exactly $\floor{c\log\ell}$. If a leaf were instead even deeper in the tree, the diameter of the associated cluster would be even smaller. 
We begin by establishing a lower bound for the number of observations $n_a$ in any node $a$ of the tree $\mathbb  T$. This is done by leveraging the fact that every node is at depth at most $c\log\ell$, the upper bound on the number of leaves $\ell$, and the property of $\alpha$-regularity. Specifically, due to $\alpha$-regularity, we have $n_a \geq nT\cdot \alpha^{c\log\ell}$. Hence, 
\begin{align*}
n_a &\geq nT\alpha^{c\log\ell}= nT\ell^{c\log\alpha}\leq nT \paren{\frac{\alpha}{2c_m M}}^{\frac{-1}{c\log \alpha} \cdot c\log\alpha}=  \frac{2 nT }{\alpha} c_m M.
\end{align*}
Rearranging, we have the following for any node $a$ of the tree $\mathbb  T$:
\begin{align}\label{eq: lower bound n}
    c_mM \leq \frac{\alpha n_a}{2nT}. 
\end{align}

Next, for every covariate $j\in[p]$ and any leaf $x$, we aim to establish an upper bound for the length of $C_x$ along the $j$-th coordinate, denoted $\Diam_j(C_x)$. This is defined as follows:
\begin{align*}
    \Diam_j(C_x) := \sup_{(z_1,t_1),(z_2,t_2)\in C_x} \abs{X^{z_1 t_1}_j - X^{z_2 t_2}_j}. 
\end{align*}

Utilizing the lower bound in \cref{assumption: continuous covariates}, the volume $V_a$ of the rectangular hull enclosing the points in a given node $a$ can be upper bounded as follows:
\begin{align}\label{eq: vol orig}
V_a\leq \frac{1}{\ubar c}\paren{\frac{n_a}{nT} +c_mM} \leq \frac{2}{\ubar c}\paren{\frac{n_a}{nT}},
\end{align}
where the last inequality made use of \eqref{eq: lower bound n}.
Given the $\alpha$-regularity of splits, we know that a child of node $a$ has at most $n_a(1-\alpha)$ observations. This implies that the volume $V_r$ of the rectangular hull enclosing the points removed by the split on node $a$ satisfies:
\begin{align}\label{eq: vol removed}
    V_r \geq \frac{1}{\bar c}\paren{\frac{\alpha n_a}{nT} - c_mM} \geq \frac{1}{2\bar c}\paren{\frac{\alpha n_a}{nT}}. 
\end{align}

Putting \eqref{eq: vol orig} and \eqref{eq: vol removed} together, we have $V_r\geq \paren{\frac{\alpha\ubar c}{4\bar c}} V_a$, which means that a split along a given covariate removes at least $\frac{\alpha\ubar c}{4\bar c}$ fraction of the length along that coordinate.

Because the tree is a fair-split tree, for every $(\pi + 1)p$ splits, at least one split is made on every covariate. Hence, the number of splits leading to node $x$ along each coordinate is at least $\floor{\frac{\floor{c\log \ell}}{(\pi+1)p}} \geq \frac{c\log\ell}{(\pi+1)p}-2$.

As a result, for every $j\in[p]$, we have $\Diam_j(C_x) \leq \paren{1-\frac{\alpha\ubar c}{4\bar c}}^{\frac{c\log\ell}{(\pi+1)p}-2}$, which implies
\begin{align*}
    \Diam(C_x) &\leq \sqrt{p}\paren{1-\frac{\alpha\ubar c}{4\bar c}}^{\frac{c}{(\pi+1)p}\log\ell-2}\\
    &\leq \sqrt{p}\paren{1-\frac{\alpha\ubar c}{4\bar c}}^{-2}  \paren{\ell}^{\frac{c}{(\pi+1)p}\log\paren{1-\frac{\alpha\ubar c}{4\bar c}}}\\
    &\leq \sqrt{p}\paren{\frac{3}{4}
    }^{-2}\paren{\ell}^{-\frac{c}{(\pi+1)p}\frac{\alpha\ubar c}{4\bar c}},
\end{align*}
which proves the desired bound. 
\end{proof}

The following lemma is a generalization of the multivariate case of the Dvoretzky–Kiefer–Wolfowitz inequality. Specifically, it provides a bound on the maximum deviation of the empirical fraction of observations within a hyper-rectangle from its expected value. It will be used in the proof of \cref{lemma: iid satisfies assumption}.

\begin{lemma} \label{lemma: dkw box}
Let $X$ be a $d$-dimensional random vector with pdf $f_X$.
Let $X_1,X_2,\dots,X_n$ be an i.i.d.~sequence of $d$-dimensional vectors drawn from this distribution. 
For any $x_1,x_2\in\mathbb  R^d$, define 
\begin{align*}
    G(x_1,x_2):= \Pr[x_1\leq X\leq x_2]
\end{align*}
and 
\begin{align*}
    G_n(x_1,x_2):= \frac{1}{n}\sum_{i=1}^n I_{\set{X_i\in [x_1,x_2]}}.
\end{align*}
Then, for any $\epsilon > 0$, we have both of the following bounds:
\begin{align*}
    &\Pr\sqb{\sup_{x_1,x_2\in \mathbb  R^d}\abs{G_n(x_1,x_2)- G(x_1,x_2)} > \epsilon } \leq 2n^{2d}\exp\paren{-2n\epsilon^2 + 16\epsilon d},\\
    &\Pr\sqb{\sup_{x_1,x_2\in \mathbb  R^d}\abs{G_n(x_1,x_2)- G(x_1,x_2)} > \frac{4d}{n} + \epsilon } \leq 2n^{2d}\exp\paren{-2n\epsilon^2 }.
\end{align*}
\end{lemma}

\begin{proof}
We use ideas from the proof of the main theorem in \cite{devroye1977uniform} and their notation. Because the i.i.d.~samples $X_1,\ldots,X_n$ are drawn from a distribution with a density function, on an almost sure event $\mathcal E=\{X_{i_1 j}\neq X_{i_2 j},i_1\neq i_2\in[n],j\in[d]\}$, they do not share any component in common. 
For $i\in[n]$, let $X_i=\paren{X_{i1},\dots,X_{id}}$, and let $\mathcal Y$ denote the set of random vectors $Y_1,\dots,Y_{n^d}\in \mathbb  R^d$ that are obtained by considering all $n^d$ vectors of the form $\paren{X_{i_11},\dots,X_{i_dd}}$, where $\paren{i_1,\dots,i_d}\in \set{1,\dots,n}^d$. 

Because $G_n$ is a staircase function with flat levels everywhere except at points in $\mathcal Y$, and $G$ is monotonic,
$\abs{G_n(x_1,x_2)- G(x_1,x_2)}$ is maximized when $x_1$ and $x_2$ approach vectors that are in $\mathcal Y$. Under the event $\mathcal E$, we have
\begin{align*}
\sup_{x_1,x_2\in \mathbb  R^d}\abs{G_n(x_1,x_2)- G(x_1,x_2)} \leq \sup_{i,j}\abs{ G_n(Y_i,Y_j) - G(Y_i,Y_j)} +\frac{2d}{n},
\end{align*}
where the $\frac{2d}{n}$ is due to the fact that, for any $i,j$ there may be up to $2d$ different points in $X_1,\dots,X_n$ that share a component with $Y_i$ or $Y_j$ under the event $\mathcal E$. These $2d$ points lie on the perimeter of the hyper-rectangle between $Y_i$ and $Y_j$ and could be included or excluded from the count for $G_n$ depending on whether $x_1$ approaches $Y_i$ (and $x_2$ approaches $Y_j$) from the inside or outside of the hyper-rectangle.

Now, it remains to upper bound the right hand side of the above inequality.
For any pair of indices $i, j\leq n^d$, there exists a subset of indices $\mathcal{I}\subseteq [n]$, corresponding to samples of $X$ that have at least one component in common with either $Y_i$ or $Y_j$. Hence, the samples indexed by $\mathcal{I}$ are not independent from both $Y_i$ and $Y_j$. In order to apply Hoeffding's inequality to the samples, we sample $d':=\abs{\mathcal I}$ additional i.i.d.~random vectors $X_{n+1},\dots,X_{n+d'}$ from the same distribution. These serve to `substitute' those samples $X_l$ for which $l \in \mathcal{I}$. Under $\mathcal E$, the number of indices in $\mathcal I$ satisfies $d'\leq 2d$, and we have
\begin{align*}
    &\sup_{i,j}\abs{ G_n(Y_i,Y_j) - G(Y_i,Y_j)} +\frac{2d}{n} \\
    &= \sup_{i,j}\abs{ \frac{1}{n}\sum_{k=1}^n I_{\set{X_k\in [Y_i,Y_j]}} - G(Y_i,Y_j)} +\frac{2d}{n}\\
    &\leq \sup_{i,j}\abs{ \frac{1}{n}\sum_{k\in[n+d']\setminus\mathcal I}  I_{\set{X_k\in [Y_i,Y_j]}} - G(Y_i,Y_j)} +\frac{1}{n}\abs{\sum_{k\in\mathcal I} I_{X_k\in[Y_i,Y_j]} - \sum_{k=n+1}^{n+d'}I_{X_k\in[Y_i,Y_j]}}+\frac{2d}{n}\\
    &\leq \sup_{i,j}\abs{ \frac{1}{n}\sum_{k\in[n+d']\setminus\mathcal I}  I_{\set{X_k\in [Y_i,Y_j]}} - G(Y_i,Y_j)} +\frac{4d}{n}.
\end{align*}
The first inequality is due to the triangle inequality, and the second inequality follows because $\abs{\mathcal I} =d'\leq 2d$. 

By the independence of both $Y_i$ and $Y_j$ with $X_k$ for $k\in[n+d']\setminus\mathcal I$, we apply Hoeffding's inequality conditioned on the realizations of $Y_i$ and $Y_j$ to establish the following upper bound:
\begin{align*}
    \Pr\sqb{ \abs{ \frac{1}{n}\sum_{k\in[n+d']\setminus\mathcal I}  I_{\set{X_k\in [Y_i,Y_j]}} - G(Y_i,Y_j)} +\frac{4d}{n} \geq \epsilon} &\leq 2\exp\paren{-2n\paren{\epsilon -\frac{4d}{n} }^2}.
\end{align*}

Applying a Union Bound over all $i, j$, we obtain
\begin{align*}
\Pr\sqb{\sup_{x_1,x_2\in \mathbb  R^d}\abs{G_n(x_1,x_2)- G(x_1,x_2)} \geq \epsilon}&\leq 2n^{2d} \exp\paren{-2n\paren{\epsilon -\frac{4d}{n} }^2}\\
&\leq 2n^{2d} \exp\paren{-2n\epsilon^2 + 16\epsilon d}.
\end{align*}
This proves our first bound. To show the second bound, we note that the first line of the above inequality can be rewritten as
\begin{equation*}
    \Pr\sqb{\sup_{x_1,x_2\in \mathbb  R^d}\abs{G_n(x_1,x_2)- G(x_1,x_2)} \geq \frac{4d}{n} + \epsilon}\leq 2n^{2d} e^{-2n\epsilon^2}.
\end{equation*}
\end{proof}

\begin{lemma}\label{lemma: iid satisfies assumption}

For large enough $n$ and $T$, the covariates $X^{zt} \in \mathbb  R^p$ for $z\in[n]$ and $t\in[T]$ satisfy \cref{assumption: continuous covariates} with probability at least $1-\frac{6}{(nT)^2}$ if, for any nonnegative integer $x$, $y$, and $z$ that add up to $p$, every $X^{zt}$ is the concatenation of vectors $w^{zt}\in \mathbb  R^x, u^{zt}\in \mathbb  R^y, v^{zt}\in \mathbb  R^z$, which are independently generated as follows:

\begin{enumerate}
    \item \textbf{Covariates varying by unit and time:}
    Each $w^{zt}$ is i.i.d., sampled from a distribution $W$ on $[0,1]^x$ with density uniformly lower bounded by $\ubar c^{1/3}$ and upper bounded by $\bar c^{1/3}$. 
    
    \item \textbf{Covariates varying only across units:}    
    Draw $n$ i.i.d.~samples from a distribution $U$ on $[0,1]^y$, bounded in density between $\ubar c^{1/3}$ and $\bar c^{1/3}$. These make up $u^{z1}$ for $z\in[n]$, with arbitrary ordering. For every $z$ and $t\in[T]$, we have $u^{zt} =u^{z1}$.

    \item \textbf{Covariates varying only over time:}
    Draw $T$ i.i.d.~samples from a distribution $V$ on $[0,1]^z$ with the same density bounds. These form $v^{1t}$ for $t\in[T]$, with arbitrary ordering. For every $t$ and $z\in[n]$, we have $v^{zt} = v^{1t}$. 
\end{enumerate}
\end{lemma}

\begin{proof}

The proof of the lemma proceeds as follows. We first formally define an event $\mathcal E$, which intuitively states that the maximum deviation between the fraction of samples falling within a specified interval and its expected value is bounded. Our second step is to show that the event occurs with high probability, and our final step is to show that under the event $\mathcal E$, \cref{assumption: continuous covariates} is satisfied.

We begin by defining our notation. For any random vector $D$, we define $\hat{G}^D(a,b)$ as the empirical fraction of observations that fall within a specified interval $(a,b)$, with $a\leq b$ having the same dimension as $D$. The expected value of this quantity is denoted $G^D(a,b) := \Pr[a \leq D \leq b]$. Note that we have
\begin{align*}
    \hat G^U(a,b)= \frac{1}{n}\sum_{z} I_{\set{u^{z1}\in [a,b]}} \quad 
    \hat G^V(a,b)= \frac{1}{T}\sum_{t} I_{\set{v^{1t}\in [a,b]}}\quad
    \hat G^X(a,b)= \frac{1}{nT}\sum_{z,t} I_{\set{X^{zt}\in [a,b]}}.
\end{align*}

\textbf{Defining event $\mathcal E$.}
Now we will formally define $\mathcal E$, which is the intersection of three events $\mathcal E = \mathcal{E_A}\cap \mathcal{E_B}\cap\mathcal{E_C}$.  
The events $\mathcal{E_A}$ (resp. $\mathcal{E_B}$) is the event that maximum deviation between the fraction of samples from random variable $U$ (resp. $V$) falling within an interval and its expected value is bounded.
To simplify our notation, we let $\epsilon:=\sqrt{\frac{(p+1)\ln (nT)}{\min(n,T)}}$. Mathematically, 
    \begin{align*}
        \mathcal{E_A}:=  \set{\sup_{u^{(1)}\leq u^{(2)}}\abs{\hat G^U(u^{(1)}, u^{(2)}) - G^U(u^{(1)}, u^{(2)})} \leq \epsilon +\frac{4y}{n}}, \\
        \mathcal{E_B}:= \set{\sup_{v^{(1)}\leq v^{(2)}} \abs{\hat G^V(v^{(1)}, v^{(2)}) - G^V(v^{(1)}, v^{(2)})} \leq \epsilon+\frac{4z}{T}}.
    \end{align*}

We will now turn to defining event $\mathcal{E_C}$, which intuitively bounds the deviation for the $W$ samples. In the event $\mathcal{E_C}$, rather than considering all of the samples of $W$, we are only considering the samples of $W$ for which the corresponding $U$ and $V$ samples satisfy given constraints. That is, $\mathcal{E_C}$ is the event that for \textit{any} constraints on $U$ and $V$, the maximum deviation between the following is bounded: 1) the proportion of samples with $U$ and $V$ satisfying the constraint, that have $W$ fall within a specified interval and 2) the expected fraction of $W$ samples falling within a specified interval. We will define the event $\mathcal{E_C}$ formally below.

To define event $\mathcal{E_C}$, we begin by defining $S_U$ (resp. $S_V$), which is the set of all possible subsets of units (resp. time periods) selected by some constraint on $U$ (resp. $V$).
Given the random variables $u^{z1}$ for $z\in[n]$, we denote by $\tilde u_1,\ldots,\tilde u_{n^y}$ all $n^y$ vectors formed by selecting its $i$-th coordinate (for each $i\in[y]$) from this set of observed values: $\{u^{z1}_i,z\in[n]\}$. We similarly define $\tilde v_1,\ldots,\tilde v_{T^z}$ using the variables $v^{1t}$ for $t\in[T]$. 
Next, for any choice of indices $i_1,i_2\in[n^y]$ such that $\tilde u_{i_1}\leq \tilde u_{i_2}$, we denote by $\mathcal Z(i_1,i_2)$ the subset of units that lie in the hyper-rectangle defined by $\tilde u_{i_1}$ and $\tilde u_{i_2}$, that is
\begin{equation*}
    \mathcal Z(i_1,i_2) := \set{z\in[n], \tilde u_{i_1} \leq u^{z1} \leq \tilde u_{i_2}}.
\end{equation*}
We define $\mathcal T(j_1,j_2)$ similarly, for any $j_1,j_2\in[T^z]$ with $\tilde v_{j_1}\leq \tilde v_{j_2}$. 
Finally, $S_U:=\{\mathcal Z(i_1,i_2): i_1,i_2\in[n^y],\tilde u_{i_1}\leq \tilde u_{i_2}\}$ and $S_V$ is similarly defined. We are now ready to define the last event:
% \begin{align*}
%     \mathcal{E_C}:= \set{ \sup_{w^{(1)}\leq w^{(2)}}  \abs{\frac{1}{|\mathcal Z| |\mathcal T|} \sum_{z\in\mathcal Z, t\in\mathcal T} I_{w^{zt}\in [w^{(1)},w^{(2)}]} - G^W(w^{(1)},w^{(2)})} \leq \frac{nT\epsilon + 4x}{|\mathcal Z| |\mathcal T|} \forall \mathcal Z\in S_U,\mathcal T\in S_V}.
% \end{align*} % before changing it to two lines
\begin{align*}
    \mathcal{E_C}:= \left\{ \sup_{w^{(1)}\leq w^{(2)}}  \abs{\frac{1}{|\mathcal Z| |\mathcal T|} \sum_{z\in\mathcal Z, t\in\mathcal T} I_{w^{zt}\in [w^{(1)},w^{(2)}]} - G^W(w^{(1)},w^{(2)})} \leq \frac{nT\epsilon + 4x}{|\mathcal Z| |\mathcal T|} \right. \\
    \left.\forall \mathcal Z\in S_U,\mathcal T\in S_V\right\}.
\end{align*}

\textbf{Bounding the probability of event $\mathcal E$.}
By \cref{lemma: dkw box}, we directly have
\begin{align*}
\Pr(\mathcal{E_A}) &\geq 1- 2n^{2y}e^{-2n\epsilon^2} \\ \Pr(\mathcal{E_B}) &\geq 1- 2T^{2z}e^{-2T\epsilon^2}.
\end{align*}
    
We now show that $\mathcal{E_C}$ occurs with high probability. Note that conditioned on the realizations of $u^{z1}$ and $v^{1t}$ for $z\in[n]$ and $t\in[T]$, and for any choice of unit and time period subsets $\mathcal Z\in S_U$ and $\mathcal T\in S_V$, all the random variables $w^{zt}$ for $z\in \mathcal Z,t\in \mathcal T$ are still i.i.d.~according to the distribution $W$ since they are sampled independently from samples of $U$ and $V$. As a result, we can apply \cref{lemma: dkw box} to obtain  
    \begin{equation*}
       \sup_{w^{(1)}\leq w^{(2)}} \abs{\frac{1}{|\mathcal Z| |\mathcal T|} \sum_{z\in\mathcal Z, t\in\mathcal T} I_{w^{zt}\in [w^{(1)},w^{(2)}]} - G^W(w^{(1)},w^{(2)})} \leq \frac{nT\epsilon + 4x}{|\mathcal Z| |\mathcal T|},
    \end{equation*}
    with probability at least $1-2(|\mathcal Z||\mathcal T|)^{2x} e^{-2(nT\epsilon)^2/(|\mathcal Z||\mathcal T| )} \geq 1-2(nT)^{2x} e^{-2nT\epsilon^2}$ for any given $\mathcal Z$ and $\mathcal T$ and for any realization of the random variables $u^{z1}$ for $z\in[n]$ and $v^{1t}$ for $t\in[T]$. Taking the union bound, over all choices of indices $i_1,i_2\in [n^y]$ and $j_1,j_2\in[T^z]$ then shows that
    \begin{equation*} 
        \Pr(\mathcal{E_C}\mid u^{zt},v^{zt}; \forall z\in[n],t\in[T]) \geq 1-2n^{2y}T^{2z}(nT)^{2x} e^{-2nT\epsilon^2}.
    \end{equation*}
    In particular, the same lower bound still holds for $\Pr(\mathcal{E_C})$ after taking the expectation over the samples from $U$ and $V$. Putting all the probabilistic bounds together, we obtain
    \begin{equation*}
        \Pr[\mathcal{E_A}\cap \mathcal{E_B}\cap\mathcal{E_C}] \geq 1- 6n^{2(x+y)}T^{2(x+z)} e^{-2\min(n,T)\epsilon^2}.
    \end{equation*}
    Plugging in $\epsilon=\sqrt{\frac{(p+1)\ln (nT)}{\min(n,T)}}$, this event has probability at least $1-\frac{6}{(nT)^2}$.

\textbf{\cref{assumption: continuous covariates} is satisfied under event $\mathcal E$.}
We suppose that event $\mathcal E$ holds for the rest of this proof. Consider an arbitrary hyper-rectangle in $[0,1]^p$ with lower and upper corners $x^{(1)}$ and $x^{(2)}$, where $x^{(1)} \leq x^{(2)}$. For convenience, we decompose $x^{(1)}$ into three distinct vectors: $w^{(1)} \in \mathbb{R}^x$, $u^{(1)} \in \mathbb{R}^y$, and $v^{(1)} \in \mathbb{R}^z$, such that the concatenation of these three vectors forms $x^{(1)}$. Similarly, $x^{(2)}$ is decomposed into vectors $w^{(2)}$, $u^{(2)}$, and $v^{(2)}$. Our goal is to bound the number of observations that simultaneously satisfy the constraints: 
\begin{align*}
    w^{(1)}\leq w^{zt}\leq w^{(2)}\qquad
    u^{(1)}\leq u^{zt}\leq u^{(2)}\qquad
    v^{(1)}\leq v^{zt}\leq v^{(2)}.
\end{align*}
First, the set $\mathcal Z$ of units that satisfies the constraint $u^{(1)}\leq u^{z1}\leq u^{(2)}$ is included within the considered sets $S_U$. 
% Indeed, it suffices to shrink the box constraint $[u^{(1)},u^{(2)}]$ while still including all units $\mathcal Z$: the resulting lower and upper corners can be written as vectors $\tilde u_{i_1},\tilde u_{i_2}$ for some appropriate indices $i_1,i_2\in[n^y]$. 
Similarly, the set $\mathcal T$ of time periods satisfying the constraint $v^{(1)}\leq v^{1t}\leq v^{(2)}$ is also in $S_V$. 
As a result, because $\mathcal{E_C}$ is satisfied, we have
\begin{equation*}
    \abs{\frac{nT}{|\mathcal Z| |\mathcal T|} \hat G^X(x^{(1)}x^{(2)}) - G^W(w^{(1)},w^{(2)})} \leq \frac{nT\epsilon + 4x}{|\mathcal Z| |\mathcal T|}.
\end{equation*}
Multiplying both sides by $\frac{|\mathcal Z| |\mathcal T|}{nT}$, we have
\begin{equation}\label{eq: covariates W}
    \abs{ \hat G^X(x^{(1)},x^{(2)}) - \frac{|\mathcal Z| |\mathcal T|}{nT}G^W(w^{(1)},w^{(2)})} \leq \epsilon + \frac{4x}{nT}.
\end{equation}
We then use the events $\mathcal{E_A}$ and $\mathcal{E_B}$ to bound the number of units and times satisfying their respective constraints:
\begin{align}\label{eq: covariates U, V}
    \abs{\frac{|\mathcal Z|}{n} - G^U(u^{(1)}, u^{(2)})} \leq \epsilon +\frac{4y}{n} , \qquad
    \abs{\frac{|\mathcal T|}{T} - G^V(v^{(1)}, v^{(2)})} \leq \epsilon+\frac{4z}{T}.
\end{align}
Now let $ P:=G^U(u^{(1)},u^{(2)})  G^V(v^{(1)},v^{(2)}) G^W(w^{(1)},w^{(2)})$. Provided that $n$ and $T$ are large enough such that $\frac{4x}{nT},\frac{4y}{n},\frac{4z}{T}\leq \epsilon\leq\frac{1}{4}$, inequalities \eqref{eq: covariates W} and \eqref{eq: covariates U, V} imply
\begin{equation}\label{eq: covariates X}
    \abs{\hat G^X(x^{(1)},x^{(2)}) - P} \leq 7\epsilon.
\end{equation}

Because of the assumed density bounds on $U$, $V$, and $W$, we have
\begin{align*}
    \ubar c\prod_{i=1}^{p} \paren{x_{i}^{(2)}-x_{i}^{(1)}}\leq  P \leq \bar c\prod_{i=1}^{p} \paren{x_{i}^{(2)}-x_{i}^{(1)}}.
\end{align*}
Combining this with \eqref{eq: covariates X} completes the proof.
\end{proof}

\section{\texorpdfstring{Proof of \cref{lem:tau-decomposition}}{Proof of Lemma \ref{lem:tau-decomposition}}}
\label{appendix: tau-decomposition}

\begin{lemma}\label{lemma: projection}
Define the matrix subspace $\mathbf T$ as follows:
\begin{align*}
\mathbf T &= \set{UA^\top + BV^{\top} + a\mathbf{1}^\top \mid A\in \mathbb{R}^{T \times r}, B \in \mathbb{R}^{n \times r}, a \in \mathbb{R}^{n}}.
\end{align*}

The projection onto its orthogonal space is
\begin{align*}
P_{{\mathbf T}^{\perp}}(X) &=(I - UU^{\top}) X\paren{I-\frac{rr^\top }{\norm{r}^2}} (I - VV^{\top}),\quad \text{where }  r = (I-VV^\top )\mathbf{1}
% &=(I - UU^{\top}) X\paren{I-\frac{(I-VV^\top )\mathbf{1}_{T\times T}(I-VV^\top )}{\norm{(I-VV^\top )\mathbf{1}}^2}} (I - VV^{\top}).
\end{align*}
\end{lemma}

\begin{proof}
    For a given matrix $X$, we define $P:=(I - UU^{\top}) X\paren{I-\frac{rr^\top }{\norm{r}^2}} (I - VV^{\top})$. To show that $P=P_{{\mathbf T}^{\perp}}(X)$, it suffices to show that $P\in \mathbf T^\perp$ and $X-P\in\mathbf T$. 
    
    First, note that $U^\top P=0$ and $P V=0$ because of the construction of $P$. This implies that for any matrices $A\in\mathbb R^{T\times r}$ and $B\in\mathbb R^{n\times r}$, we have $\langle P,UA^\top\rangle = \langle P, BV^\top\rangle=0$. Next, we note that
    \begin{equation*}
        P\mathbf 1 = (I - UU^{\top}) X\paren{I-\frac{rr^\top }{\norm{r}^2}} r =0.
    \end{equation*}
    As a result, we also have $\langle P,a\mathbf 1^\top\rangle =0$ for all $a\in\mathbb R^n$. This ends the proof that $P\in \mathbf T^\perp$.

    Next, we compute
    \begin{align*}
        X-P &= U\underbrace{ U^\top X \paren{I-\frac{rr^\top }{\norm{r}^2}} (I - VV^{\top})}_{A^{'\top}} + \underbrace{X \paren{I-\frac{rr^\top }{\norm{r}^2}} V}_{B'} V^\top + X\frac{rr^\top }{\norm{r}^2}\\
        &= UA^{'\top} + \paren{B' - X\frac{r \mathbf 1^\top V}{\norm{r}^2}} V^\top + X\frac{r}{\norm{r}^2} \mathbf 1^\top
    \end{align*}
    Hence, $X-P\in\mathbf T$, which concludes the proof of the lemma.
\end{proof}
% \begin{proof}

% We can find the projection by solving 
% \begin{align*}
% \min_{A, B, a} \norm{X - UA^\top - BV^{\top} - a\mathbf{1}^\top  }_\F^2,
% \end{align*}

% whose first order conditions are:
% \begin{align*}
% \paren{X - UA^\top - BV^{\top} - a\mathbf{1}^\top }^\top U = 0\\
% \paren{X - UA^\top - BV^{\top} - a\mathbf{1}^\top  }V = 0\\
% a = \frac{1}{T}\paren{X - UA^\top - BV^{\top} }\mathbf{1}.
% \end{align*}
    
% Solving, we find the optimal solution:
% \begin{align*}
% A^{\star\top} &= U^{\top} X\paren{I-\frac{rr^\top }{\norm{r}^2}}\\
% B^{\star} &= (I - UU^{\top}) X\paren{I-\frac{rr^\top }{\norm{r}^2}} V\\
% a^{\star} &= X\frac{rr^\top }{\norm{r}^2}.
% \end{align*}

% The projection $P_{{\mathbf T}^{\perp}}(X) = X - UA^{\star\top} - B^{\star}V^{\top} - a^{\star}\mathbf{1}^\top$ can be rearranged into the desired expression.
% \end{proof}

\begin{proof}[Proof of \cref{lem:tau-decomposition}.]
It is known \cite{cai2010singular} that the set of subgradients of the nuclear norm of an arbitrary matrix $X\in \mathbb {R}^{n_1\times n_2}$ with SVD $X = U\Sigma V^\top $ is
\begin{align*}
    \partial \norm{X}_\star = \set{UV^\top  + W : W \in \mathbb {R}^{n_1\times n_2}, U^\top W=0, WV=0, \norm{W}\leq 1}.
\end{align*}

Thus, the first-order optimality conditions of \eqref{eq:convex-program} are
\begin{subequations}\label{eq:convex-conditions}
\begin{align}
\sang{Z_j, O - \hat{M} - \hat m\mathbf{1}^\top- \sum_{i=1}^k \hat{\tau}_iZ_i} &= 0 \qquad \text{for }j=1,2,\dots,k \label{eq:convex-condition-tau}\\
 O - \hat{M}- \hat m\mathbf{1}^\top - \sum_{i=1}^k \hat{\tau}_iZ_i &= \lambda\paren{\hat{U}\hat{V}^\top  + W}\label{eq:convex-condition-M}\\
\hat U^\top W &=0\label{eq:convex-condition-WU}\\
W\hat V&=0\label{eq:convex-condition-WV}\\
\norm{W} &\leq 1 \label{eq:convex-condition-W}\\
\hat m &= \frac{1}{T}\paren{O-\hat M - \sum_{i=1}^k \hat{\tau}_iZ_i} \mathbf{1}.\label{eq:convex-condition-m}
\end{align}
\end{subequations}

Combining equations \eqref{eq:convex-condition-tau} and \eqref{eq:convex-condition-M}, we have 
\begin{align}\label{eq:ZUVT-ZW}
\sang{Z_j, \lambda\paren{\hat{U}\hat{V}^\top  + W}} &= 0  \qquad \text{for }j=1,2,\dots,k
\end{align}

By the definition of our model, we have $O =  M^{\star} + \sum_{i=1}^k \tau^{\star}_i Z_i + \hat{E}.$ We plug this into equation \eqref{eq:convex-condition-M}:
\begin{align*}
 M^{\star}- \hat{M}- \hat m\mathbf{1}^\top + \sum_{i=1}^k \paren{\tau^{\star}_i-\hat{\tau}_i} Z_i  + \hat{E}  &= \lambda\paren{\hat{U}\hat{V}^\top  + W}.
\end{align*}

We apply $P_{\hat {\mathbf T}^{\perp}}(\cdot)$ to both sides, where the formula for the projection is given in \cref{lemma: projection}:
\begin{align*}
    &\underbrace{P_{\hat {\mathbf T}^{\perp}}(M^{\star}+\hat{E}) + \sum_{i=1}^k\paren{\tau^{\star}_i - \hat{\tau}_i}P_{\hat {\mathbf T}^{\perp}}(Z_i)}_{B}  \\
    &= \lambda P_{\hat {\mathbf T}^{\perp}}\paren{W} \\
    &= \lambda (I - \hat U\hat U^{\top}) W\paren{I-\frac{(I-\hat V\hat V^\top )\mathbf{1}\mathbf{1}^\top(I-\hat V\hat V^\top )}{\norm{(I-\hat V\hat V^\top )\mathbf{1}}^2}} (I - \hat V\hat V^{\top}) \\
  &=  \lambda W\paren{I-\frac{\mathbf{1}\mathbf{1}^\top(I-\hat V\hat V^\top )}{\norm{(I-\hat V\hat V^\top )\mathbf{1}}^2}}\\
  &= \lambda W - \lambda W\paren{\frac{\mathbf{1}\mathbf{1}^\top}{T}}.
\end{align*}
The last line makes use of the following claim:
\begin{claim} \label{claim: hat V 1 = 0}
    $\hat V^\top\mathbf1  = 0$.
\end{claim}

To compute $\lambda W \paren{\frac{\mathbf{1}\mathbf{1}^\top}{T}}$, we right-multiply \eqref{eq:convex-condition-M} by $\paren{\frac{\mathbf{1}\mathbf{1}^\top}{T}}$. We note that by \eqref{eq:convex-condition-m}, we have $\paren{O-\hat M - \hat m\mathbf{1}^\top - \sum_{i=1}^k\hat \tau_i Z_i}\mathbf{1} = 0$. Hence, we find
\begin{align*}
\lambda W \paren{\frac{\mathbf{1}\mathbf{1}^\top}{T}} = -\lambda \hat U\hat V^\top \paren{\frac{\mathbf{1}\mathbf{1}^\top}{T}} \overset{(i)}{=} 0,
\end{align*}
where $(i)$ is due to \cref{claim: hat V 1 = 0}. As a result, we can conclude that $B =\lambda W$.

Substituting $\lambda W=B$ into \eqref{eq:ZUVT-ZW}, we obtain for all $j=1,2,\dots,k$:
\begin{align}
\sang{Z_j, \lambda\hat{U}\hat{V}^\top 
+ P_{\hat {\mathbf T}^{\perp}}(M^{\star}+\hat{E}) + \sum_{i=1}^k\paren{\tau^{\star}_i - \hat{\tau}_i}P_{\hat {\mathbf T}^{\perp}}(Z_i) } = 0 
\label{eq:tau}
\end{align}

Rearranging, we have 
\begin{align*}
\sqb{\sang{P_{\hat {\mathbf T}^{\perp}}(Z_i), P_{\hat {\mathbf T}^{\perp}}(Z_j)}}_{i,j\in[k]}\paren{\hat{\tau}-\tau^{\star}}=\sqb{\sang{Z_i, \lambda \hat{U}\hat{V}^\top  + P_{\hat {\mathbf T}^{\perp}}(M^{\star}+\hat{E})}}_{i\in[k]}
\end{align*}
This is equivalent to the error decomposition in the statement of the lemma, completing the proof.
\end{proof}

\begin{proof}[Proof of \cref{claim: hat V 1 = 0}]
Suppose for contradiction that $\hat V^\top\mathbf1 \neq 0$. Then, there is a solution that for which the objective is lower, contradicting the optimality of $\hat M, \hat\tau, \hat m$. That solution can be constructed as follows. Define $V'$ to be the matrix resulting from projecting each column of $\hat V$ onto the space orthogonal to the vector $\mathbf 1$. That is, the $i$-th column of $V'$ is $V'_i = \hat V_i - \mathbf 1 \paren{ \frac{\hat V_i^\top  \mathbf 1}{\mathbf1^\top \mathbf 1}}$. Let $m'$  be the vector such that $m'\mathbf 1^\top = \hat U\hat\Sigma \paren{\hat V - V'}^\top +\hat m\mathbf 1^\top$. The solution $\paren{\hat U\hat\Sigma V'^\top, \hat \tau, m'}$ achieves lower objective value because the value of the expression in the Frobenius norm remains unchanged, but $\norm{\hat U\hat\Sigma V'^\top}_\star < \norm{\hat U\hat\Sigma \hat V^\top}_\star = \norm{\hat M}_\star$. % since the column norms of $\hat V$ is decreased by the projections. 
The inequality is due to the fact that matrix $V'$ is still an orthogonal matrix, but its columns have norm that is at most 1, and at least one of its columns has norm strictly less than 1. 
\end{proof}

\section{\texorpdfstring{Proof of \cref{thm: convergence}}{Proof of Theorem \ref{thm: convergence}}}
\label{appendix: thm convergence}

We aim to bound $\abs{\tau^d_i - \tilde\tau^{\star}_i}$. Throughout this section, we assume that the assumptions made in the statement of \cref{thm: convergence} are satisfied.
To simplify notation, we define $\eta_i:= \|\tilde Z_i\|_\F$. 
We have 
\begin{align}\label{eq: thm step 1}
\begin{split}
    \abs{\tau^d_i - \tilde\tau^{\star}_i} &\leq \norm{\tau^d - \tilde\tau^{\star}} \\
    &=\frac{1}{\eta_i} \norm{\eta_i \tau^d - \tau^{\star}} \\
    &\overset{(i)}{=}\frac{1}{\eta_i} \norm{  D^{-1}(\Delta^2 + \Delta^3) }\\
    &\leq \frac{1}{\sigma_{\min}(D)\eta_i}\paren{\norm{\Delta^2}+\norm{\Delta^3}}, 
\end{split}
\end{align}
where $(i)$ is due to the error decomposition in \cref{lem:tau-decomposition} and the definition of $\tau^d$.

Hence, we aim to upper bound$\norm{\Delta^2}$ and $\norm{\Delta^3}$ and lower bound  $\sigma_{\min}(D)$. The main challenge lies in upper bounding $\norm{\Delta^3}$. In fact, the desired bounds for $\sigma_{\min}(D)$ and $\norm{\Delta^2}$ (presented in the following lemma) are obtained during the process of bounding $\norm{\Delta^3}$.
% To do use, we utilize the following lemma, which extends the assumed lower bound on $\sigma_{\min}(D^{\star})$ from \cref{assum:conditions-D} to $\sigma_{\min}(D)$:
\begin{restatable}{lemma}{LemmaSigmaMinD}
\label{lemma: sigma min D}
    Let $D$ and $\Delta^2$ be defined as in \cref{lem:tau-decomposition}. We have $\sigma_{\min}(D) \geq \frac{c_s}{2\log n}$. Furthermore,
    %with probability at least $1-2n^{-9}$, 
    we have the following for sufficiently large $n$:
    \begin{align*}
    \norm{\Delta^2} \lesssim 
     \frac{\sigma^2 r^{1.5} \kappa  n \log^{6.5}(n) }{\sigma_{\min} } +  \log^{0.5}(n) \cdot\max_{i\in[k]}\frac{\abs{\sang{P_{{\mathbf T}^{\star\perp}}(\tilde Z_i), P_{{\mathbf T}^{\star\perp}}(E+\delta)}}}{ \|\tilde Z_i\|_\F} .
\end{align*}
\end{restatable}
The proof of the above lemma is provided in \cref{subsection: lemma sigma min D}.

We now discuss the strategy for bounding $\norm{\Delta^3}$. In order to control $P_{\hat {\mathbf T}^{\perp}}(M^{\star})$, we aim to show that the true counterfactual matrix $M^{\star}$ has tangent space close to that of $\hat M$. Because we introduced $\hat m$ in the convex formulation, which centers the rows of $\hat M$ without penalization from the nuclear norm term, we will focus on the projection of $M^{\star}$ that has rows with mean zero, $P_{{\mathbf 1}^\perp}(M^{\star})$. 
Recall that $P_{{\mathbf 1}^\perp}(M^{\star}) = U^{\star}\Sigma^{\star}V^{\star\top}$ is its SVD, and define $X^{\star}:= U^{\star}\Sigma^{\star1/2}$ and $Y^{\star}:=V^{\star}\Sigma^{\star1/2}$. We similarly define these quantities for $\hat M$: $\hat X:= \hat U\hat\Sigma^{1/2}$ and $\hat Y:=\hat V\hat\Sigma^{1/2}$. We aim to show that we have $(\hat X, \hat Y)\approx (X^{\star}, Y^{\star})$. This is stated formally in \cref{lemma: main}, which requires introducing a few additional notations first.

Let $g(M,\tau,m)$ denote the convex function that we are optimizing in \eqref{eq:convex-program}. Recall that $(\hat{M}, \hat{\tau}, \hat{m})$ is a global optimum of function $g$. 
Following the main proof idea in \cite{farias2021learning}, we define a non-convex proxy function $f$, which is similar to $g$, except $M$ is split into two variables and expressed as $XY^\top$. 
\begin{align*}
f(X,Y) &= \min_{m\in\mathbb  R^n, \tau\in\mathbb  R^k} \frac{1}{2}\norm{ O - XY^\top -  m\mathbf 1^\top - \sum_{i=1}^k \tau_iZ_i}_\F^2 + \frac{1}{2}\lambda \sang{X, X} + \frac{1}{2}\lambda\sang{Y, Y}.
\end{align*}
Our analysis will relate $(\hat X, \hat Y)$ to a specific local optimum of $f$, which we will show is close to $(X^{\star}, Y^{\star})$. The local optimum of $f$ considered is the limit of the gradient flow of $f$, initiated at $(X^{\star},Y^{\star})$, formally defined by the following differential equation:
\begin{align}\label{eq: gradient flow}
\begin{cases}(\dot X^t,\dot Y^t) = - \nabla f(X^t,Y^t)\\
(X^0,Y^0) = (X^{\star},Y^{\star}).
\end{cases}
\end{align}

We define $H_{X,Y}$ as the rotation matrix that optimally aligns \((X,Y)\) with \((X^{\star},Y^{\star})\). That is, letting \(\mathcal{O}^{r \times r}\) denote the set of $r \times r$ orthogonal matrices, $H_{X,Y}$ is formally defined as follows:
\begin{equation*}
    H_{X,Y} := \argmin_{R\in\mathcal O^{r\times r}} \norm{XR-X^{\star}}^2_\F + \norm{YR-Y^{\star}}^2_\F.
\end{equation*}
Define \(F^{\star}\) to be the vertical concatenation of matrices $X^\star$ and $Y^\star$. That is,  \(F^{\star} = [(X^{\star})^\top, (Y^{\star})^\top]^\top\). We are now ready to present \cref{lemma: main}. 

\begin{lemma}\label{lemma: main}
For sufficiently large $n$,
%the following holds with a probability of at least $1-n^{-9}$. 
the gradient flow of $f$, starting from $(X^{\star},Y^{\star})$, converges to $(X,Y)$ such that 
$X = \hat X R$ and $Y = \hat Y R$ for some rotation matrix $R\in\mathcal O^{r\times r}$. Furthermore, $(X,Y)\in\mathcal B$, where
\begin{align*}
\mathcal{B} &:= \set{(X,Y)\mid \norm{XH_{X,Y}-X^{\star}}_\F^2 + \norm{YH_{X,Y}-Y^{\star}}_\F^2 \leq \rho^2}, \quad \rho:= \frac{\sigma \sqrt{nr}\log^{6}(n)}{\sigma_{\min}}\norm{F^{\star}}_\F.
\end{align*}
\end{lemma}

Let us derive an upper bound on $\rho$.
Because $\norm{X^{\star}}_\F\leq \norm{X^{\star}}\sqrt r = \sqrt{\sigma_{\max}r}$ and $\norm{Y^{\star}}_\F\leq \sqrt{\sigma_{\max}r}$, we can conclude that $\norm{F^{\star}}_\F\lesssim \sqrt{\sigma_{\max}r}$.
Thus, we can upper bound $\rho$ as follows:
\begin{align}\label{eq: rho bound, unsimplified}
    \rho \lesssim \frac{\sigma \sqrt {nr}\log^{6}(n)}{\sigma_{\min}}\sqrt{\sigma_{\max}r}=\frac{\sigma \sqrt{nr}\log^{6}(n)}{\sigma_{\min}}\sqrt{\sigma_{\min} \kappa r}=\frac{\sigma r \kappa^{0.5} \sqrt{n }\log^{6}(n)}{\sqrt{\sigma_{\min} }}.
\end{align}
Using \eqref{eq: rho bound, unsimplified} and the assumption 
$\frac{\sigma\sqrt{n}}{\sigma_{\min}} \lesssim \frac{1}{\kappa^2 r^{2} \log^{12.5}(n)}$ provided in the theorem, we can further upper bound $\rho$ as follows:
\begin{align}\label{eq: rho bound}
   \rho \lesssim \frac{\sqrt{\sigma_{\min}}}{\log^{6.5}(n) \kappa}.
\end{align}

With \cref{lemma: main} and the bound on $\rho$, we can complete the proof of the theorem as follows. 
\begin{align*}
    \norm{\Delta^3} &\leq \sqrt{k} \norm{\Delta^3}_\infty \\
    &\leq \sqrt k \norm{P_{\hat {\mathbf T}^{\perp}}(M^{\star})}_\F\\
    &=\sqrt k \norm{(I-\hat U\hat U^\top)(X^{\star}Y^{\star\top} + m^{\star}\mathbf 1^\top)\paren{I-\frac{ \hat r\hat r^\top}{\hat r^\top \hat r}}(I-\hat V\hat V^\top)}_\F,
\end{align*}
where $\hat r = (I-\hat V\hat V^\top)\mathbf 1$, and the last line comes from the closed form of the projection derived in \cref{lemma: projection}. We can use the fact that $\hat V^\top \mathbf 1 = 0$ from \cref{claim: hat V 1 = 0} to simplify the expression. Note that with $\hat V^\top \mathbf 1 = 0$, $\hat r$ simply evaluates to $\mathbf 1$.
\begin{align*}
   \sqrt k \norm{P_{\hat {\mathbf T}^{\perp}}(M^{\star})}_\F &=\sqrt k \norm{(I-\hat U\hat U^\top)X^{\star}Y^{\star\top}\paren{I-\frac{\mathbf 1\mathbf 1^\top}{T}}(I-\hat V\hat V^\top)}_\F \\
    &=\sqrt k \norm{(I-\hat U\hat U^\top)X^{\star}Y^{\star\top}(I-\hat V\hat V^\top)}_\F,
\end{align*}
where the last step is due to the fact that $X^{\star}Y^{\star\top} \mathbf 1 = 0$ because $X^{\star}Y^{\star\top} = P_{{\mathbf 1}^\perp}(M^{\star})$.

Finally, let $(X,Y)$ be the limit of the gradient flow of $f$ starting from $(X^{\star},Y^{\star})$. By \cref{lemma: main}, 
%with probability at least $1-n^{-9}$, 
we have $X = \hat X R$ and $Y = \hat Y R$ for some rotation matrix $R\in\mathcal O^{r\times r}$. This, combined with the definition of $\hat{\mathbf T}$, implies $P_{\hat {\mathbf T}^{\perp}}(XA^\top) = P_{\hat {\mathbf T}^{\perp}}(BY^\top) = 0$ for any $A\in\mathbb  R^{T\times r}$ and $B\in\mathbb  R^{n\times r}$. Hence, 
%with probability at least $1-n^{-9}$,
\begin{align}\label{eq: bound Delta^3}
\begin{split}
    \norm{\Delta^3}&\leq \sqrt k \norm{(I-\hat U\hat U^\top)(XH_{X,Y}-X^{\star})(YH_{X,Y}-Y^{\star})^\top(I-\hat V\hat V^\top)}_\F\\
    &\leq \sqrt k\norm{XH_{X,Y}-X^{\star}}_\F\norm{YH_{X,Y}-Y^{\star}}_\F\\
    &\lesssim \sqrt k\rho^2\\
    &\lesssim \frac{\sigma^2 r^2 \kappa  n \log^{12.5}(n)}{\sigma_{\min}},
\end{split}
\end{align}
where the last step is due to the fact that $k = O(\log n).$

Now that we have bounded $\sigma_{\min}(D)$, $\norm{\Delta^2}$, and $\norm{\Delta^3}$, we can plug these bounds into \eqref{eq: thm step 1} to complete the proof of the theorem. %Because the upper bound on $\norm{\Delta^3}$ holds with probability at least $1-n^{-9}$, and the upper bound on $\norm{\Delta^2}$ holds with probability at least $1-2n^{-9}$, by the Union bound, they hold simultaneously with probability at least $1-n^{-8}$. 

\section{\texorpdfstring{Proof of \cref{lemma: main}}{Proof of Lemma \ref{lemma: main}}}
\label{appendix: main lemma}
 
Throughout this section, we assume that the assumptions made in the statement of \cref{thm: convergence} are satisfied.
The proof of \cref{lemma: main} is completed by combining the results of the following two lemmas.

\begin{lemma}\label{lemma: stay in B}
Any point $(X,Y)$ on the gradient flow of $f$ starting from $(X^{\star},Y^{\star})$ satisfies $(X,Y)\in\mathcal B$.
\end{lemma}

\begin{lemma}\label{lemma: converges to M_hat}
The limit $(X,Y)$ of the gradient flow of $f$ starting from $(X^{\star},Y^{\star})$ satisfies
$X = \hat X R$ and $Y = \hat Y R$ for some rotation matrix $R\in\mathcal O^{r\times r}$. 
\end{lemma}

% \begin{enumerate}
    % \item  
    % Showing that $f$ has at least one local optimum $(X',Y')$ in $\mathcal B$ with probability at least.
    % \item Showing that $(X',Y')$ is also in $\mathcal R$.
    % \item 
    % Showing that if there is a critical point of $f$ inside of region $\mathcal B$, that one of these critical points is precisely $\hat X, \hat Y$ with probability at least $1-\frac{1}{n^2}$. 
% \end{enumerate}
We note our methodology differs from that of \cite{farias2021learning}. Instead of analyzing a gradient descent algorithm, we analyze the gradient flow of the function $f$. This allows us to simplify the analysis by avoiding error terms due to the discretization of gradient descent.

In this section, we start by deriving the gradient of the function $f$ and examining some properties of its gradient flow in \cref{subsection: gradient of f}. 
Then, we prove some technical lemmas.
We extend \cref{assum:identification} to a broader subset of matrices within the set $\mathcal B$ in \cref{subsection: assumptions to B}, and establish bounds on the noise in \cref{subsection: error}. 
Finally, we complete the proofs of \cref{lemma: stay in B} in \cref{subsection: stay in B} and \cref{lemma: converges to M_hat} in \cref{subsection: converges to M_hat}.

\subsection{\texorpdfstring{The gradient of the function $f$ }{The gradient of the function f}} \label{subsection: gradient of f}
Before we prove \cref{lemma: stay in B}, we need to derive the gradient of the function $f$. 

Define $P_{\mathbf Z}$ to be the projection operator that projects onto the subspace $\mathbf Z$. Note that in the definition of $f(X,Y)$, the quantities $m$ and $\tau$ are chosen to minimize the distance between $O-XY^\top$ and this subspace, measured in terms of Euclidean norm. Hence, we can view $m$ and $\tau$ as coordinates of the projection of $O-XY^\top$ onto this subspace. This observation gives us the following equivalent definition of $f(X,Y)$:
\begin{align*}
f(X,Y) &= \frac{1}{2}\norm{ P_{\mathbf Z^\perp}\paren{O - XY^\top }}_\F^2 + \frac{1}{2}\lambda \sang{X, X} + \frac{1}{2}\lambda\sang{Y, Y}.
\end{align*}

Consider a single entry of the matrix $X$: $X_{ij}$. Note that the expression inside the Frobenius norm is linear in $X_{ij}$. 
Hence, we can rewrite $f(X,Y)$ as follows:
\begin{align*}
f(X,Y) &= \frac{1}{2}\norm{X_{ij}A + B}_\F^2 + \frac{1}{2}\lambda \sang{X, X} + \frac{1}{2}\lambda\sang{Y, Y},
\end{align*}
for some matrices $A$ and $B$ that do not depend on $X_{ij}$ (but may depend on other entries of $X$, and $Y$). 

Now we take the partial derivative of $f$ with respect to $X_{ij}$. Let $E_{ij}$ be the matrix where all elements are zero except for the element in the $i$-th row and $j$-th column, which is 1. Then, 
\begin{align*}
\frac{\partial f}{\partial X_{ij}} &= \sang{X_{ij}A + B, A} + \lambda X_{ij}\\
&\overset{(i)}{=} \sang{ P_{\mathbf Z^\perp}\paren{O - XY^\top }, -P_{\mathbf Z^\perp}\paren{E_{ij}Y^\top}} + \lambda X_{ij}\\
&= -\sang{ P_{\mathbf Z^\perp}\paren{O - XY^\top }, E_{ij}Y^\top} + \lambda X_{ij},
\end{align*}
where $(i)$ is due to the fact that $A$ and $B$ were defined such that $X_{ij}A + B = P_{\mathbf Z^\perp}\paren{O - XY^\top }.$  
Hence,
\begin{align*}
\frac{\partial f}{\partial X}
&= - P_{\mathbf Z^\perp}\paren{O - XY^\top }Y + \lambda X\\
&= - \paren{O - XY^\top - m(X,Y)\mathbf 1^\top - \sum_{i=1}^k \tau_i(X,Y)Z_i}Y + \lambda X,
\end{align*}
where $m(X,Y),\tau(X,Y) := \argmin_{m,\tau} \norm{O - XY^\top - m\mathbf 1^\top - \sum_{i=1}^k \tau_iZ_i}^2_\F$. We will write $m$ and $\tau$ to represent $m(X,Y)$ and $\tau(X,Y)$ below for notational simplicity.

Using $O = X^{\star}Y^{\star\top} + m^{\star}\mathbf{1}^\top + \sum_{i=1}^k \tau^{\star}_i Z_i + \hat{E}$, we can simplify the gradient as follows: 
\begin{align*}
    \frac{\partial f}{\partial X} =  \paren{ XY^\top - X^{\star}Y^{\star\top} + (m-m^{\star})\mathbf 1^\top + \sum_{i=1}^k (\tau_i-\tau_i^{\star})Z_i - \hat E}Y + \lambda X.
\end{align*}

Because $m$ and $\tau$ are coordinates of the projection of $O-XY^\top$ onto the subspace $\mathbf Z$, we have
\begin{equation*}
    m\mathbf1^\top + \sum_{i=1}^k \tau_iZ_i =  P_{\mathbf Z}\paren{O-XY^\top}.
\end{equation*}
Additionally, because $m^{\star}\mathbf1^\top + \sum_{i=1}^k \tau_i^{\star}Z_i = O-X^{\star}Y^{\star\top} - \hat E$, the right hand side quantity is also in the subspace $\mathbf Z$, so we have
\begin{align*}
    &m^{\star}\mathbf1^\top + \sum_{i=1}^k \tau_i^{\star}Z_i  = P_{\mathbf Z}\paren{O-X^{\star}Y^{\star\top} - \hat E}.
\end{align*}
Subtracting the two equations above, we have
\begin{align}\label{eq: replace_m_tau}
    & (m-m^{\star})\mathbf 1^\top + \sum_{i=1}^k (\tau_i-\tau_i^{\star})Z_i = P_{\mathbf Z}\paren{X^{\star}Y^{\star\top}-XY^\top+\hat E}.
\end{align}
Using \eqref{eq: replace_m_tau}, we can rewrite
\begin{align*}
    \frac{\partial f}{\partial X} =  \paren{P_{\mathbf Z^\perp} \paren{XY^\top - X^{\star}Y^{\star\top}-\hat E}}Y + \lambda X.
\end{align*}
Similarly, we can derive
\begin{align*}
\frac{\partial f}{\partial Y}
&= \paren{P_{\mathbf Z^\perp} \paren{XY^\top - X^{\star}Y^{\star\top}-\hat E}}^\top X + \lambda Y.
\end{align*}

\subsubsection{Properties of the gradient flow}

We now prove some properties of the gradient flow of the function $f$ starting from the point $(X^{\star}, Y^{\star})$. 
For simplicity of notation, we define 
\begin{align*}
    D := P_{\mathbf Z^\perp} \paren{XY^\top - X^{\star}Y^{\star\top}-\hat E},
\end{align*}
so we can simplify our gradients as follows: 
\begin{align}\label{eq:gradient_f_formula}
\frac{\partial f}{\partial X} = DY + \lambda X \qquad \text{and} \qquad
\frac{\partial f}{\partial Y} = D^\top X + \lambda Y.
\end{align}

\begin{lemma}\label{lemma: XTX=YTY}
Every point on the gradient flow of function $f(X,Y)$ starting from the point $\paren{X^{\star}, Y^{\star}}$ satisfies $X^\top X=Y^\top Y.$
\end{lemma}
\begin{proof}
Note that at the starting point $\paren{X^{\star}, Y^{\star}}$, we have that ${X^{\star}}^\top X^{\star} = \Sigma^{\star} = {Y^{\star}}^\top Y^{\star}$ as desired. 

Next, we examine the change in the value $\phi(t) :={X^t}^\top X^t-{Y^t}^\top Y^t$ as we move along the gradient flow defined by \eqref{eq: gradient flow}. For convenience, we omit the superscript $t$ in $X^t$ and $Y^t$. Using \eqref{eq: gradient flow}, we have
\begin{align*}
{\begin{pmatrix}
\dot X\\\dot Y
\end{pmatrix}} = -\nabla f\begin{pmatrix}
X\\Y
\end{pmatrix} = - \begin{pmatrix}
    DY +\lambda X\\
    D^\top X + \lambda Y
\end{pmatrix}.
\end{align*}

The derivative of $\phi$ is 
\begin{align*}
\phi'(t)& = \dot X^\top X + {X}^\top \dot X -\dot Y^\top Y - Y^\top \dot Y  \\
&= - \paren{DY+\lambda X}^\top X-X^\top (DY+\lambda X)   + \paren{D^\top X + \lambda Y}^\top Y + Y^\top (D^\top X+\lambda Y)\\
&= -2\lambda\paren{X^\top X - Y^\top Y},
\end{align*}
showing that $\phi'(t) = -2\lambda \phi(t)$.
As a result, recalling that $\phi(0)=0$, we can solve the differential equation to obtain $\phi(t) = e^{-2\lambda t} \phi(0) = 0$ for all $t$. 
\end{proof}

\begin{lemma}\label{lemma: YT1=0}
Every point on the gradient flow of function $f(X,Y)$ starting from the point $\paren{X^{\star}, Y^{\star}}$ satisfies $Y^\top \mathbf 1 = 0.$
Furthermore, if $U\Sigma V^\top$ is the SVD of $XY^\top$, then $V^\top \mathbf 1 = 0.$
\end{lemma}
\begin{proof}
Using $X^{\star}Y^{\star\top} \mathbf 1 = 0$ (because $X^{\star}Y^{\star\top} = P_{{\mathbf 1}^\perp}(M^{\star})$), we have that 
\begin{align*}
    \paren{X^{\star}Y^{\star\top} \mathbf 1}^\top X^{\star}Y^{\star\top} \mathbf 1 = \mathbf 1^\top V^\star\Sigma^{\star 2} V^{\star\top} \mathbf 1 = 0,
\end{align*}
which can only happen if $V^{\star\top} \mathbf 1 = 0$. This implies that $Y^{\star\top} \mathbf1 = 0$. 

Using the same approach as the proof of \cref{lemma: XTX=YTY}, we examine the change in the value of $\phi(t) := {Y^t}^\top \mathbf 1$ as we move along the gradient flow defined by \eqref{eq: gradient flow}. We omit the superscript $t$ in $Y^t$ for notational convenience. Using \eqref{eq: gradient flow}, we have $\dot Y = -D^\top X - \lambda Y$. 
Now we compute the derivative of $\phi(t)$:
\begin{align*}
 \phi'(t)
&= \dot Y^\top\mathbf 1\\
&= -X^\top D\mathbf 1 - \lambda Y^\top \mathbf 1\\
&\overset{(i)}{=} - \lambda Y^\top \mathbf 1\\
&= - \lambda \phi(t).
\end{align*}
The equality $(i)$ follows because $D$ is the projection onto a space orthogonal to $\set{\alpha \mathbf 1^\top\mid \alpha\in\mathbb  R^n}$, effectively centering its rows. Because $D$ has zero-mean rows, $D\mathbf 1 = 0$. 

Note that $\phi(0) = Y^{\star\top} \mathbf1 = 0$. 
Solving the differential equation, we have $\phi(t) = e^{-\lambda t} \phi(0) = 0$ for all $t$. Using the same logic we used to show $V^{\star\top} \mathbf 1 = 0$, we have that $Y^\top \mathbf 1$ implies $V^\top \mathbf 1 = 0$.
% fact that $A:=XY^\top \mathbf 1  = 0$. We have that 
% \begin{align*}
%     A^\top A = \mathbf 1^\top V\Sigma^2 V^\top \mathbf 1 = 0,
% \end{align*}
% which can only happen if $V^\top \mathbf 1 = 0$.
\end{proof}

\subsection{\texorpdfstring{Extending assumptions to $\mathcal{B}$}{Extending assumptions to B}}
\label{subsection: assumptions to B}

We prove a lemma that extends Assumptions \ref{assum:conditions-Z}, \ref{assumption: spectral}, and \ref{assum:conditions-D} to a broader subset of matrices within the set $\mathcal B$, thereby expanding the applicability of the original assumptions. 

We'll begin by proving a useful lemma that provides bounds for the singular values of matrices in the set $\mathcal B$. We show these values are within a constant factor of the singular value range of $X^{\star}$ and $Y^{\star}$, spanning from $\sqrt{\sigma_{\min}}$ to $\sqrt{\sigma_{\max}}$.

\begin{lemma}\label{lemma: sigma_X sigma_Y}
For large enough $n$ and $(X,Y)\in \mathcal B$, we have the following for any $i\in [r]$:
\begin{align*}
    \sigma_i(X),\sigma_i(Y) \in \left[\frac{\sqrt{\sigma_{\min}}}{2}, 2\sqrt{\sigma_{\max}}\right].
\end{align*}
\end{lemma}
\begin{proof}
The singular values of a matrix are not changed by right-multiplying by an orthogonal matrix. Hence, without loss of generality, we suppose that $(X,Y)$ are rotated such that they are optimally aligned with $(X^{\star},Y^{\star})$; in other words, $H_{X,Y}=I$. Then, $(X,Y)\in \mathcal B$ gives us
\begin{align*}
    \norm{X-X^{\star}}_\F \leq  \rho \lesssim \frac{\sqrt{\sigma_{\min}}}{\log^{6.5}(n) \kappa},
\end{align*}
where the last step is due to \eqref{eq: rho bound}.

Then by Weyl's inequality for singular values, for any $i\in[r]$, we have
\begin{align*}
    \sigma_{i}(X) \leq \sigma_{i}(X^{\star}) + \norm{X - X^{\star}} \leq \sigma_1(X^{\star}) + \norm{X - X^{\star}}_\F \leq 2\sqrt{\sigma_{\max}}. 
\end{align*}

We also have
\begin{align*}
    \sigma_{i}(X) \geq \sigma_{i}(X^{\star}) - \norm{X - X^{\star}} \geq \sigma_{r}(X^{\star}) - \norm{X - X^{\star}}_\F \geq \frac{\sqrt{\sigma_{\min}}}{2}.
\end{align*}
The proof for $\sigma_{i}(Y)$ is identical.
\end{proof}

Now, we present the lemma which extends \cref{assum:identification} to a broader subset of matrices in $\mathcal B$. 

\begin{lemma}\label{lem:general-conditions-small-ball}
Suppose that $(X, Y) \in \mathcal B$ and that there exists a rotation matrix $H\in \mathcal O^{r\times r}$ such that $(XH, YH)$ is along the gradient flow of function $f$ starting from the point $(X^{\star},Y^{\star})$. Let $m,\tau$ denote the values that minimize $f(X,Y)$. Let $XY^{\top} = U\Sigma V^{\top}$ be the SVD of $XY^{\top}$. 
Let $\mathbf T_0$ be the span of the tangent space of $XY^\top$, and denote by $\mathbf T$ the span of $\mathbf T_0$ and $\set{\alpha \mathbf 1^\top \mid \alpha \in \mathbb  R^n}$.
Define $\Tilde\Delta^1 \in \mathbb {R}^{k}$ as the vector with components $\Tilde\Delta^1_i =  \sang{Z_i, U  V^{\top} }.$ Define $\Tilde{D}$ to be the matrix with entries $\Tilde{D}_{ij} = \sang{P_{{\mathbf T}^{\perp}}(Z_i),P_{{\mathbf T}^{\perp}}(Z_j)}$.

 Assume \cref{assum:conditions-Z} holds, then for large enough $n$, 
\begin{align}
    \norm{ZV}_\F^2 + \norm{Z^{\top}U}_\F^2 &\leq 1-\frac{c_{r_1}}{2\log(n)}  \label{eq:generalized-condition-1}\\
    %\norm{P_{{\mathbf T}^{\perp}}(Z)}_\F^2 &\geq \frac{c_{r_1}}{2\log(n)}  \label{eq:generalized-PTP-Z-ratio}\\
    \norm{P_{{\mathbf T_0}^{\perp}}(Z)}_\F^2 &\geq \frac{c_{r_1}}{2\log(n)}  \label{eq:generalized-PTP-Z-ratio}\\
     % \norm{P_{{\mathbf T}^{\perp}}(Z) - P_{{\mathbf T}^{\star\perp}}(Z)}_\F &\lesssim \frac{\sigma r\kappa^2\sqrt{n} \log^{2.5}(n)}{\sigma_{\min}} \norm{Z}_\F \label{eq:PTP-PTPstar-Z-F}\\
     \norm{P_{{\mathbf T}^{\perp}}(Z) - P_{{\mathbf T}^{\star\perp}}(Z)}_\star &\lesssim \frac{\rho \sqrt{\kappa r}}{\sqrt{\sigma_{\min}} }%\frac{\sigma r^{1.5}\kappa^2\sqrt{n} \log^{2.5}(n)}{\sigma_{\min}} 
     \label{eq:PTP-PTPstar-Z}
\end{align}
for any $ Z\in \mathbf Z$ that has $\norm{Z}_\F = 1$.

Assume \cref{assumption: spectral} holds, then for large enough $n$, 
\begin{align}\label{eq: generalized spectral}
\|\Tilde D^{-1}\| \norm{\Tilde\Delta^1} \sum_{i=1}^k\norm{P_{{\mathbf T}^{\perp}}(Z_i)} \leq 1-\frac{c_{r_2}}{2\log n}.
\end{align}

Assume \cref{assum:conditions-D} holds, then for large enough $n$,
\begin{align}\label{eq: generalized D}
\sigma_{\min}(\Tilde{D}) \geq \frac{c_s}{2\log n}.
\end{align}

\end{lemma}

\begin{proof}[Proof of \cref{lem:general-conditions-small-ball}.]

We will make use of the following claim throughout this proof:
\begin{claim}\label{claim: davis-khan}
\begin{align*}
    \norm{U^{\star}U^{\star\top}-UU^\top}_\F, \norm{V^{\star}V^{\star\top} - VV^\top}_\F, \norm{U^{\star}V^{\star\top} - UV^\top}_\F\lesssim  \frac{\rho \sqrt{\kappa}}{\sqrt{\sigma_{\min}} } \lesssim\frac{1}{\log^{6.5}(n)}.
\end{align*}
\end{claim}

The proof of \eqref{eq:generalized-condition-1} and \eqref{eq:generalized-PTP-Z-ratio} follow the exact same logic as the proof of (68) and (69) in the proof of Lemma 13 in \cite{farias2021learning}.

\textbf{Proof of \eqref{eq:PTP-PTPstar-Z}.}
We use \cref{lemma: projection} to compute the projections of the treatment matrices. Let $r=(I-VV^\top)\mathbf 1$ and let $r^{\star} = (I-V^{\star}V^{\star\top})\mathbf 1$. By \cref{lemma: YT1=0}, we have $V^\top \mathbf 1 = V^{\star\top} \mathbf 1 = 0$. Hence, $r = r^{\star} = \mathbf 1$, which allows us to simplify as follows: 
\begin{align*}
&\norm{P_{{\mathbf T}^{\perp}}(Z) - P_{{\mathbf T}^{\star\perp}}(Z)}_\star \\
&=\norm{(I-UU^\top)P_{{\mathbf 1}^\perp}(Z)(I-VV^\top) - (I-U^{\star}U^{\star\top})P_{{\mathbf 1}^\perp}(Z)(I-V^{\star}V^{\star\top})}_\star \\
&\leq \norm{(U^{\star}U^{\star\top}-UU^\top)P_{{\mathbf 1}^\perp}(Z)(I-VV^\top)}_\star + \norm{(I-U^{\star}U^{\star\top})P_{{\mathbf 1}^\perp}(Z)(VV^\top-V^{\star}V^{\star\top})}_\star\\
&\overset{(i)}{\lesssim} \norm{U^{\star}U^{\star\top}-UU^\top}_\F\norm{P_{{\mathbf 1}^\perp}(Z)}_\F\sqrt{r} + \norm{P_{{\mathbf 1}^\perp}(Z)}_\F\norm{VV^\top-V^{\star}V^{\star\top}}_\F\sqrt{r}\\
&\leq \norm{U^{\star}U^{\star\top}-UU^\top}_\F\sqrt{r} +\norm{VV^\top-V^{\star}V^{\star\top}}_\F\sqrt{r}\\
&\lesssim  \frac{\rho \sqrt{\kappa r}}{\sqrt{\sigma_{\min}} },
\end{align*}
where $(i)$ is due to $\norm{A}_\star \leq \norm{A}_\F\sqrt{rank(A)}$, the fact that projection matrices $(I-VV^\top)$ and $(I-U^{\star}U^{\star\top})$ have Frobenius norm at most 1. In particular for step $(i)$, we note that the rank of the matrices inside the nuclear norm is $O(r)$ because $U^{\star}U^{\star\top}$, $UU^\top$, $V^{\star}V^{\star\top}$, and $VV^\top$ are all rank $r$ matrices, and we have that $rank(A+B) \leq rank(A) + rank(B)$ and $rank(AB)\leq \min(rank(A), rank(B))$.

\textbf{Proof of \eqref{eq: generalized D}.}
Define $\Delta^D := \Tilde D - D^{\star}$ such that $\Delta^D_{ij} = \sang{P_{{\mathbf T}^{\perp}}(Z_i),P_{{\mathbf T}^{\perp}}(Z_j)} - \sang{P_{{\mathbf T}^{\star\perp}}(Z_i),P_{{\mathbf T}^{\star\perp}}(Z_j)}$ for $i,j\in[k]$. 
We upper bound the magnitude of the entries of $\Delta^D$. For any $i\in[k]$ and $j\in[k]$, we have
\begin{align}
\label{eq: projection of treatment matrices}
\begin{split}
\abs{\Delta^D_{ij}} &= \abs{\sang{P_{{\mathbf T}^{\perp}}(Z_i),P_{{\mathbf T}^{\perp}}(Z_j)} - \sang{P_{{\mathbf T}^{\star\perp}}(Z_i),P_{{\mathbf T}^{\star\perp}}(Z_j)}}\\
&\overset{(i)}{=} \abs{\sang{P_{{\mathbf T}^{\perp}}(Z_i),Z_j} - \sang{P_{{\mathbf T}^{\star\perp}}(Z_i),Z_j}}\\
&= \abs{\sang{P_{{\mathbf T}^{\perp}}(Z_i)-P_{{\mathbf T}^{\star\perp}}(Z_i),Z_j}}\\
&\overset{(ii)}{\leq} \norm{P_{{\mathbf T}^{\perp}}(Z_i)-P_{{\mathbf T}^{\star\perp}}(Z_i)}_\F\\
 &=\norm{(I-UU^\top)P_{{\mathbf 1}^\perp}(Z_i)(I-VV^\top) - (I-U^{\star}U^{\star\top})P_{{\mathbf 1}^\perp}(Z_i)(I-V^{\star}V^{\star\top})}_\F\\
&=\norm{(U^{\star}U^{\star\top}-UU^\top)P_{{\mathbf 1}^\perp}(Z_i)(I-VV^\top) - (I-U^{\star}U^{\star\top})P_{{\mathbf 1}^\perp}(Z_i)(VV^\top-V^{\star}V^{\star\top})}_\F\\
&\leq\norm{U^{\star}U^{\star\top}-UU^\top}_\F +\norm{V^{\star}V^{\star\top} - VV^\top}_\F\\
% &\lesssim \frac{\sigma r\kappa^2\sqrt{n} \log^{2.5}(n)}{\sigma_{\min}}  \\
& \overset{(iii)}{\lesssim} \frac{1}{\log^{6.5}(n)},
\end{split}
\end{align}
The equality $(i)$ is due to the fact that for any projection operator $P$ and matrices $A$ and $B$, we have $\sang{P(A), B} = \sang{P(A),P(B)} = \sang{A, P(B)}$, and $(ii)$ is due to the Cauchy–Schwarz inequality and the fact that the $\norm{Z_i}_\F = 1$, and $(iii)$ follows from \cref{claim: davis-khan}.

By Weyl's inequality, we have $\sigma_{\min}(\Tilde{D}) \geq \sigma_{\min}(D^{\star})  - \sigma_{\max}(\Delta^D).$
% \begin{align*}
% &= \min_{x\in \mathbb {R}^{k}:\norm{x}=1} \norm{\Tilde{D}x}\\
% &=\min_{x\in \mathbb {R}^{k}:\norm{x}=1} \norm{(D^{\star} + \Delta^D)x}\\
% &\geq \min_{x\in \mathbb {R}^{k}:\norm{x}=1} \norm{D^{\star}x} - \norm{\Delta^Dx}\\
% &\geq \min_{x\in \mathbb {R}^{k}:\norm{x}=1} \norm{D^{\star}x} - \max_{y\in \mathbb {R}^{k}:\norm{y}=1} \norm{\Delta^Dy}\\
% &= \sigma_{\min}(D^{\star}) - \sigma_{\max}(\Delta^D)\\
% &\geq \frac{c_s}{\log n} - \sigma_{\max}(\Delta^D).
% \end{align*}
Now we upper bound the second term as follows: 
\begin{equation} \label{eq: Delta^D}
\sigma_{\max}(\Delta^D) = \norm{\Delta^D} \leq k\max_{l,m}\abs{\Delta^D_{l,m}} = O\paren{\frac{1}{\log^{5.5}(n)}}.
\end{equation}
Putting everything together, for large enough $n$,
\begin{align*}
\sigma_{\min}(\Tilde{D}) &\geq \frac{c_s}{\log n} - \sigma_{\max}(\Delta^D)  \geq \frac{c_s}{2\log n}.
\end{align*}

\textbf{Proof of \eqref{eq: generalized spectral}.} 
We wish to bound $ A:= \norm{\Tilde D^{-1}}\norm{\Tilde\Delta^1}\sum_{i=1}^k\norm{P_{{\mathbf T}^{\perp}}(Z_i)}.$
This quantity can be restructured as 
\begin{align}\label{eq: A defined}
    &A= \paren{\norm{ D^{\star-1}}\norm{\Delta^{\star1}}+E_1}\paren{\sum_{i=1}^k\norm{P_{{\mathbf T}^{\star\perp}}(Z_i)} + E_2},
\end{align}
where $D^{\star}$ and $\Delta^{\star1}$ are defined just above \cref{assum:identification}, and $E_1$ and $E_2$ are defined as follows: 
\begin{align*}
    E_1 &:=   \norm{\Tilde D^{-1}}\norm{\Tilde\Delta^1}-\norm{D^{\star-1}}\norm{\Delta^{\star1}}\\
    E_2&:=\sum_{i=1}^k (\norm{P_{{\mathbf T}^{\perp}}(Z_i)}-\norm{P_{{\mathbf T}^{\star\perp}}(Z_i)}).
\end{align*}

We bound these two quantities as follows. 
\begin{claim}\label{claim: E1E2}
    We have that $\abs{E_1} \lesssim  \frac{1}{\log^{2.5} n}$ and $\abs{E_2} \lesssim  \frac{1}{\log^{5.5} n}$.
\end{claim}

To bound $A$, we consider the following additional observations. Firstly, we have $\norm{ D^{\star-1}}\norm{\Delta^{\star1}} \lesssim \log^2 n$. This follows from the fact that $\sigma_{\min}(D^{\star}) \gtrsim \frac{1}{\log n}$, and $\Delta^{\star1}$ is a vector of length $k$ (where $k \lesssim \log n$) with entries that do not exceed 1. Moreover, we have $\sum_{i=1}^k\norm{P_{{\mathbf T}^{\star\perp}}(Z_i)} \leq k \lesssim \log n$ because for every $i\in[k]$, $P_{{\mathbf T}^{\star\perp}}(Z_i)$ has Frobenius norm (and hence spectral norm) at most 1. Now, we are ready to plug all of this into \eqref{eq: A defined} to bound $A$ as follows.
\begin{align*}
    A& \leq
    \norm{ D^{\star-1}}\norm{\Delta^{\star1}}\sum_{i=1}^k\norm{P_{{\mathbf T}^{\star\perp}}(Z_i)} + O\paren{\frac{1}{\log^{1.5} n}}\\
    &\overset{(i)}{\leq}  1-\frac{c_{r_2}}{\log n}  + O\paren{\frac{1}{\log^{1.5} n}}\\
    &\leq  1-\frac{c_{r_2}}{2\log n},
\end{align*}
where $(i)$ is due to \cref{assumption: spectral}.
\end{proof}

\begin{proof}[Proof of \cref{claim: davis-khan}]
First, we note that we have 
\begin{align*}
    \norm{XY^\top - X^{\star}Y^{\star\top}}_\F &=\norm{XH_{X,Y}(YH_{X,Y})^\top - X^{\star}Y^{\star\top}}_\F \\
    &\leq \norm{XH_{X,Y}-X^{\star}}_\F\norm{Y} + \norm{X^{\star}}\norm{YH_{X,Y}-Y^{\star}}_\F\\
    &\lesssim \rho \sqrt{\sigma_{\max}},
\end{align*}
where $\norm{Y}$ is bounded by \cref{lemma: sigma_X sigma_Y}. 
Hence, we can make use of Theorem 2 from \cite{yu2015useful}, combined with the symmetric dilation trick from section C.3.2 of \cite{abbe2020entrywise}, to obtain the following. There exists an orthogonal matrix $ O\in\mathbb  R^{r\times r}$ such that
\begin{align*}
    \norm{ U O - U^{\star}}_\F+\norm{ V O - V^{\star}}_\F \lesssim \frac{\norm{XY^\top - X^{\star}Y^{\star\top}}_\F}{\sigma_r(X^{\star}Y^{\star\top}) - \sigma_{r+1}(X^{\star}Y^{\star\top})} \lesssim \frac{\rho \sqrt{\sigma_{\max}}}{\sigma_{\min}} = \frac{\rho \sqrt{\kappa}}{\sqrt{\sigma_{\min}}}.
\end{align*}
Now we can obtained the desired bound as follows:
\begin{align*}
    \norm{UV^\top - U^{\star}V^{\star\top}}_\F &= \norm{UO(VO)^\top - U^{\star}V^{\star\top}}_\F \\
    &\leq \norm{UO - U^{\star}}_\F \norm{VO} +  \norm{VO - V^{\star}}_\F \norm{U^{\star}}\\
    &\leq \norm{UO - U^{\star}}_\F+\norm{VO - V^{\star}}_\F  \\
    &\lesssim \frac{\rho \sqrt{\kappa}}{\sqrt{\sigma_{\min}} }\\
    &\lesssim  \frac{1}{\log^{6.5}n }  \qquad \text{using \eqref{eq: rho bound}}.
\end{align*}
We can similarly bound both $\norm{UU^\top-U^{\star}U^{\star\top}}_\F$ and $\norm{VV^\top-V^{\star}V^{\star\top}}_\F$. 
\end{proof}

\begin{proof}[Proof of \cref{claim: E1E2}]

$\abs{E_2}$ can be simply bounded as follows:
\begin{align*}
    \abs{E_2}
    &\leq \sum_{i=1}^k\norm{P_{{\mathbf T}^{\perp}}(Z_i)-P_{{\mathbf T}^{\star\perp}}(Z_i)}_\F \lesssim \frac{k}{\log^{6.5}n},
\end{align*}
where the last bound was shown in \cref{eq: projection of treatment matrices}. Because $k=O(\log n)$, we have the desired bound for $E_2$.

It remains to bound $\abs{E_1}$: %Let $\Delta' :=  \Tilde \Delta^1-\Delta^{\star1}$, such that $\Delta'_i = \sang{Z_i,   P_{{\mathbf 1}^\perp}(U V^{\top})-P_{{\mathbf 1}^\perp}({U^{\star}}{V^{\star}}^\top )}$
\begin{align*}
    \abs{E_1} &= \left|\|D^{\star-1}\| \|\Delta^{\star1}\|-\|\Tilde D^{-1}\|  \|\Tilde\Delta^1\| \right|\\
    &\leq \|D^{\star-1}\| \left|\|\Delta^{\star1}\| - \|\Tilde\Delta^1\|\right| +\left|\|D^{\star-1}\| - \|\Tilde D^{-1}\|\right| \|\Tilde\Delta^1\|\\
    &\leq \underbrace{\|D^{\star-1}\| \|\Delta^{\star1} - \Tilde\Delta^1\|}_{A_1} +\underbrace{\|D^{\star-1} - \Tilde D^{-1}\|  \|\Tilde\Delta^1\|}_{A_2}.
\end{align*}

\textit{Bounding $A_1$. } Applying Assumption \ref{assum:conditions-D}, we can bound $A_1$ as follows: 
\begin{align*}
    A_1 &\lesssim \log (n) \norm{\Delta^{\star1} - \Tilde\Delta^1}\\
    &\leq \log (n) k\norm{\Delta^{\star1} - \Tilde\Delta^1}_{\infty}\\
    &= \log (n) k \max_{i\in[k]} \abs{\sang{Z_i, U^{\star}  V^{\star\top}- U  V^{\top}}}\\
    &\leq \log (n) k \norm{U^{\star}  V^{\star\top}- U  V^{\top}}_\F\\
    & \lesssim \frac{1}{\log^{4.5}(n)},
\end{align*}
where the last step makes use of \cref{claim: davis-khan}.

\textit{Bounding $A_2$. }
Because $\Tilde D = D^{\star} + \Delta^D = (I + \Delta^D D^{\star-1})D^{\star}$, we can write 
\begin{align*}
    \Tilde D^{-1} = D^{\star-1} \sum_{k=0}^{\infty}(-\Delta^D D^{\star-1})^k. 
\end{align*}
Note that this quantity is well-defined because 
\begin{align*}
    \norm{\Delta^D D^{\star-1}} \leq \frac{\sigma_{\max}(\Delta^D)}{\sigma_{\min}(D^{\star})} \lesssim \frac{1}{\log^{4.5} n} < 1,
\end{align*}
where the last bound is due to \eqref{eq: Delta^D} and \cref{assum:conditions-D}. 
Now, we have 
\begin{align*}
     D^{\star-1} - \Tilde D^{-1} =  D^{\star-1} \Delta^D D^{\star-1}\sum_{k=0}^{\infty}(-\Delta^D D^{\star-1})^k. 
\end{align*}
We use this to bound $A_2$. 
Using \eqref{eq: Delta^D} and \cref{assum:conditions-D}, and the fact that $\|\Tilde\Delta^1\| \leq k$ because it is a vector of length $k$ with all components less than or equal to 1, we have
\begin{align*}
    A_2\leq \frac{\sigma_{\max}(\Delta^D)}{\sigma_{\min}(D^{\star})^2} \cdot \frac{1}{1- \frac{\sigma_{\max}(\Delta^D)}{\sigma_{\min}(D^{\star})}} \cdot \|\Tilde\Delta^1\| \lesssim  \frac{1}{\log^{2.5} n}. 
\end{align*}
\end{proof}

\subsection{\texorpdfstring{Discussion of $\hat E$}{Discussion of E hat}}
\label{subsection: error}

Recall that the matrix $\hat E$ consists of two components $\hat E = E +  \delta$. $E$ is a noise matrix, and $\delta$ is the matrix of the approximation error.
Refer to \cref{assum: error assumption} for our assumptions on $E$ and $\delta$. 

% In the remainder of this subsection, we will show the following bounds:

% We first start by bounding the spectral norm of the total error matrix $\hat E$.

% \begin{lemma}
% \label{lemma: E hat 1}
% We have $\|\hat E\| \lesssim \sigma\sqrt n$. 
% \end{lemma}

% \begin{proof}[Proof of \cref{lemma: E hat 1}]
%     This is immediate from \cref{assum:delta singular value} since $\|\hat E\|\leq \norm{E}+\norm{\delta}$.
% \end{proof}

We need to ensure that the error $\hat E$ does not significantly interfere with the recovery of the treatment effects. That is, we need to ensure that $\hat E$ is not confounded with the treatment matrices $Z_i$ for $i\in[k]$. This condition is formalized and established in the following lemma.

\begin{lemma}\label{lemma: E hat 2}
 Let $(X, Y) \in \mathcal B$ be along the gradient flow of function $f$ starting from the point $(X^{\star},Y^{\star})$, and let $m, \tau$ denote the values that minimize $f(X,Y)$. 
Let $\mathbf T$ be the span of the tangent space of $XY^\top$ and $\set{\alpha \mathbf 1^\top \mid \alpha \in \mathbb  R^n}$. Then, we have
%with probability at least $1-n^{-10}$, we have
\begin{align*}
    \max_{i\in[k]}\abs{\sang{P_{{\mathbf T}^{\perp}}(Z_i), \hat{E}}} \lesssim \sigma\sqrt {nr}.
\end{align*}
\end{lemma}

\begin{proof}[Proof of \cref{lemma: E hat 2}]

Fix any $i\in[k]$. We recall that $\hat E=E+\delta$. We note that $rank(P_{\mathbf T}(A)) \leq 2r+1$ for any matrix $A$ due to the definition of $\mathbf T$. Then,
%For any $m\in[k]$:
\begin{align*}
    \abs{\sang{P_{{\mathbf T}^{\perp}}(Z_i), \hat E}}&\leq \abs{\sang{Z_i,  E}}+\abs{\sang{Z_i, \delta}} + \abs{\sang{P_{\mathbf T}(Z_i), \hat E}}\\
    &\overset{(i)}{\lesssim} \sigma\sqrt{n} +   \abs{\sang{P_{\mathbf T}(Z_i), \hat E}}\\
    &\overset{(ii)}{\leq} \sigma\sqrt{n} +  \|\hat E\|\norm{P_{\mathbf T}(Z_i)}_\star\\
    &\overset{(iii)}{\leq} \sigma\sqrt{n} + (\norm{E}+\norm{\delta})\sqrt{2r+1}\norm{P_{\mathbf T}(Z_i)}_\F\\
    &\overset{(iv)}{\lesssim} \sigma \sqrt {nr}.
\end{align*}
Here, (i) is due to \cref{assum: error assumption} $\abs{\sang{Z_i, E}} ,\abs{\sang{Z_i, \delta}} \lesssim \sigma\sqrt{n}$; (ii) is due to Von Neumann’s trace inequality; (iii) is due to the inequality $\norm{A}_\star\leq \sqrt{rank(A)} \norm{A}_\F$; and (iv) is due to the fact that  $\norm{Z_i}_\F = 1$, and our assumed upper bound on $\norm{E}$ and $\norm{\delta}$ from \cref{assum: error assumption}. 
\end{proof}

We now argue that the assumptions on the noise matrix $E$ from \cref{assum: error assumption} are mild. The following lemma shows that under standard sub-Gaussianity assumptions, these are satisfied with very high probability.

\begin{lemma}\label{lemma: i.i.d. E satisfies assumptions}
    Suppose that the entries of $E$ are independent, zero-mean, sub-Gaussian random variables, and the sub-Gaussian norm of each entry is bounded by $\sigma$, and $E$ is independent from $Z_i$ for $i\in[k]$. Then, for $n$ sufficiently large, with probability at least $1-e^{-n}$, we have
    \begin{equation*}
        \norm{E} \lesssim \sigma \sqrt n\quad \text{and}\quad \abs{\sang{Z_i, E}} \lesssim \sigma \sqrt n,\quad \forall i\in[k].
    \end{equation*}
\end{lemma}

\begin{proof}[Proof of \cref{lemma: i.i.d. E satisfies assumptions}]
We start by proving that with probability at least $1-2e^{-n}$, we have $\norm{E}\lesssim \sigma\sqrt n$.
We use the following result bounding the norm of matrices with sub-Gaussian entries.
\begin{theorem}[Theorem 4.4.5, \cite{vershynin2018high}]
    Let $A$ be an $m\times n$ random matrix whose entries $A_{ij}$ are independent, mean zero, sub-Gaussian random variables with sub-Gaussian norm bounded by $\sigma$. Then, for any $t>0$, we have
    \begin{align*}
        \norm{A} \lesssim \sigma(\sqrt{m} +\sqrt{n} + t)
    \end{align*}
    with probability at least $1-2\exp(-t^2)$. 
\end{theorem}
As a result of the above theorem, recalling that $E$ is a $n\times T$ matrix with $T\lesssim n$ and using $t=\sqrt{2n}$, we have $\norm{ E}\lesssim \sigma\sqrt{n}$ with probability at least $1-2\exp(-2n)$.

We next turn to the second claim. 
The general version of Hoeffding's inequality states that: for zero-mean independent sub-Gaussian random variables $X_1,\dots,X_n$, we have 
\begin{align*}
    \Pr\sqb{\abs{\sum_{i=1}^n X_i}\geq t} \leq 2\exp\paren{-\frac{ct^2}{ \sum_{i=1}^n \norm{X_i}_{\psi_2}^2 }},
\end{align*}
for some absolute constant $c>0$. Hence, for any $i\in[k]$:
\begin{align*}
\Pr\sqb{\abs{\sang{Z_i, E}} \geq \sigma\sqrt{2n/c}}\leq 2\exp\paren{-\frac{2\sigma^2n}{ \sigma^2\norm{Z_i}_\F^2 }} \leq 2e^{-2n}. 
\end{align*}

Applying the Union Bound over all $i\in[k]$:
\begin{align*}\label{eq: first term of E hat 2}
\Pr\sqb{\max_{i\in[k]}\abs{\sang{Z_i, E}} \geq \sigma\sqrt{n/c}}\leq 2ke^{-2n}\lesssim 2\log(n) e^{-2n}.
\end{align*}
Applying the Union Bound shows that the two desired claims hold with probability at least $1-2e^{-2n}(1+\log n)\geq 1-e^{-n}$ for sufficiently large $n$.
\end{proof}

\subsection{\texorpdfstring{Proof of \cref{lemma: stay in B}}{Proof of Lemma \ref{lemma: stay in B}}}
\label{subsection: stay in B}

If the gradient flow of $f$ started at $(X^{\star},Y^{\star})$ never intersects the boundary of \( \mathcal{B} \), then we are done. For the rest of the proof, we will suppose that such an intersection exists.
Fix any point \((X, Y)\) at the intersection of the boundary of \( \mathcal{B} \) and the gradient flow of $f$. We aim to show that at this boundary point, the inner product of the normal vector to the region (pointing to the exterior of the region) and the gradient of \( f \) is positive. This property implies that the gradient flow of \( f \), initiated from any point inside \( \mathcal{B} \), cannot exit \( \mathcal{B} \). Proving this characteristic is sufficient to prove \cref{lemma: stay in B}.

We denote by $H$ the rotation matrix $H_{X,Y}$ that optimally aligns the fixed boundary point $(X,Y)$ to $(X^{\star},Y^{\star})$. Consider
\begin{align*}
\mathcal{B}_H = \set{(X',Y')\mid \norm{X'H-X^{\star}}_\F^2 + \norm{Y'H-Y^{\star}}_\F^2 \leq \rho^2}.
\end{align*}
Note that $\mathcal{B}_H$ differs from $\mathcal{B}$ because $H$ is fixed to be the rotation matrix associated with a given point $(X,Y)$. 
$\mathcal B_H$ and $\mathcal B$ are tangent to each other at $(X,Y)$, so the normal vector to $\mathcal B_H$ and $\mathcal{B}$ are co-linear at $(X,Y)$. Thus, it suffices to show that the normal vector to $\mathcal B_H$ at $(X,Y)$, and the gradient of $f$ at $(X,Y)$ have positive inner product. 

To simplify our notation, we define 
\begin{equation}
F:=\left[\begin{array}{c}
X\\
Y
\end{array}\right],\qquad {F}^{\star}:=\left[\begin{array}{c} 
X^{\star}\\
Y^{\star}
\end{array}\right],
\qquad\text{and}\qquad 
    \frac{\partial f}{\partial F} := \left[\begin{array}{c}
\frac{\partial f}{\partial X}\\
\frac{\partial f}{\partial Y}
\end{array}\right].
\nonumber
\end{equation}
Additionally, we define $\Delta_X := XH-X^{\star}$, $\Delta_Y:=YH-Y^{\star}$, and $\Delta_\F := FH-F^{\star}$. 

The normal vector to $\mathcal B$ at $(X,Y)$ is simply the subgradient of $\norm{XH-X^{\star}}_\F^2 + \norm{YH-Y^{\star}}_\F^2$ at $(X,Y)$, which is $2(FH-F^{\star})H^\top$. Hence, we aim to show that the following inner product is positive:
\begin{align*}
\Gamma := \sang{(FH-F^{\star})H^\top, \frac{\partial f}{\partial F}}
=&\sang{\Delta_X,  P_{\mathbf Z^\perp} \paren{XY^\top - X^{\star}Y^{\star\top}-\hat E}YH + \lambda XH} \\
+&\sang{\Delta_Y, P_{\mathbf Z^\perp} \paren{XY^\top - X^{\star}Y^{\star\top}-\hat E}^\top XH + \lambda YH},
\end{align*}
where we used the formula for the gradient of $f$ from Eq~\eqref{eq:gradient_f_formula}.
In the remainder of this subsection, we will use the following shorthand notations for simplicity: we denote \(XH\), \(YH\), and \(FH\) as \(X\), \(Y\), and \(F\) respectively. The above inner product $\Gamma$ becomes:  
\begin{align*}
\Gamma =\underbrace{\sang{\Delta_X Y^\top + X\Delta_Y^\top , P_{\mathbf Z^\perp} \paren{XY^\top - X^{\star}Y^{\star\top}} }}_{A_3}
-\underbrace{\sang{\Delta_X, P_{\mathbf Z^\perp} (\hat{E})Y }-\sang{\Delta_Y, P_{\mathbf Z^\perp} (\hat{E})^\top X}}_{A_1} \\+
\underbrace{\sang{\Delta_X, \lambda X} + \sang{\Delta_Y, \lambda Y}}_{A_2}.
\end{align*}

By \cref{assum:delta singular value}, we have $\|\hat E\|\leq \|E\|+\|\delta\| \lesssim \sigma\sqrt n$. In particular,
\begin{align*}
\abs{A_1}&\lesssim \norm{\Delta_X}_\F\|\hat{E}\|\norm{Y}_\F + \norm{\Delta_Y}_\F\|\hat{E}\|\norm{ X}_\F\\
&\lesssim \|\hat{E}\|\norm{\Delta_\F}_\F\norm{F}_\F \\
&\overset{(i)}{\lesssim}\sigma\sqrt{n}\norm{\Delta_\F}_\F\norm{F^{\star}}_\F,
\end{align*}
where (i) is because $\norm{F}_\F -\norm{F^{\star}}_\F \leq \norm{\Delta_\F}_\F$ by the Triangle Inequality, and $(X,Y)\in\mathcal B$ gives $\norm{\Delta_\F}_\F \leq \rho \lesssim \norm{F^{\star}}_\F$, where the last step is due to the assumed bound on $\frac{\sigma\sqrt{n}}{\sigma_{\min}}$. 

Recalling $\lambda = \Theta\paren{\sigma \sqrt{nr} \log^{4.5}(n)}$, we bound $\abs{A_2}$:
\begin{align}
    |A_2| &\leq \lambda \norm{\Delta_{X}}_\F\norm{X}_\F + \lambda \norm{\Delta_{Y}}_\F \norm{Y}_\F\nonumber\\
    &\leq 2 \lambda \norm{\Delta_\F}_\F  \norm{F}_\F\nonumber\\
    &\lesssim \sigma \sqrt{nr} \log^{4.5}(n)\norm{\Delta_\F}_\F\norm{F^{\star}}_\F.\nonumber
\end{align}

Finally, we will lower bound $A_3$. 
For some $  Z\in \mathbf Z$ with $\norm{ Z}_\F = 1$, we have the following:  
\begin{align*}
A_3 &= \sang{\Delta_X Y^\top + X\Delta_Y^\top , P_{\mathbf Z^\perp} \paren{XY^\top - X^{\star}Y^{\star\top}} }\\
&= \sang{\Delta_X Y^\top + X\Delta_Y^\top , P_{\mathbf Z^\perp}\paren{
\Delta_X Y^\top + X\Delta_Y^\top - \Delta_X\Delta_Y^\top
}}\\
&= \norm{P_{\mathbf Z^\perp}\paren{\Delta_X Y^\top + X\Delta_Y^\top }}_\F^2 - \underbrace{
\sang{\Delta_X Y^\top + X\Delta_Y^\top , P_{\mathbf Z^\perp}\paren{\Delta_X\Delta_Y^\top}}}_{B_0}\\
&= \norm{\Delta_X Y^\top + X\Delta_Y^\top}_\F^2 - \norm{P_{\mathbf Z}\paren{\Delta_X Y^\top + X\Delta_Y^\top }}_\F^2 - B_0\\
&= \norm{\Delta_X Y^\top + X\Delta_Y^\top}_\F^2- \sang{ Z, \Delta_X Y^\top + X\Delta_Y^\top}^2- B_0\\
&\overset{(i)}{=} \norm{ Z}_\F^2\norm{\Delta_X Y^\top + X\Delta_Y^\top}_\F^2 - \sang{P_{\mathbf T_0}( Z), \Delta_X Y^\top + X\Delta_Y^\top}^2- B_0\\
&\geq \norm{ Z}_\F^2\norm{\Delta_X Y^\top + X\Delta_Y^\top}_\F^2 - \norm{P_{\mathbf T_0}( Z)}_\F^2\norm{\Delta_X Y^\top + X\Delta_Y^\top}_\F^2- B_0\\
&= \norm{P_{{\mathbf T_0}^{\perp}}( Z)}_\F^2\norm{\Delta_X Y^\top + X\Delta_Y^\top}_\F^2 - B_0\\
&\overset{(ii)}{\geq} \frac{c_{r_1}}{2\log(n)}\norm{\Delta_X Y^\top + X\Delta_Y^\top}_\F^2 - B_0,
\end{align*}
where $(i)$ is due to the fact that $\norm{ Z}_\F = 1$ and $\Delta_X Y^\top + X\Delta_Y^\top$ is already in the subspace $\mathbf T_0$ by the definition of $\mathbf T_0 = \set{UA^\top + BV^\top \mid A\in\mathbb R^{T\times r},B\in\mathbb R^{n\times r}}$ as the tangent space of $XY^\top$; $(ii)$ is due to \cref{lem:general-conditions-small-ball}. Note that \cref{lem:general-conditions-small-ball} applies because $\mathbf T_0$ is the tangent space of $XY^\top$, and $(X, Y)$ was fixed to be on the gradient flow.  
Furthermore, we can bound $\abs{B_0}$ as follows:
\begin{align*}
\abs{B_0} &\leq \norm{\Delta_X}_\F\norm{\Delta_Y}_\F \paren{\norm{\Delta_X Y^\top}_\F + \norm{X \Delta_Y^\top}_\F}\\
&\leq \norm{\Delta_X}_\F\norm{\Delta_Y}_\F \paren{\norm{\Delta_X}_\F\norm{Y} + \norm{\Delta_Y}_\F\norm{X}}\\
&\overset{(i)}{\lesssim} \norm{\Delta_\F}_\F^3 \norm{F^{\star}}\\
&\leq \rho^2\norm{\Delta_\F}_\F \norm{F^{\star}}\\
&\overset{(ii)}{\lesssim} \frac{\sigma^2 r^2 \kappa n\log^{12}(n)}{\sigma_{\min} }\norm{\Delta_\F}_\F \norm{F^{\star}}\\
&\lesssim \sigma \sqrt n \norm{\Delta_\F}_\F \norm{F^{\star}},
\end{align*}
where $(i)$ makes use of \cref{lemma: sigma_X sigma_Y}, and $(ii)$ follows from the bound on $\rho$ from \eqref{eq: rho bound, unsimplified}, and the last step follows from the bound on $\frac{\sigma\sqrt n}{\sigma_{\min}}$.

Now, we are ready to put everything together to bound the inner product:
\begin{align*}
\Gamma & = A_3 - A_1 + A_2\\
&\geq \frac{c_{r_1}}{2\log(n)}\norm{\Delta_X Y^\top + X\Delta_Y^\top}_\F^2 - B_0 - A_1 + A_2
\end{align*}
Putting previous bounds together, we have
\begin{align*}
    \abs{A_1}+\abs{A_2}+\abs{B_0} \lesssim \sigma \sqrt{nr} \log^{4.5}(n)\norm{\Delta_\F}_\F\norm{F^{\star}}_\F.
\end{align*}

On the other hand, we can lower bound the positive term of $\Gamma$ using the following lemma.
\begin{lemma}\label{lemma: lower bound positive term}
Consider a point $(X,Y)$ on the gradient flow of function $f$ starting from the point $\paren{X^{\star},Y^{\star}}$, such that $(X,Y)\in\mathcal B$. Then,
\begin{align*}
    \norm{\Delta_X Y^\top + X\Delta_Y^\top}_\F^2 \geq \frac{\sigma_{\min}}{4}\norm{\Delta_\F}^2_\F.
\end{align*}
\end{lemma}

By \cref{lemma: lower bound positive term} we have
\begin{align*}
    \frac{c_{r_1}}{2\log(n)}\norm{\Delta_X Y^\top + X\Delta_Y^\top}_\F^2 \geq \frac{c_{r_1}}{2\log(n)}\cdot \frac{\sigma_{\min}}{4}\norm{\Delta_\F}^2_\F   
\end{align*}
We plug in $\norm{\Delta_\F}_\F =\rho$ because we assumed that $(X,Y)$ is at the border of region $\mathcal B$.
Hence,
\begin{align*}
    \frac{c_{r_1}}{2\log(n)}\norm{\Delta_X Y^\top + X\Delta_Y^\top}_\F^2 &\gtrsim \frac{\sigma_{\min}}{\log(n)} \norm{\Delta_\F}_\F \cdot \rho  \\
    &=\sigma\sqrt{nr}\log^5(n)\norm{\Delta_\F}_\F\norm{F^{\star}}_\F.   
\end{align*}
Finally, for large enough $n$,
\begin{align*}
    \Gamma & \gtrsim \sigma\sqrt{nr}\log^5(n)\norm{\Delta_\F}_\F\norm{F^{\star}}_\F - \sigma \sqrt{nr} \log^{4.5}(n)\norm{\Delta_\F}_\F\norm{F^{\star}}_\F >0. 
\end{align*}

\begin{proof}[Proof of \cref{lemma: lower bound positive term}]
Claim 11 in \cite{farias2021learning} states
\begin{align}
\norm{\Delta_{X}Y^{\top}+X\Delta_{Y}^{\top}}_\F^2 \geq \frac{\sigma_{\min}}{4}\norm{\Delta_\F}_\F^2  - \frac{\sigma^2}{n^{13}}. \label{eq:DeltaXY}
\end{align}
We can follow the steps of their proof to prove the desired statement. In particular, we note that the term $-\sigma^2/n^{13}$ in the right-hand side of Eq~\eqref{eq:DeltaXY} comes from the fact that they use the bound \begin{equation*}
    \norm{X^\top X-Y^\top Y}_\F\lesssim \frac{\sigma}{\kappa n^{15}}.
\end{equation*}
Because we have, by \cref{lemma: XTX=YTY} that $X^\top X=Y^\top Y$, we do not incur the term $-\sigma^2/n^{13}$ in our lower bound. 
\end{proof}

\subsection{\texorpdfstring{Proof of \cref{lemma: converges to M_hat}}{Proof of Lemma \ref{lemma: converges to M_hat}}}
\label{subsection: converges to M_hat}

Let $g(M,\tau,m)$ denote the convex function that we are optimizing in \eqref{eq:convex-program}. Recall that $(\hat{M}, \hat{\tau}, \hat{m})$ is a global optimum of function $g$. In this subsection, we prove \cref{lemma: converges to M_hat}, which relates $(\hat{M}, \hat{\tau}, \hat{m})$ to a local optimum of $f$. 

We begin by proving the following useful lemma, which is similar to Lemma 20 from \cite{chen2020noisy}, but allows $X,Y$ to have different dimensions. 
\begin{lemma}\label{lemma: X=U Sigma^1/2}
    Consider matrices $X$ and $Y$ such that $X^\top X = Y^\top Y$. There is an SVD of $XY^\top$ denoted by $U\Sigma V^\top$ such that $X = U\Sigma^{1/2}R$ and $Y = V\Sigma^{1/2}R$ for some rotation matrix $R\in\mathcal O^{r\times r}$. 
\end{lemma}
\begin{proof}
Let $X = U_X\Sigma_XV_X^\top$ and $Y =  U_Y\Sigma_YV_Y^\top$ be their respective SVDs, ordering the diagonal components of $\Sigma_X$ and $\Sigma_Y$ by decreasing order. Then, $X^\top X = Y^\top Y$ implies $\Sigma_X=\Sigma_Y$ and that the singular subspaces of $V_X$ and $V_Y$ coincide. 
% Precisely, let the common singular values be $\sigma_1\geq \ldots \geq \sigma_r$. We group these by unique values: let $1=i_1\leq i_1\leq\ldots\leq i_s<n=i_{s+1}$ such that $\sigma_i$ is equal for all $i\in[i_j,i_{j+1})$ for $j\in[s]$, and such that $\sigma_{i_1}>\sigma_{i_2}>\ldots >\sigma_{i_s}$. Then, $X^\top X = Y^\top Y$ shows that for all $j\in[s]$,
% \begin{equation*}
%     Span(v_{X,i},i\in[i_j,i_{j+1})) = Span(v_{Y,i},i\in[i_j,i_{j+1})),
% \end{equation*}
% where $V_X=[v_{X_1},\ldots,v_{X,r}]$ and $V_Y=[v_{Y_1},\ldots,v_{Y,r}]$. As a result, by making an adequate rotation of the matrices $U_Y$ and $V_Y$, 
Hence, there exists an SVD decomposition of $Y=\tilde U_Y \Sigma_Y\tilde V_Y^\top$ such that $V_X=\tilde V_Y$. 
Then, $XY^\top = U_X \Sigma_X^2 \tilde U_Y^\top$. This is an SVD of $XY^\top$, with $U=U_X$, $\Sigma=\Sigma_X^2$, and $V=\tilde U_Y$. Substituting these quantities into the SVD of $X$ and $Y$, we complete the proof that $X=U \Sigma^{1/2} R$ and $Y= V\Sigma^{1/2}R$, where $R=V_X=\tilde V_Y\in\mathcal O^{r\times r}$.
\end{proof}

Let $(X,Y)$ represent the limit of the gradient flow of $f$ from the initial point $(X^{\star},Y^{\star})$. Let $m$ and $\tau$ be the values that minimize $f(X,Y)$. Furthermore, let the SVD of $XY^\top$ be denoted by $U\Sigma V^\top$.

By \cref{lemma: XTX=YTY}, we have that $X^\top X = Y^\top Y$. Then, \cref{lemma: X=U Sigma^1/2} gives us 
\begin{align}\label{eq: X,Y perp}
    (I-UU^\top) X = 0\qquad\text{and}\qquad (I-VV^\top)Y=0. 
\end{align}

We claim that to prove \cref{lemma: converges to M_hat},  it suffices to prove that $XY^\top = \hat M$.
If we prove that $XY^\top = \hat M$, we would have by \cref{lemma: X=U Sigma^1/2} $X = \hat X R$ and $Y = \hat Y R$ for some rotation matrix $R\in\mathcal O^{r\times r}$. This proves \cref{lemma: converges to M_hat}.

The proof that $XY^\top = \hat M$ consists of two parts. We will first establish that $(XY^\top, \tau, m)$ is also an optimal point of $g$ by verifying the first order conditions of $g$ are satisfied. We will then show that $g$ has a \textit{unique} optimal solution $(\hat M, \hat \tau, \hat m)$. Putting these two parts together establishes that $XY^\top = \hat M$.

\subsubsection{\texorpdfstring{First order conditions of $g$ are satisfied}{First order conditions of g are satisfied}}
\label{subsubsection: FOCs of g}
We will first show that the following first order conditions of $g$ are satisfied at $(XY^\top, \tau, m)$. 
\begin{subequations}
\begin{align}
\sang{Z_l, O - XY^\top -  m\mathbf{1}^\top- \sum_{i=1}^k {\tau}_iZ_i} &= 0 \qquad \text{for }l=1,2,\dots,k \label{eq:convex-condition-tau2}\\
 O - XY^\top- m\mathbf{1}^\top - \sum_{i=1}^k \tau_iZ_i &= \lambda\paren{UV^\top  + W}\label{eq:convex-condition-M2}\\
 U^\top W &=0\label{eq:convex-condition-WU2}\\
W V&=0\label{eq:convex-condition-WV2}\\
\norm{W} &\leq 1\label{eq:convex-condition-W2}\\
m &= \frac{1}{T}\paren{O-XY^\top- \sum_{i=1}^k \tau_iZ_i} \mathbf{1}. \label{eq:convex-condition-m2}
\end{align}
\end{subequations}

We select $W := \frac{1}{\lambda} \paren{O - XY^\top- m\mathbf{1}^\top - \sum_{i=1}^k \tau_iZ_i } - UV^\top$. Note that \eqref{eq:convex-condition-tau2} and \eqref{eq:convex-condition-m2} are automatically satisfied given the definition of $\tau$ and $m$, and \eqref{eq:convex-condition-M2} is automatically satisfied by our choice of $W$. 

To show \eqref{eq:convex-condition-WU2} and \eqref{eq:convex-condition-WV2}, we use the fact that $\frac{\partial f}{\partial X}=\frac{\partial f}{\partial Y}=0$:
\begin{align*}
\paren{O - XY^\top - m\mathbf 1^\top - \sum_{i=1}^k \tau_iZ_i}Y = \lambda X \quad\text{and} \quad
\paren{O - XY^\top - m\mathbf 1^\top - \sum_{i=1}^k \tau_iZ_i}^\top X = \lambda Y
\end{align*}
Now, by \cref{lemma: X=U Sigma^1/2}, we can decompose $X=U\Sigma^{1/2}R$ and $Y=V\Sigma^{1/2}R$ where $R$ is a rotation matrix. Right-multiplying the above equations by $R^{-1}\Sigma^{-1/2}$ gives 
\begin{align*}
\paren{O - XY^\top - m\mathbf 1^\top - \sum_{i=1}^k \tau_iZ_i}V = \lambda U\quad\text{and} \quad
\paren{O - XY^\top - m\mathbf 1^\top - \sum_{i=1}^k \tau_iZ_i}^\top U = \lambda V.
\end{align*}
Rearranging, the first equation shows that $WV=0$ and the second shows $U^\top W=0$.

The last step of the proof is to verify \cref{eq:convex-condition-W2}. Using \eqref{eq:convex-condition-WU2} and \eqref{eq:convex-condition-WV2}, we have
\begin{align*}
W &= (I-UU^\top) W (I-VV^\top)\\
&=\frac{1}{\lambda} (I-UU^\top)\paren{O - m\mathbf{1}^\top - \sum_{i=1}^k \tau_iZ_i } (I-VV^\top),
\end{align*}
where the last line is obtained by plugging in our chosen value of $W$ and using \eqref{eq: X,Y perp} to get rid of the $XY^\top$ term. Plugging in $O = X^{\star}Y^{\star\top} + m^{\star}\mathbf{1}^\top + \sum_{i=1}^k \tau^{\star}_i Z_i + \hat{E}$, we have
\begin{align*}
W&=\frac{1}{\lambda} (I-UU^\top)\paren{X^{\star}Y^{\star\top}  +(m^{\star}-m)\mathbf{1}^\top + \sum_{i=1}^k (\tau^{\star}_i-\tau_i)Z_i +\hat E} (I-VV^\top).
\end{align*}
We will use substitution to get rid of the $(m^{\star}-m)\mathbf{1}^\top$ term.  
Because we have $m^{\star} = \frac{M^{\star}\mathbf 1}{T}$ and 
\begin{align*}
    m &= \frac{1}{T}\paren{O-XY^\top- \sum_{i=1}^k \tau_iZ_i} \mathbf{1} = \frac{1}{T}\paren{M^{\star}-XY^\top+ \sum_{i=1}^k (\tau^{\star}_i-\tau_i)Z_i + \hat E} \mathbf{1},
\end{align*}
this implies that
\begin{align*}
        (m^{\star}-m)\mathbf{1}^\top &=\paren{XY^\top- \sum_{i=1}^k (\tau^{\star}_i-\tau_i)Z_i -\hat E} \frac{\mathbf{1}\mathbf{1}^\top  }{T} = P_{\mathbf 1}\paren{XY^\top- \sum_{i=1}^k (\tau^{\star}_i-\tau_i)Z_i -\hat E} .
\end{align*}
Substituting this expression into our expression for $W$, we have
\begin{align*}
W = \frac{1}{\lambda} (I-UU^\top)\paren{X^{\star}Y^{\star\top}  + \sum_{i=1}^k (\tau^{\star}_i-\tau_i)P_{{\mathbf 1}^\perp}(Z_i)  + P_{{\mathbf 1}^\perp}(\hat E)  } (I-VV^\top)
\end{align*}
where the $P_{\mathbf 1}(XY^\top)$ term went away because $(I-UU^\top)X = 0$ by \eqref{eq: X,Y perp}.

By \cref{lemma: YT1=0}, we have $V^\top \mathbf 1 = 0$. This allows us to simplify the closed form expression for the projection given in \cref{lemma: projection}: $P_{{\mathbf T}^{\perp}}(A) = (I-UU^\top)P_{{\mathbf 1}^\perp}(A)(I-VV^\top)$. Hence,
\begin{align*}
W = \frac{1}{\lambda} (I-UU^\top)\ X^{\star}Y^{\star\top}  (I-VV^\top) + \frac{1}{\lambda} P_{{\mathbf T}^\perp}(\hat E)   +  \frac{1}{\lambda} \sum_{i=1}^k (\tau^{\star}_i-\tau_i)P_{{\mathbf T}^{\perp}}(Z_i) .
\end{align*}
We can upper bound its spectral norm as follows:
\begin{align*}
\lambda\norm{W}&\leq\underbrace{\norm{ (I-UU^\top)X^{\star}Y^{\star\top} (I-VV^\top)} }_{A_1} +\underbrace{\norm{\hat E}}_{A_2} + \underbrace{\norm{\tau^{\star}-\tau}\sum_{i=1}^k  \norm{P_{{\mathbf T}^{\perp}}(Z_i) }}_{A_3}. 
\end{align*}

Now, we bound each of these terms separately.

\textbf{Bounding $A_1$.} By \eqref{eq: X,Y perp}, we have $(I-UU^\top) X = 0$ and $(I-VV^\top)Y=0$. Hence,
\begin{align*}
    \norm{ (I-UU^\top)X^{\star}Y^{\star\top} (I-VV^\top)}& = \norm{ (I-UU^\top)(XH_{X,Y}-X^{\star})(YH_{X,Y}-Y^{\star\top}) (I-VV^\top)}\\
    &\leq \norm{XH_{X,Y}-X^{\star}}_\F\norm{YH_{X,Y}-Y^{\star}}_\F\\
    &\lesssim \rho^2\\
    &\overset{(i)}{\lesssim} \frac{\sigma^2 r^2 \kappa n\log^{12}(n)}{\sigma_{\min} }\\
&\lesssim \sigma \sqrt n,
\end{align*}
where $(i)$ follows from the bound on $\rho$ from \eqref{eq: rho bound, unsimplified}, and the last step follows from the bound on $\frac{\sigma\sqrt n}{\sigma_{\min}}$.
    %&\lesssim \sqrt n,
%\end{align*}
%where the last step is due to \eqref{eq: rho bound}.

\textbf{Bounding $A_2$.} $A_2$ is bounded by \cref{assum:delta singular value} which gives $\|\hat E\|\leq \norm{E}+\norm{\delta} \lesssim \sigma\sqrt{n}$.

\textbf{Bounding $A_3$.}
If $(XY^\top, \tau, m)$ satisfy \eqref{eq:convex-condition-tau2}--\eqref{eq:convex-condition-WV2} and \eqref{eq:convex-condition-m2}, then the following decomposition holds due to the same proof as in \cref{lem:tau-decomposition}.
\begin{align*}
    \Tilde D \paren{\tau-\tau^{\star}}
= \lambda\Tilde\Delta^1 +\Tilde \Delta^2 + \Tilde\Delta^3,
\end{align*}
where $\Tilde D \in \R^{k \times k}$ is the matrix with entries $\Tilde D_{ij} = \langle
P_{{\mathbf T}^{\perp}}(Z_i),P_{{\mathbf T}^{\perp}}(Z_j)
\rangle
$
and $\Tilde \Delta^1,\Tilde \Delta^2, \Tilde\Delta^3 \in \mathbb {R}^{k}$ are vectors with components
\begin{align*}
\Tilde\Delta^1_i = 
\sang{
Z_i, U  V^\top} ,\quad
\Tilde\Delta^2_i = 
\sang{Z_i, P_{{\mathbf T}^{\perp}}(\hat{E})},\quad
\Tilde\Delta^3_i = 
\sang{Z_i, P_{{\mathbf T}^{\perp}}(M^{\star})}.
\end{align*}

This leads us to have:
\begin{align*}
    A_3 &= \norm{\Tilde D^{-1}(\lambda\Tilde\Delta^1 +\Tilde \Delta^2 + \Tilde\Delta^3)} \sum_{i=1}^k  \norm{P_{{\mathbf T}^{\perp}}(Z_i) }\\
    &\leq \lambda\norm{\Tilde D^{-1}\Tilde\Delta^1}\sum_{i=1}^k  \norm{P_{{\mathbf T}^{\perp}}(Z_i) } +\paren{\norm{\Tilde D^{-1}\Tilde \Delta^2} + \norm{\Tilde D^{-1}\Tilde\Delta^3}}\sum_{i=1}^k  \norm{P_{{\mathbf T}^{\perp}}(Z_i) }\\
    &\leq \lambda\paren{1-\frac{c_{r_2}}{2\log n}} +\frac{2\log^2 n}{c_s}\paren{\norm{\Tilde \Delta^2} + \norm{\Tilde\Delta^3}},
\end{align*}
where the last inequality made use of \cref{lem:general-conditions-small-ball}.

\begin{claim}\label{claim: Tilde Delta}
    We have 
    \begin{align*}
        \norm{\Tilde\Delta^2}, \norm{\Tilde\Delta^3} \lesssim \sigma \sqrt {nr} \log n. 
    \end{align*}
\end{claim}

With the above claim, we are ready to bound $\norm{W}$.
\begin{align*}
    \norm{W} &\leq \frac{1}{\lambda} \paren{A_1 + A_2 + A_3} \\
    &\leq \frac{1}{\lambda} \paren{A_1 + A_2 +\lambda\paren{1-\frac{c_{r_2}}{2\log n}} +\frac{2\log^2 n}{c_s}\paren{\norm{\Tilde \Delta^2} + \norm{\Tilde\Delta^3}}}.
\end{align*}

%Our bound on $A_2$ holds with probability at least $1-2e^{-n}$ and the bounds from \cref{claim: Tilde Delta} hold with probability at least $1-n^{-10}$. By the Union Bound, the probability that both simultaneously hold true is at least $1-2n^{-10}$ for large enough $n$.
%Hence, we conclude that, with probability at least $1-2n^{-10}$, 
Hence, we conclude that
\begin{align*}
    \norm{W} 
    &\lesssim  1-\frac{c_{r_2}}{2\log n} +  O\paren{\frac{\sigma\sqrt {nr} \log^3 (n)}{\lambda} }\leq 1
\end{align*}
for large enough $n$ and $\lambda = \Theta\paren{\sigma \sqrt{nr}\log^{4.5}(n)}$.

\begin{proof}[Proof of \cref{claim: Tilde Delta}.]
    
\textit{Bounding $\norm{\Tilde\Delta^2}$.} Using \cref{lemma: E hat 2}, we have
\begin{align*}
    \norm{\Tilde\Delta^2} \leq k \max_{i\in[k]} \abs{\sang{P_{{\mathbf T}^{\perp}}(Z_i), \hat{E}}} \lesssim \sigma \sqrt {nr} \log n.
\end{align*}

\textit{Bounding $\norm{\Tilde\Delta^3}$.} $\norm{\Tilde\Delta^3}$ can be bounded by following the exact same steps as the bound of $\norm{\Delta^3}$ \eqref{eq: bound Delta^3}, replacing $\hat {\mathbf T}$, $\hat U$, and $\hat V$ with $\mathbf T$, $U$, and $V$, respectively. In particular, note that \cref{lemma: YT1=0} gives that $V^\top \mathbf 1 = 0$ along the gradient flow that ends at $X,Y$, which allows the same simplifications. 
\begin{align*}
\norm{\Tilde\Delta^3} \lesssim \frac{\sigma^2 r^2 \kappa  n \log^{12.5}(n)}{\sigma_{\min}}\lesssim \sigma \sqrt n,
\end{align*}
where the last step used the assumed bound on $\frac{\sigma\sqrt n}{\sigma_{\min}}$.
\end{proof}

\subsubsection{\texorpdfstring{Function $g$ has a unique minimizer}{Function g has a unique minimizer}}
\label{subsection: g unique minimizer}

In this subsection, we aim to show that the convex function $g(M,\tau,m)$ has a unique minimizer.
Throughout the proof, we fix a global minimum $(\hat M,\hat \tau,\hat m)$ of $g$, such that $\hat M = XY^\top$, where $(X,Y)$ is the limit of the gradient flow of $f$ starting from $(X^{\star},Y^{\star})$. We have already showed in \cref{subsubsection: FOCs of g} that $(XY^\top, \tau , m)$, satisfies all of the first order conditions of $g$; hence $\hat M = XY^\top$ is indeed a minimizer of $g$. First note that up to a bijective change in variables, minimizing $g$ is equivalent to minimizing the following function
\begin{equation*}
    \tilde g(N, \tau,m) = \frac{1}{2}\|O-N\|_\F^2 + \lambda  \left\| N - \sum_i \tau_i Z_i - m\mathbf 1^\top \right\|_\star.
\end{equation*}
Hence, it suffices to show that $\tilde g$ has a unique minimizer.
This function is strictly convex in $N$, because of the term $\|O-N\|_\F^2$. We can then fix $\hat N$ to be the unique value of $N$ that minimizes $\tilde g$. In particular, note that $\hat N = \hat M+\sum_i \hat \tau_i Z_i + \hat m \mathbf 1^\top$. 
It only remains to show that the following convex optimization problem
\begin{equation}\label{eq:min_pb}
    \min_{\tau,m} \left\| \hat N-\sum_i \tau_i Z_i -m\mathbf 1^\top \right\|_\star = \min_{\tau,m} \left\| \hat M-\sum_i \tau_i Z_i -m\mathbf 1^\top \right\|_\star
\end{equation}
has a unique minimizer. By the definition of $\hat M=XY^\top$, a minimum of the right-hand side of \eqref{eq:min_pb} is attained for $\tau=0$ and $m=0$ (otherwise, $\hat M$ wouldn't be an optimum of $g$). Now consider any other optimal solution to the problem in \eqref{eq:min_pb}, that is $Z\in\mathbf Z$ such that $\|XY^\top\|_\star = \|XY^\top +Z\|_\star$. Write the SVD $XY\top = U\Sigma V^\top$. Recall that the subgradients of the nuclear norm are $\set{UV^\top  + W : U^\top W=0, WV=0, \norm{W}\leq 1}$. 
By the convexity of the nuclear norm, we have
\begin{equation}\label{eq:lower_bound_convexity}
    \|XY^\top +Z\|_\star - \|XY^\top\|_\star \geq \sang{Z,UV^\top + W},
\end{equation}
for any matrix $W$ with $U^\top W=0$, $WV=0$ and $\|W\|\leq 1$. Recall that $\mathbf T_0$ is defined to be the tangent space of $XY^\top$. Defining the SVD $P_{\mathbf T_0^\perp}(Z) = \tilde U \tilde \Sigma \tilde V^\top$, we can take $W=\tilde U\tilde V^\top$, which gives
\begin{equation}\label{eq:lower_bound_nuclear_norm}
    \|XY^\top +Z\|_\star - \|XY^\top\|_\star \geq \sang{Z,UV^\top} + \|P_{\mathbf T_0^\perp}(Z)\|_\star.
\end{equation}
We now use the following lemma.

\begin{claim}\label{claim:reduction_assumptions}
    Under the same assumptions as in \cref{lem:general-conditions-small-ball}, we have
    \begin{equation*}
        \|P_{\mathbf T_0^\perp}(Z)\|_\star > |\sang{Z,UV^\top}|\qquad \forall Z\in\mathbf Z\setminus\{0\}.
    \end{equation*}
\end{claim}

With this result, we have that if $Z\neq 0$, then $\|XY^\top +Z\|_\star > \|XY^\top\|_\star$, which contradicts the definition of $Z$. Hence, $Z=0$ which ends the proof that $g$ has a unique minimizer.

\begin{proof}[Proof of \cref{claim:reduction_assumptions}]
    Up to changing $Z$ into $-Z$, it suffices to show that for all non-zero $Z\in\mathbf Z$, we have $\|P_{\mathbf T_0^\perp}(Z)\|_\star > \sang{Z,UV^\top}$. First, using similar arguments as in \eqref{eq:lower_bound_nuclear_norm} we show that for any matrices $A$ and $B$ such that $\sang{A,B}=0$ with SVD $A=U_A\Sigma_A V_A^\top$, we have
    \begin{equation*}
        \|A+B\|_\star - \|A\|_\star \geq \sang{B,U_A V_A^\top} = 0. 
    \end{equation*}
    Hence, $ \|A+B\|_\star \geq \|A\|_\star $. In particular, we can take $A=P_{\mathbf T^\perp}(Z)$ and $B =P_{\mathbf T_0^\perp}(Z) - P_{\mathbf T^\perp}(Z)$ since they are orthogonal to each other due to the fact that $\mathbf T^\perp\subset \mathbf T_0^\perp$.
    Then, we have
    \begin{equation*}
        \|P_{\mathbf T_0^\perp}(Z)\|_\star \geq \|P_{\mathbf T^\perp}(Z)\|_\star.
    \end{equation*}
    Next, by \cref{lemma: YT1=0} we have $V^\top \mathbf 1=0$, so that $\sang{Z,UV^\top} = \sang{Z,P_{\mathbf 1^\perp}(UV^\top)}=\sang{P_{\mathbf 1^\perp}(Z),P_{\mathbf 1^\perp}(UV^\top)}$. The two previous steps essentially show that we can ignore the terms of the form $\alpha \mathbf 1^\top$ within $Z$. Formally, it suffices to show that for any non-zero $Z\in span(Z_i,i\in[k])$, we have $\|P_{\mathbf T^\perp}(Z)\|_\star > \sang{Z,P_{\mathbf 1^\perp}(UV^\top)}$. We decompose such a matrix as $Z = \sum_{i\in[k]}\alpha_iZ_i$. First, by \eqref{eq:generalized-PTP-Z-ratio} of \cref{lem:general-conditions-small-ball}, we have that $\|P_{\mathbf T^\perp}(Z)\|_\F^2\geq \frac{c_{r_1}}{2\log(n)}\|Z\|_F^2>0$, which implies $\|P_{\mathbf T^\perp}(Z)\|>0$.  Then, with $\alpha:=(\alpha_1,\ldots,\alpha_k)\neq 0 $, we have
    \begin{align*}
        \|P_{\mathbf T^\perp}(Z)\|_\star - \sang{Z,P_{\mathbf 1^\perp}(UV^\top)} &\overset{(i)}{\geq} \frac{\|P_{\mathbf T^\perp}(Z)\|_\F^2}{\|P_{\mathbf T^\perp}(Z)\|} - \sang{Z,P_{\mathbf 1^\perp}(UV^\top)}\\
        &\overset{(ii)}{=} \frac{1}{\|P_{\mathbf T^\perp}(Z)\|}\paren{\alpha^\top \tilde D \alpha - \alpha^\top \tilde \Delta^1 \cdot  \|P_{\mathbf T^\perp}(Z)\|}\\
        &= \frac{1}{\|P_{\mathbf T^\perp}(Z)\|}\paren{\alpha^\top \tilde D \alpha - \alpha^\top \tilde \Delta^1 \cdot  \norm{ \sum_{i\in[k]}\alpha_i P_{\mathbf T^\perp}(Z_i)}}\\
        &\geq \frac{1}{\|P_{\mathbf T^\perp}(Z)\|}\underbrace{\paren{ \alpha^\top \tilde D \alpha - \alpha^\top \tilde \Delta^1 \cdot \|\alpha\|\sum_{i\in[k]}\|P_{\mathbf T^\perp}(Z_i)\|}}_{a},
    \end{align*}
where $(i)$ is due to the fact that $\|P_{\mathbf T^\perp}(Z)\|>0$ and the fact that the Frobenius norm of a matrix, squared, is the sum of the squares of the singular values of that matrix; $(ii)$ uses the identity $\|P_{\mathbf T^\perp}(Z)\|_\F^2 = \sang{\sum_{i\in[k]}\alpha_i P_{\mathbf T^\perp}(Z_i),\sum_{j\in[k]}\alpha_j P_{\mathbf T^\perp}(Z_j)}=\alpha^\top \tilde D \alpha$.
    
By \eqref{eq: generalized D} of \cref{lem:general-conditions-small-ball}, the matrix $\tilde D$ is invertible. It is also symmetric by construction. We next define $\beta = \tilde D^{1/2}\alpha$. We obtain
    \begin{align*}
        a&= \|\beta\|^2 - \beta^\top \tilde D^{-1/2}\tilde \Delta^1   \cdot \|\tilde D^{-1/2}\beta\|\sum_{i\in[k]}\|P_{\mathbf T^\perp}(Z_i)\|\\
        &\geq \|\beta\|^2\paren{1-\|\tilde D^{-1/2}\| \cdot \|\tilde \Delta^1   \| \cdot \| \tilde D^{-1/2}\| \sum_{i\in[k]}\|P_{\mathbf T^\perp}(Z_i)\|}\\
        &\overset{(i)}{>}0,
    \end{align*}
    where in $(i)$ we used \eqref{eq: generalized spectral} of \cref{lem:general-conditions-small-ball} together with the fact that $\beta\neq 0$. Combining the two previous inequalities shows that
    \begin{equation*}
        \|P_{\mathbf T^\perp}(Z)\|_\star > \sang{Z,P_{\mathbf 1^\perp}(UV^\top)},\quad \forall Z\in span(Z_i,i\in[k]). 
    \end{equation*}
    This ends the proof of the claim.
\end{proof}

\subsection{\texorpdfstring{Proof of \cref{lemma: sigma min D}}{Proof of Lemma \ref{lemma: sigma min D}}}
\label{subsection: lemma sigma min D}

\LemmaSigmaMinD*

\begin{proof}
    By \cref{lemma: main}, 
    %with probability at least $1-n^{-9}$ and
    for sufficiently large $n$, we can write $\hat M = XY^\top$ where $(X,Y)$ is the limit of the gradient flow of $f$ started at $(X^{\star},Y^{\star})$, and $(X,Y)\in\mathcal B$. Hence, the conditions for applying both \cref{lem:general-conditions-small-ball} and \cref{lemma: E hat 2} with $\hat{\mathbf T}$, the tangent space of $\hat M$, are satisfied. By \eqref{eq: generalized D} from \cref{lem:general-conditions-small-ball}, we have the desired bound on $\sigma_{\min}(D)$. Furthermore, we have
\begin{align*}
    \norm{\Delta^2} \leq \sqrt k\norm{\Delta^2}_{\infty} \lesssim\sqrt{\log n} \norm{\Delta^2}_{\infty}.  
\end{align*}

We aim to bound  $\norm{\Delta^2}_{\infty} = \max_{i\in[k]}\abs{\sang{P_{\hat{\mathbf T}^{\perp}}(Z_i), \hat{E}}}$, where $\hat E=E+\delta$.
%For any $m\in[k]$, we have $\abs{\sang{P_{\hat{\mathbf T}^{\perp}}(Z_m), \hat{E}}} \leq  \abs{\sang{P_{\hat{\mathbf T}^{\perp}}(Z_m), {E}}} +  \abs{\sang{P_{\hat{\mathbf T}^{\perp}}(Z_m), \delta} }.$
%By \eqref{eq: first term of E hat 2},
%$\abs{\sang{P_{\hat{\mathbf T}^{\perp}}(Z_m), {E}}} \lesssim\sigma\log^{0.5}n$ with probability at least $1-n^{-10}$. Now it remains to bound $\abs{\sang{P_{\hat{\mathbf T}^{\perp}}(Z_m), \delta} }. $
We have
\begin{align*}
    \abs{\sang{P_{\hat{\mathbf T}^{\perp}}(Z_i), \hat E}}&\leq \abs{\sang{P_{\hat{\mathbf T}^{\perp}}(Z_i) - P_{{\mathbf T}^{\star\perp}}(Z_i) , \hat E}} + \abs{\sang{P_{{\mathbf T}^{\star\perp}}(Z_i), \hat E}}\\
    &\overset{(i)}{\leq} \norm{P_{\hat{\mathbf T}^{\perp}}(Z_i) - P_{{\mathbf T}^{\star\perp}}(Z_i)}_\star(\norm{E}+\norm{\delta}) + \abs{\sang{P_{{\mathbf T}^{\star\perp}}(Z_i),\hat E}} \\
    &\overset{(ii)}{\lesssim} \frac{\rho \sqrt{\kappa r}}{\sqrt{\sigma_{\min}} } \sigma \sqrt n + \abs{\sang{P_{{\mathbf T}^{\star\perp}}(Z_i), \hat E}}\\
    &\overset{(iii)}{\lesssim} \frac{\sigma^2 r^{1.5} \kappa  n \log^6(n) }{\sigma_{\min} } + \abs{\sang{P_{{\mathbf T}^{\star\perp}}(Z_i), P_{{\mathbf T}^{\star\perp}}(\hat E)}}
\end{align*}
where $(i)$ is due to Von Neumann's trace inequality, $(ii)$ is due to \eqref{eq:PTP-PTPstar-Z} and \cref{assum:delta singular value}, and $(iii)$ is due to \eqref{eq: rho bound, unsimplified}. 
%Applying a union bound, with probability at least $1-n^{-9}-n^{-10}$,
Hence,
\begin{align*}
     \norm{\Delta^2} &\lesssim \frac{\sigma^2 r^{1.5} \kappa  n \log^{6.5}(n) }{\sigma_{\min} } +  \log^{0.5}(n) \cdot \max_{i\in[k]}
     \abs{\sang{P_{{\mathbf T}^{\star\perp}}(Z_i), P_{{\mathbf T}^{\star\perp}}(\hat E)}}\\
     & =  \frac{\sigma^2 r^{1.5} \kappa  n \log^{6.5}(n) }{\sigma_{\min} } +  \log^{0.5}(n) \cdot \max_{i\in[k]}
     \frac{\abs{\sang{P_{{\mathbf T}^{\star\perp}}(\tilde Z_i), P_{{\mathbf T}^{\star\perp}}(E+\delta)}}}{ \|\tilde Z_i\|_\F}.  
\end{align*}

\end{proof}

\section{Details of Computational Experiments}
\label{appendix: computational}

In this section, we presents more details and results for \cref{section: experiments}. 
Our code and data are available at \href{https://github.com/emilyyzhang/PaCE}{this GitHub repository}. 

\textbf{Compute resources:}
Our experiments utilized a high-performance computing cluster, leveraging CPU-based nodes for all experimental runs. Our home directory within the cluster provides 10 TB of storage. 
Experiments were conducted on nodes equipped with Intel Xeon CPUs, configured with 5 CPUs and 20 GB of RAM. We utilized approximately 40 nodes operating in parallel to conduct 200 iterations of experiments for each parameter set. While a single experiment completes within minutes on a personal laptop, executing the full suite of experiments on the cluster required approximately one day.

\textbf{Method implementations:}
Similar to the methodology outlined in the supplemental materials of \cite{farias2021learning}, we implemented PaCE using an alternating minimization technique to solve the convex optimization problem. To tune the hyperparameter $\lambda$, we initially set it to a large value and gradually reduced it until the rank of $\hat{M}$ reached a pre-defined rank $r$. In all of our experiments, we fixed $r$ to 6. We note that we did not tune this pre-defined rank, and our method might achieve better performance with a different rank.

The following models from the Python \texttt{econml} package were implemented. The supervised learning models from \texttt{sklearn} were chosen for the parameters in our implementation:
\begin{itemize}
    \item \texttt{DML} and \texttt{CausalForestDML:}
    \begin{itemize}
        \item \texttt{model\_y:} \texttt{GradientBoostingRegressor}
        \item \texttt{model\_t:} \texttt{GradientBoostingClassifier}
        \item \texttt{model\_final:} \texttt{LassoCV} (for \texttt{DML})
    \end{itemize}

    \item \texttt{LinearDML:}
    \begin{itemize}
        \item \texttt{model\_y:} \texttt{RandomForestRegressor}
        \item \texttt{model\_t:} \texttt{RandomForestClassifier}
    \end{itemize}

    \item \texttt{XLearner} and \texttt{ForestDRLearner:}
    \begin{itemize}
        \item \texttt{models/model\_regression:} \texttt{GradientBoostingRegressor}
        \item \texttt{propensity\_model/model\_propensity:} \texttt{GradientBoostingClassifier}
    \end{itemize}
\end{itemize}
Additionally, we implemented \texttt{DRLearner} from the Python \texttt{econml} package using the default parameters. We also implemented \texttt{multi\_arm\_causal\_forest} from the \texttt{grf} package in R.

Finally, we implemented matrix completion with nuclear norm minimization (MCNNM), where the matrix of observed entries is completed using nuclear norm penalization. The regularization parameter for the nuclear norm term, $\lambda$, is chosen in the same way that it is chosen in our implementation of PaCE.

We note that LLMs facilitated our implementation of all these methods, the experiment and visualization code, and our data processing code.

\textbf{Notes about the data:}
We imputed missing data values for 2020 by calculating the average values from 2019 and 2021. Furthermore, in the supplemental code and data provided [available at \href{https://github.com/emilyyzhang/PaCE}{this GitHub repository}], we redacted a covariate column `Snap\_Percentage\_Vulnerable' that was used in our experiments because it was obtained using proprietary information.  

\textbf{Additional results:} 
\cref{tab:winners-all-methods} presents the proportion of instances in which each method achieved the lowest normalized Mean Absolute Error (nMAE) across all of our experiments. It encompasses every method we benchmarked.

\cref{tab:avg_nmae} and \cref{tab:avg_nmae_treated} detail the complete results for the average nMAE of each method, along with the corresponding standard deviations. \cref{tab:avg_nmae} reports the nMAE values for estimating heterogeneous treatment effects across all observations. In contrast, \cref{tab:avg_nmae_treated} focuses on the nMAE values for treated observations only. Note that MCNNM only appears in \cref{tab:avg_nmae_treated}, as this method does not directly give estimates of heterogeneous treatment effects for untreated observations.

In the following tables, the "All" column presents the result for all instances across all parameter settings. These settings include proportions of zip codes treated at 0.05, 0.25, 0.5, 0.75, and 1.0; either with adaptive or non-adaptive treatment application approaches; and under additive or multiplicative treatment effects. Each parameter set is tested in 200 instances, totaling 4000 instances reflected in each value within the "All" column. Subsequent columns detail outcomes when one parameter is held constant. Specifically, the "Adaptive" and "Effect" columns each summarize results from 2000 instances, while each result under the "Proportion treated" column corresponds to 800 instances. 
The standard deviations, shown in parentheses within \cref{tab:avg_nmae} and \cref{tab:avg_nmae_treated}, are calculated for each result by analyzing the variability among the instances that contribute to that particular data point.

\begin{table}[t]
\caption{Proportion of instances where each method results in the lowest nMAE, across all methods considered. }\label{tab:winners-all-methods}
\centering
\resizebox{\textwidth}{!}{
\begin{tabular}[t]{lccccccccccccc}
\toprule
 & && \multicolumn{2}{c}{Adaptive} && \multicolumn{5}{c}{Proportion treated} && \multicolumn{2}{c}{Effect}\\
\cline{4-5}\cline{7-11}\cline{13-14}
\addlinespace
 & All  && N & Y && 0.05 & 0.25 & 0.5 & 0.75 & 1.0 && Add. & Mult.\\
\midrule
SNAP \\
\cline{1-1}
\addlinespace
PaCE & 0.59 &  & 0.58 & 0.60 &  & 0.26 & 0.54 & 0.76 & 0.75 & 0.65 &  & 0.64 & 0.54\\
Causal Forest & 0.16 &  & 0.17 & 0.16 &  & 0.14 & 0.17 & 0.16 & 0.14 & 0.21 &  & 0.02 & 0.30\\
CausalForestDML & 0.00 &  & 0.00 & 0.00 &  & 0.00 & 0.00 & 0.00 & 0.00 & 0.00 &  & 0.00 & 0.00\\
DML & 0.11 &  & 0.21 & 0.00 &  & 0.12 & 0.13 & 0.08 & 0.09 & 0.11 &  & 0.17 & 0.04\\
DRLearner & 0.00 &  & 0.00 & 0.00 &  & 0.00 & 0.00 & 0.00 & 0.00 & 0.00 &  & 0.00 & 0.00\\
ForestDRLearner & 0.00 &  & 0.00 & 0.00 &  & 0.00 & 0.00 & 0.00 & 0.00 & 0.00 &  & 0.00 & 0.00\\
LinearDML & 0.14 &  & 0.04 & 0.24 &  & 0.48 & 0.15 & 0.01 & 0.02 & 0.04 &  & 0.17 & 0.11\\
XLearner & 0.00 &  & 0.00 & 0.00 &  & 0.00 & 0.00 & 0.00 & 0.00 & 0.00 &  & 0.00 & 0.00\\
\addlinespace
State \\
\cline{1-1}
\addlinespace
PaCE & 0.41 &  & 0.39 & 0.42 &  & 0.23 & 0.32 & 0.39 & 0.54 & 0.56 &  & 0.46 & 0.36\\
Causal Forest & 0.05 &  & 0.09 & 0.01 &  & 0.06 & 0.04 & 0.04 & 0.05 & 0.04 &  & 0.04 & 0.05\\
CausalForestDML & 0.07 &  & 0.07 & 0.07 &  & 0.03 & 0.07 & 0.10 & 0.07 & 0.08 &  & 0.06 & 0.09\\
DML & 0.27 &  & 0.29 & 0.26 &  & 0.41 & 0.36 & 0.27 & 0.17 & 0.16 &  & 0.31 & 0.24\\
DRLearner & 0.00 &  & 0.00 & 0.00 &  & 0.00 & 0.00 & 0.00 & 0.00 & 0.00 &  & 0.00 & 0.00\\
ForestDRLearner & 0.01 &  & 0.01 & 0.01 &  & 0.00 & 0.00 & 0.00 & 0.01 & 0.02 &  & 0.00 & 0.01\\
LinearDML & 0.02 &  & 0.01 & 0.04 &  & 0.02 & 0.06 & 0.01 & 0.01 & 0.02 &  & 0.02 & 0.02\\
XLearner & 0.17 &  & 0.13 & 0.20 &  & 0.25 & 0.14 & 0.19 & 0.15 & 0.11 &  & 0.10 & 0.23\\
\addlinespace
County &&&&&&&\\
\cline{1-1}
\addlinespace
PaCE & 0.06 &  & 0.10 & 0.02 &  & 0.16 & 0.08 & 0.04 & 0.02 & 0.01 &  & 0.07 & 0.06\\
Causal Forest & 0.29 &  & 0.38 & 0.20 &  & 0.25 & 0.28 & 0.31 & 0.31 & 0.30 &  & 0.23 & 0.34\\
CausalForestDML & 0.04 &  & 0.05 & 0.03 &  & 0.00 & 0.01 & 0.03 & 0.05 & 0.11 &  & 0.03 & 0.05\\
DML & 0.35 &  & 0.35 & 0.35 &  & 0.34 & 0.37 & 0.34 & 0.39 & 0.31 &  & 0.47 & 0.24\\
DRLearner & 0.00 &  & 0.00 & 0.00 &  & 0.00 & 0.00 & 0.00 & 0.00 & 0.00 &  & 0.00 & 0.00\\
ForestDRLearner & 0.01 &  & 0.00 & 0.01 &  & 0.00 & 0.00 & 0.00 & 0.00 & 0.02 &  & 0.00 & 0.01\\
LinearDML & 0.20 &  & 0.04 & 0.35 &  & 0.22 & 0.22 & 0.21 & 0.16 & 0.17 &  & 0.19 & 0.20\\
XLearner & 0.05 &  & 0.07 & 0.04 &  & 0.03 & 0.03 & 0.06 & 0.07 & 0.07 &  & 0.01 & 0.10\\
\bottomrule
\end{tabular}
}

\end{table}

\begin{table}[!ht]
\centering
\caption{Average nMAE across methods. Standard deviations shown in parentheses.} \label{tab:avg_nmae}
\centering
\resizebox{\textwidth}{!}{
\begin{tabular}[t]{lccccccccccccc}
\toprule
 & && \multicolumn{2}{c}{Adaptive} && \multicolumn{5}{c}{Proportion treated} && \multicolumn{2}{c}{Treatment effect}\\
\cline{4-5}\cline{7-11}\cline{13-14}
\addlinespace
 & All  && N & Y && 0.05 & 0.25 & 0.5 & 0.75 & 1.0 && Add. & Mult.\\
\midrule
\addlinespace
SNAP   &&&&&&&\\
\cline{1-1}
\addlinespace
PaCE & 0.18 &  & 0.17 & 0.18 &  & 0.4 & 0.17 & 0.12 & 0.1 & 0.09 &  & 0.13 & 0.22\\
 & (0.16) &  & (0.11) & (0.2) &  & (0.21) & (0.1) & (0.07) & (0.06) & (0.05) &  & (0.13) & (0.18)\\
 \addlinespace
Causal Forest (R) & 0.21 &  & 0.21 & 0.21 &  & 0.42 & 0.2 & 0.17 & 0.13 & 0.12 &  & 0.18 & 0.23\\
 & (0.15) &  & (0.07) & (0.19) &  & (0.19) & (0.06) & (0.04) & (0.04) & (0.06) &  & (0.12) & (0.17)\\
 \addlinespace
CausalForestDML & 40.81 &  & 5.01 & 76.6 &  & 195.95 & 6.85 & 0.8 & 0.26 & 0.17 &  & 40.93 & 40.68\\
 & (115.45) &  & (10.86) & (154.86) &  & (191.08) & (6.37) & (0.42) & (0.05) & (0.06) &  & (115.69) & (115.24)\\
 \addlinespace
DML & 0.38 &  & 0.31 & 0.45 &  & 0.72 & 0.35 & 0.32 & 0.27 & 0.26 &  & 0.28 & 0.48\\
 & (0.29) &  & (0.21) & (0.33) &  & (0.41) & (0.17) & (0.15) & (0.14) & (0.16) &  & (0.27) & (0.27)\\
 \addlinespace
DRLearner & 1186.44 &  & 371.84 & 2001.03 &  & 107.4 & 1016.66 & 2277.78 & 2445.56 & 84.78 &  & 1095.4 & 1277.47\\
 & (1764.1) &  & (462.49) & (2164.23) &  & (229.15) & (815.45) & (2036.92) & (2339.59) & (211.55) &  & (1597.24) & (1912.51)\\
 \addlinespace
ForestDRLearner & 0.94 &  & 0.84 & 1.04 &  & 2.08 & 0.9 & 0.71 & 0.57 & 0.45 &  & 0.99 & 0.89\\
 & (7.08) &  & (1.53) & (9.9) &  & (15.78) & (0.1) & (0.13) & (0.17) & (0.21) &  & (9.49) & (3.2)\\
 \addlinespace
LinearDML & 0.25 &  & 0.28 & 0.21 &  & 0.34 & 0.25 & 0.25 & 0.2 & 0.19 &  & 0.15 & 0.34\\
 & (0.16) &  & (0.16) & (0.16) &  & (0.23) & (0.13) & (0.12) & (0.13) & (0.14) &  & (0.1) & (0.16)\\
 \addlinespace
XLearner & 0.51 &  & 0.46 & 0.56 &  & 0.89 & 0.55 & 0.47 & 0.36 & 0.27 &  & 0.5 & 0.52\\
 & (0.25) &  & (0.17) & (0.3) &  & (0.29) & (0.05) & (0.04) & (0.02) & (0.02) &  & (0.23) & (0.27)\\
\addlinespace
State &&&&&&&\\
\cline{1-1}
\addlinespace
PaCE & 0.28 &  & 0.34 & 0.22 &  & 0.57 & 0.32 & 0.21 & 0.17 & 0.14 &  & 0.25 & 0.32\\
 & (0.24) &  & (0.26) & (0.2) &  & (0.3) & (0.18) & (0.12) & (0.1) & (0.08) &  & (0.21) & (0.25)\\
 \addlinespace
Causal Forest (R) & 0.33 &  & 0.39 & 0.28 &  & 0.6 & 0.35 & 0.29 & 0.23 & 0.19 &  & 0.3 & 0.36\\
 & (0.22) &  & (0.24) & (0.18) &  & (0.27) & (0.16) & (0.11) & (0.08) & (0.08) &  & (0.19) & (0.24)\\
 \addlinespace
CausalForestDML & 0.75 &  & 1.02 & 0.48 &  & 2.76 & 0.36 & 0.25 & 0.21 & 0.19 &  & 0.76 & 0.74\\
 & (3.22) &  & (4.41) & (1.1) &  & (6.85) & (0.25) & (0.11) & (0.1) & (0.08) &  & (3.31) & (3.14)\\
 \addlinespace
DML & 0.35 &  & 0.42 & 0.28 &  & 0.59 & 0.34 & 0.29 & 0.27 & 0.26 &  & 0.31 & 0.39\\
 & (0.33) &  & (0.41) & (0.19) &  & (0.6) & (0.21) & (0.17) & (0.14) & (0.14) &  & (0.37) & (0.28)\\
 \addlinespace
DRLearner & 929.38 &  & 520.49 & 1338.07 &  & 1466.8 & 2905.93 & 159.67 & 68.36 & 47.3 &  & 891.42 & 967.37\\
 & (4601.05) &  & (3268.39) & (5597.16) &  & (5776.54) & (8056.84) & (898.43) & (348.15) & (677.07) &  & (4089.69) & (5062.03)\\
 \addlinespace
ForestDRLearner & 60.29 &  & 119.37 & 1.24 &  & 299.72 & 0.87 & 0.4 & 0.28 & 0.26 &  & 52.15 & 68.43\\
 & (629.42) &  & (886.43) & (2.42) &  & (1382.62) & (1.11) & (0.21) & (0.15) & (0.14) &  & (451.79) & (767.07)\\
\addlinespace
LinearDML & 0.86 &  & 1.34 & 0.39 &  & 2.72 & 0.59 & 0.4 & 0.33 & 0.29 &  & 0.85 & 0.88\\
 & (1.56) &  & (2.06) & (0.36) &  & (2.74) & (0.41) & (0.22) & (0.16) & (0.13) &  & (1.63) & (1.48)\\
\addlinespace
XLearner & 0.29 &  & 0.37 & 0.22 &  & 0.53 & 0.31 & 0.23 & 0.2 & 0.19 &  & 0.27 & 0.31\\
 & (0.19) &  & (0.2) & (0.14) &  & (0.25) & (0.13) & (0.1) & (0.08) & (0.06) &  & (0.16) & (0.2)\\
\addlinespace
County &&&&&&&\\
\cline{1-1}
\addlinespace
PaCE & 0.34 &  & 0.35 & 0.33 &  & 0.49 & 0.37 & 0.29 & 0.28 & 0.26 &  & 0.28 & 0.39\\
 & (0.2) &  & (0.19) & (0.22) &  & (0.29) & (0.19) & (0.14) & (0.12) & (0.11) &  & (0.13) & (0.24)\\
\addlinespace
Causal Forest (R) & 0.26 &  & 0.26 & 0.27 &  & 0.43 & 0.3 & 0.22 & 0.2 & 0.18 &  & 0.22 & 0.31\\
 & (0.17) &  & (0.13) & (0.21) &  & (0.23) & (0.17) & (0.12) & (0.1) & (0.08) &  & (0.12) & (0.21)\\
 \addlinespace
CausalForestDML & 2.1 &  & 0.47 & 3.72 &  & 8.83 & 0.95 & 0.28 & 0.23 & 0.2 &  & 2.05 & 2.15\\
 & (5.24) &  & (0.68) & (7.01) &  & (8.92) & (0.78) & (0.11) & (0.09) & (0.07) &  & (5.17) & (5.31)\\
 \addlinespace
DML & 0.3 &  & 0.34 & 0.25 &  & 0.46 & 0.28 & 0.26 & 0.23 & 0.26 &  & 0.21 & 0.38\\
 & (0.25) &  & (0.25) & (0.24) &  & (0.27) & (0.2) & (0.23) & (0.22) & (0.26) &  & (0.17) & (0.28)\\
 \addlinespace
DRLearner & 717.31 &  & 260.77 & 1173.85 &  & 560.3 & 153.83 & 72.65 & 428.02 & 2371.74 &  & 396.35 & 1038.27\\
 & (5046.04) &  & (834.88) & (7058.49) &  & (1091.51) & (664.27) & (462.94) & (827.06) & (11014.13) &  & (1118.57) & (7034.13)\\
 \addlinespace
ForestDRLearner & 0.75 &  & 0.51 & 1 &  & 2.57 & 0.4 & 0.3 & 0.26 & 0.23 &  & 0.57 & 0.93\\
 & (5.99) &  & (1.35) & (8.35) &  & (13.24) & (0.15) & (0.12) & (0.1) & (0.06) &  & (2.93) & (7.94)\\
\addlinespace
LinearDML & 0.36 &  & 0.47 & 0.24 &  & 0.61 & 0.34 & 0.29 & 0.26 & 0.28 &  & 0.28 & 0.44\\
 & (0.32) &  & (0.36) & (0.24) &  & (0.42) & (0.23) & (0.27) & (0.25) & (0.27) &  & (0.25) & (0.37)\\
 \addlinespace
XLearner & 0.32 &  & 0.33 & 0.32 &  & 0.5 & 0.36 & 0.29 & 0.25 & 0.23 &  & 0.29 & 0.36\\
 & (0.14) &  & (0.13) & (0.16) &  & (0.17) & (0.11) & (0.09) & (0.07) & (0.04) &  & (0.09) & (0.18)\\
\bottomrule
\end{tabular}
}
\end{table}

\begin{table}[!ht]
\small
\centering
\caption{Average nMAE among treated units. Standard deviations shown in parentheses.} \label{tab:avg_nmae_treated}
\centering
\resizebox{\textwidth}{!}{
\begin{tabular}[t]{lccccccccccccc}
\toprule
 & && \multicolumn{2}{c}{Adaptive} && \multicolumn{5}{c}{Proportion treated} && \multicolumn{2}{c}{Treatment effect}\\
\cline{4-5}\cline{7-11}\cline{13-14}
\addlinespace
 & All  && N & Y && 0.05 & 0.25 & 0.5 & 0.75 & 1.0 && Add. & Mult.\\
\midrule
\addlinespace
SNAP   &&&&&&&\\
\cline{1-1}
\addlinespace
PaCE & 0.1 &  & 0.11 & 0.09 &  & 0.15 & 0.1 & 0.09 & 0.09 & 0.09 &  & 0.07 & 0.13\\
 & (0.08) &  & (0.04) & (0.1) &  & (0.13) & (0.06) & (0.05) & (0.05) & (0.05) &  & (0.03) & (0.09)\\
 \addlinespace
Causal Forest (R) & 0.13 &  & 0.18 & 0.07 &  & 0.18 & 0.13 & 0.12 & 0.11 & 0.1 &  & 0.11 & 0.14\\
 & (0.08) &  & (0.04) & (0.07) &  & (0.11) & (0.07) & (0.07) & (0.06) & (0.06) &  & (0.07) & (0.09)\\
 \addlinespace
CausalForestDML & 0.22 &  & 0.27 & 0.17 &  & 0.41 & 0.2 & 0.17 & 0.16 & 0.15 &  & 0.18 & 0.26\\
 & (0.24) &  & (0.11) & (0.31) &  & (0.45) & (0.09) & (0.08) & (0.07) & (0.07) &  & (0.12) & (0.31)\\
 \addlinespace
DML & 0.44 &  & 0.29 & 0.6 &  & 0.89 & 0.41 & 0.37 & 0.29 & 0.27 &  & 0.27 & 0.62\\
 & (0.72) &  & (0.19) & (0.98) &  & (1.44) & (0.37) & (0.28) & (0.17) & (0.18) &  & (0.22) & (0.97)\\
 \addlinespace
DRLearner & 2074.25 &  & 376.92 & 3771.58 &  & 170.11 & 2703.71 & 4134.22 & 3279.81 & 83.39 &  & 1405.57 & 2742.93\\
 & (4283.12) &  & (461.16) & (5542.67) &  & (263) & (5417.88) & (5543.25) & (4232.75) & (203.48) &  & (2139.72) & (5587.94)\\
 \addlinespace
ForestDRLearner & 4.86 &  & 1.19 & 8.54 &  & 20.94 & 1.39 & 0.88 & 0.63 & 0.47 &  & 4.08 & 5.64\\
 & (98.58) &  & (1.92) & (139.32) &  & (219.79) & (0.89) & (0.45) & (0.31) & (0.26) &  & (103.28) & (93.67)\\
\addlinespace
LinearDML & 0.26 &  & 0.25 & 0.26 &  & 0.31 & 0.28 & 0.27 & 0.22 & 0.22 &  & 0.12 & 0.39\\
 & (0.26) &  & (0.12) & (0.35) &  & (0.3) & (0.29) & (0.28) & (0.21) & (0.2) &  & (0.09) & (0.3)\\
 \addlinespace
MCNNM & 0.19 &  & 0.24 & 0.14 &  & 0.27 & 0.18 & 0.17 & 0.16 & 0.16 &  & 0.16 & 0.21\\
 & (0.18) &  & (0.05) & (0.24) &  & (0.34) & (0.1) & (0.09) & (0.08) & (0.08) &  & (0.09) & (0.23)\\
 \addlinespace
XLearner & 0.92 &  & 0.44 & 1.41 &  & 2.67 & 0.75 & 0.54 & 0.38 & 0.28 &  & 0.58 & 1.27\\
 & (2.77) &  & (0.19) & (3.85) &  & (5.81) & (0.62) & (0.35) & (0.15) & (0.07) &  & (0.49) & (3.85)\\
\addlinespace
State &&&&&&&\\
\cline{1-1}
\addlinespace
My Estimator & 0.21 &  & 0.23 & 0.19 &  & 0.36 & 0.24 & 0.17 & 0.15 & 0.13 &  & 0.18 & 0.24\\
 & (0.17) &  & (0.19) & (0.15) &  & (0.27) & (0.14) & (0.09) & (0.08) & (0.07) &  & (0.13) & (0.2)\\
 \addlinespace
Causal Forest (R) & 0.3 &  & 0.32 & 0.28 &  & 0.46 & 0.33 & 0.29 & 0.23 & 0.19 &  & 0.27 & 0.33\\
 & (0.22) &  & (0.24) & (0.19) &  & (0.37) & (0.17) & (0.13) & (0.09) & (0.08) &  & (0.15) & (0.26)\\
 \addlinespace
CausalForestDML & 0.26 &  & 0.33 & 0.2 &  & 0.43 & 0.27 & 0.22 & 0.21 & 0.19 &  & 0.24 & 0.28\\
 & (0.24) &  & (0.31) & (0.12) &  & (0.46) & (0.13) & (0.1) & (0.09) & (0.08) &  & (0.16) & (0.3)\\
 \addlinespace
DML & 0.34 &  & 0.38 & 0.29 &  & 0.51 & 0.34 & 0.29 & 0.27 & 0.26 &  & 0.3 & 0.37\\
 & (0.23) &  & (0.25) & (0.2) &  & (0.34) & (0.21) & (0.17) & (0.14) & (0.14) &  & (0.2) & (0.25)\\
 \addlinespace
DRLearner & 1704.15 &  & 791.94 & 2615.89 &  & 4387.86 & 3853.02 & 165.89 & 68.68 & 47.43 &  & 1413.19 & 1995.26\\
 & (11821.06) &  & (5032.31) & (15890.26) &  & (23128.14) & (11970.8) & (1085.93) & (353.98) & (662.49) &  & (6457.8) & (15418.33)\\
 \addlinespace
ForestDRLearner & 427.34 &  & 849.26 & 5.63 &  & 2134.78 & 1.45 & 0.46 & 0.29 & 0.26 &  & 329.6 & 525.13\\
 & (4126.02) &  & (5805.96) & (17.91) &  & (9032.24) & (2.32) & (0.33) & (0.17) & (0.16) &  & (2885.66) & (5071.02)\\
\addlinespace
LinearDML & 0.43 &  & 0.52 & 0.35 &  & 0.81 & 0.42 & 0.35 & 0.31 & 0.28 &  & 0.39 & 0.48\\
 & (0.53) &  & (0.65) & (0.36) &  & (1.06) & (0.24) & (0.17) & (0.14) & (0.12) &  & (0.27) & (0.7)\\
 \addlinespace
MCNNM & 0.23 &  & 0.29 & 0.18 &  & 0.25 & 0.21 & 0.21 & 0.24 & 0.25 &  & 0.22 & 0.24\\
 & (0.17) &  & (0.21) & (0.08) &  & (0.3) & (0.11) & (0.1) & (0.09) & (0.11) &  & (0.11) & (0.21)\\
 \addlinespace
XLearner & 0.28 &  & 0.33 & 0.23 &  & 0.5 & 0.29 & 0.22 & 0.2 & 0.18 &  & 0.26 & 0.3\\
 & (0.31) &  & (0.38) & (0.21) &  & (0.61) & (0.12) & (0.08) & (0.07) & (0.05) &  & (0.15) & (0.41)\\
\addlinespace
County &&&&&&&\\
\cline{1-1}
\addlinespace
PaCE & 0.31 &  & 0.29 & 0.33 &  & 0.44 & 0.34 & 0.27 & 0.26 & 0.26 &  & 0.23 & 0.39\\
 & (0.34) &  & (0.13) & (0.46) &  & (0.6) & (0.36) & (0.19) & (0.13) & (0.11) &  & (0.09) & (0.46)\\
\addlinespace
Causal Forest (R) & 0.2 &  & 0.25 & 0.15 &  & 0.3 & 0.21 & 0.18 & 0.16 & 0.16 &  & 0.16 & 0.24\\
 & (0.16) &  & (0.12) & (0.18) &  & (0.26) & (0.15) & (0.11) & (0.08) & (0.07) &  & (0.1) & (0.19)\\
 \addlinespace
CausalForestDML & 0.28 &  & 0.3 & 0.26 &  & 0.51 & 0.29 & 0.22 & 0.19 & 0.18 &  & 0.22 & 0.34\\
 & (0.35) &  & (0.18) & (0.46) &  & (0.69) & (0.21) & (0.12) & (0.09) & (0.07) &  & (0.11) & (0.48)\\
 \addlinespace
DML & 0.3 &  & 0.33 & 0.27 &  & 0.41 & 0.27 & 0.26 & 0.24 & 0.31 &  & 0.19 & 0.41\\
 & (0.34) &  & (0.24) & (0.41) &  & (0.39) & (0.26) & (0.26) & (0.24) & (0.45) &  & (0.15) & (0.42)\\
 \addlinespace
DRLearner & 1183.28 &  & 278.38 & 2088.17 &  & 1010.99 & 160.94 & 69.94 & 564.55 & 4109.95 &  & 443.32 & 1923.23\\
 & (10927.48) &  & (1008.75) & (15369.32) &  & (2371.12) & (956.66) & (448.59) & (1563.04) & (24023.26) &  & (1265.5) & (15367.99)\\
 \addlinespace
ForestDRLearner & 1.57 &  & 1.14 & 1.99 &  & 6.58 & 0.54 & 0.3 & 0.22 & 0.19 &  & 1.06 & 2.08\\
 & (12.1) &  & (2.9) & (16.86) &  & (26.48) & (0.41) & (0.17) & (0.1) & (0.08) &  & (5.01) & (16.35)\\
\addlinespace
LinearDML & 0.33 &  & 0.41 & 0.25 &  & 0.48 & 0.31 & 0.28 & 0.26 & 0.32 &  & 0.22 & 0.44\\
 & (0.39) &  & (0.29) & (0.46) &  & (0.56) & (0.27) & (0.28) & (0.27) & (0.44) &  & (0.17) & (0.5)\\
 \addlinespace
MCNNM & 0.56 &  & 0.49 & 0.62 &  & 0.73 & 0.58 & 0.51 & 0.49 & 0.48 &  & 0.45 & 0.67\\
 & (0.63) &  & (0.12) & (0.88) &  & (1.21) & (0.57) & (0.33) & (0.19) & (0.14) &  & (0.11) & (0.87)\\
 \addlinespace
XLearner & 0.36 &  & 0.32 & 0.4 &  & 0.67 & 0.39 & 0.29 & 0.24 & 0.22 &  & 0.28 & 0.44\\
 & (0.52) &  & (0.14) & (0.72) &  & (1.04) & (0.34) & (0.17) & (0.09) & (0.06) &  & (0.12) & (0.72)\\
\bottomrule
\end{tabular}
}
\end{table}

\end{document}